\documentclass[11pt]{article}
\usepackage{enumerate}
\usepackage{pdfsync}
\usepackage[OT1]{fontenc}

\usepackage{url}           
\usepackage{booktabs}    
\usepackage{amsfonts}      
\usepackage{nicefrac}      
\usepackage{microtype}    
\usepackage{color}

\usepackage{fullpage}
\usepackage{setspace}
\usepackage{tabularx}
\usepackage{float}
\usepackage{graphicx}
\usepackage{wrapfig,lipsum}
\usepackage{enumitem}
\usepackage{natbib}
\usepackage{makecell}
\usepackage{booktabs}
\usepackage{array}
\usepackage{multirow} 
\usepackage[dvipsnames]{xcolor}

\makeatletter
\newcommand*{\rom}[1]{\expandafter\@slowromancap\romannumeral #1@}
\makeatother

\usepackage{amsfonts}       %
\usepackage{nicefrac}       %
\usepackage{amsmath,amsthm}
\usepackage{caption}

\usepackage[colorlinks,
            linkcolor=blue,
            anchorcolor=blue,
            citecolor=blue
           ]{hyperref}
\usepackage[nameinlink]{cleveref}   

\usepackage{algorithm}
\usepackage{algorithmic}
\usepackage{subfigure}
\usepackage{comment}
\usepackage{enumitem}

\newtheorem{theorem}{Theorem}
\newtheorem{lemma}[theorem]{Lemma}

\newtheorem{definition}{Definition}

\newtheorem{assumption}{Assumption}
\newtheorem{proposition}{Proposition}
\newtheorem{corollary}[theorem]{Corollary}

\newcommand{\E}{\mathbb{E}}

\newcommand{\cA}{\mathcal{A}}
\newcommand{\cN}{\mathcal{N}}
\newcommand{\cS}{\mathcal{S}}
\newcommand{\cJ}{\mathcal{J}}

\newcommand{\TV}{\mathrm{TV}}
\newcommand{\dist}{\mathrm{dist}}

\newcommand{\actionmax}{A_{\max}}
\newcommand{\Qtau}{Q}
\newcommand{\Vtau}{V}
\newcommand{\initialState}{\rho}

\newcommand{\ext}[1]{\overline{#1}}
\newcommand{\algname}{\text{LPI}}
\newcommand{\khop}{{\color{black}\kappa}}

\newcommand{\nik}{{\color{black}\cN_i^{\khop}}}
\newcommand{\njk}{{\color{black}\cN_j^{\khop}}}

\newcommand{\nminusik}{{\color{black}\cN_{-i}^{\khop}}}

\newcommand{\policyevaluation}{{\texttt{PolicyEvaluation}}}
\newcommand{\localizedtd}{{\text{Localized TD(0)}}}

\newcommand{\nib}{{\cN_i^{\beta}}}
\newcommand{\njb}{{\cN_j^{\beta}}}
\newcommand{\nlb}{{\cN_\ell^{\beta}}}
\newcommand{\nkb}{{\cN_k^{\beta}}}

\newcommand{\abs}[1]{\left\lvert#1\right\rvert}
\DeclareMathOperator*{\argmax}{arg\,max}
\DeclareMathOperator*{\argmin}{arg\,min}

\allowdisplaybreaks
\title{Global Convergence of Localized Policy Iteration in Networked Multi-Agent Reinforcement Learning}

\author{
    Yizhou Zhang\footnote{Yizhou Zhang, Guannan Qu, Pan Xu contributed equally to this work. } \thanks{Tsinghua University; e-mail: {\tt yz-zhang19@mails.tsinghua.edu.cn}}
    ~~
    Guannan Qu\footnotemark[1] \thanks{Carnegie Mellon University; e-mail: {\tt gqu@andrew.cmu.edu}}
    ~~ 
    Pan Xu\footnotemark[1] \thanks{Duke University; e-mail: {\tt pan.xu@duke.edu}}
    ~~ 
    Yiheng Lin\thanks{California Institute of Technology; e-mail: {\tt yihengl@caltech.edu}}
    ~~
    Zaiwei Chen\thanks{California Institute of Technology; e-mail: {\tt zchen458@caltech.edu}}
    ~~ 
    Adam Wierman\thanks{California Institute of Technology; e-mail: {\tt adamw@caltech.edu}}
}

\begin{document}

\date{}
\maketitle
\begin{abstract}
We study a multi-agent reinforcement learning (MARL) problem where the agents interact over a given network. The goal of the agents is to cooperatively maximize the average of their entropy-regularized long-term rewards. To overcome the curse of dimensionality and to reduce communication, we propose a Localized Policy Iteration ($\algname$) algorithm that provably learns a near-globally-optimal policy using only local information. In particular, we show that, despite restricting each agent's attention to only its $\kappa$-hop neighborhood, the agents are able to learn a policy with an optimality gap that decays polynomially in $\kappa$. In addition, we show the finite-sample convergence of $\algname$ to the global optimal policy, which explicitly captures the trade-off between optimality and computational complexity in choosing $\kappa$. Numerical simulations demonstrate the effectiveness of $\algname$.

\end{abstract}

\maketitle

\section{Introduction}

Reinforcement learning (RL) has seen remarkable successes in recent years, many of which fall into the multi-agent setting, such as playing multi-agent games \citep{silver2016mastering,mnih2015human}, smart grid \citep{chen2022reinforcement}, queueing networks \citep{walton2021learning}, etc. In this work, we focus on a form of \emph{networked} Multi-Agent RL (MARL) where the agents interact according to a given network graph. 

Compared to single-agent RL, MARL faces many additional challenges. First of all, the curse of dimensionality (which is already a major challenge in single-agent RL) becomes a more severe issue in MARL because the complexity of the problem scales exponentially with the number of agents \citep{marl_littman1994markov,marl_bu2008comprehensive,kearns1999efficient,factor_guestrin2003efficient}. This is because the size of the state and action space scales exponentially with the number of agents and, as a result, the dimension of the value/$Q$-functions and the policies all scale exponentially with the number of agents, which is hard to compute and store in moderately large networks \citep{bertsekas1996neuro}. Moreover, since the agents are coupled by the global state, the training process requires extensive communication among the agents in the entire network.

To overcome the aforementioned challenges, the existing literature considers performing MARL with only local information. For example, it has been proposed that each agent's policy depends only on the states of itself and, potentially, its neighboring agents. An example of this type is the independent learners approach \citep{tan1993multi,marl_claus1998dynamics}, where each agent learns a policy for itself while treating the states and actions of other agents as part of the environment. Another example is the recent work of \citet{qu2020scalablelocal,qu2020scalable,lin2021multi}, which focuses on learning localized policies where each agent is allowed to choose its action based on its own and neighbors' states. Beyond the localization in policy, it has also been proposed to approximate each agent's value or $Q$ functions in a way that they only depend on the local and nearby agents' states and actions \citep{factor_guestrin2003efficient,yang2018mean,qu2020scalablelocal,qu2020scalable,lin2021multi}, as opposed to the full state. 
These approaches greatly relieve the computation and communication burden both empirically \citep{tan1993multi,foerster2016learning,sukhbaatar2016learning} and theoretically  \citep{qu2020scalablelocal,qu2020scalable,lin2021multi}.

Despite this progress, there are still major limitations. As discussed above, in order to speed up learning, many previous works, e.g.,  \cite{tan1993multi,marl_claus1998dynamics,qu2020scalablelocal,qu2020scalable,lin2021multi}, 
restrict consideration to localized policies where each agent makes decisions only based on the local state of the agent and, potentially, neighbors, as opposed to the global state.
This leads to a fundamental performance gap between the class of localized policies considered and the optimal centralized policy. Even when the best localized policy can be found, there is a performance degradation compared to the best centralized policy, as the centralized policy class is a strict superset of the localized policy class. An important open question that remains is the following:
\begin{center}
    \textit{How large is the gap between localized policies and the optimal centralized policy? \\ How much information must be available to each local agent \\in order to achieve a near-optimal performance?}
\end{center}
This question has received increasing attention in linear control settings \citep{bamieh2002distributed,motee2008optimal,shin2022near}, but is still open in MARL settings. In this work, we provide the first bounds on the gap between localized and centralized policies under a MARL model in networked systems.

Another limitation of prior work studying MARL is that, while existing results have provided convergence bounds for local policies, the convergence of their methods is typically only to a suboptimal policy in the localized policy class. For example, in \citet{qu2020scalable,lin2021multi}, the policies are shown to converge to a stationary point of the objective function defined on localized policies, as opposed to the global optimum. This is unsatisfying as converging to a stationary point does not even guarantee converging to the best localized policy. Therefore, another important and open question that remains is the following:
\begin{center}
   \textit{Is it possible to design a MARL algorithm that provably \\ finds a near-globally-optimal policy using only local information?}  
\end{center}
This question has received increasing attention in single-agent settings, where people have studied the convergence to the global optimum for policy-gradient methods under various policy classes \citep{bhandari2019global,agarwal2019theory}. However, it is open in the context of learning localized policies in MARL and in this work, we provide an affirmative answer to this question under an MARL model in networked systems.

\subsection{Contributions} Motivated by the open questions above, our work proposes and analyzes a new class of localized policies for networked MARL and proposes a Localized Policy Iteration (LPI) algorithm that converges to a near-globally-optimal policy.  Our main contributions are summarized in the following.

\begin{itemize}[leftmargin=*]
    \item \textbf{Near-Globally-Optimal $\kappa$-Hop Localized Policies.} We show that a class of $\kappa$-hop localized policies are nearly globally optimal, where a $\kappa$-hop localized policy means that each agent is allowed to choose its action based on the states of its $\kappa$-hop local neighborhood. More specifically, in \Cref{thm:subopt_dist_policy}, we show that there exists a $\kappa$-hop localized policy whose optimality gap is polynomially small in $\kappa$, where the optimality gap is with respect to the best centralized policy that is allowed to depend on the global state. As a result, even with a small $\kappa$, the class of $\kappa$-hop localized policies is near optimal despite the fact that each agent only uses  information from within a small $\kappa$-hop neighborhood to make its decision. This result justifies using localized policies in networked MARL. 
    \item \textbf{Localized Policy Iteration.} Motivated by the result described above, we propose a localized MARL algorithm, a.k.a. $\algname$, given in \Cref{alg:SVI}. At a high level, $\algname$ iteratively performs policy improvement (and policy evaluation), but restricted to $\kappa$-hop policies.  
    While standard policy improvement requires using the information of global states and actions in order to conduct the policy improvement step, we develop a \textit{soft} policy improvement that approximately performs policy improvement to $\kappa$-hop localized policies using only local information. Therefore, our proposed algorithm can be implemented in a truly localized manner.
    \item \textbf{Finite-Sample Analysis of $\algname$.} We provide a global convergence guarantee for $\algname$ in \Cref{thm:convergence_svi}.  The guarantee holds for any policy evaluation method used as a subroutine in the evaluation step of $\algname$ (denoted as $\policyevaluation$) that satisfies a mild  condition specified in \Cref{def:beta-hop-state-aggregation-policy-eva}. 
    Furthermore, in \Cref{cor:convergence_with_localized_td}, we provide the finite-sample complexity (cf. \eqref{eq:sample_complexity}) for $\algname$ when choosing a specific $\policyevaluation$ method proposed by \citet{lin2021multi}. 
    Specifically, we show that in order to achieve an $\varepsilon$ optimality (compared to the best centralized policy), one needs to use a $\kappa$-hop localized policy with $\kappa = \Theta(poly(\frac{1}{\varepsilon}))$ in $\algname$, and the sample complexity in learning such a $\kappa$-hop localized policy with $\varepsilon$ optimality gap scales polynomially with the largest state-action space size of local neighborhoods, as opposed to the global network. 
    \end{itemize}

The key technical novelty in this paper is a policy closure argument (\Cref{thm:convergence_prototype}), where we identify a class of policies that satisfy a form of spatial decaying properties and show that they are closed under the entropy regularized Bellman operator in MARL. This key observation implies that when starting from a policy in this class with spatial decaying properties, the policy improvement procedure will result in a new policy within this class, which further reveals that the optimal policy is also in the spatial decaying policy class.  We then show that $\algname$ is an inexact version of the above policy iteration given by the regularized Bellman operator, where the learned policies and $Q$-functions are truncated to only depend on a localized neighborhood of each agent, and the exact policy evaluation and improvement steps are approximated by finite-sample estimations. Therefore, combining the policy closure argument for the exact version of $\algname$ and error bound analyses for the approximation made in $\algname$, we are able to show its convergence to a near global optimal policy, even though we only use localized information in $\algname$.

While the main contribution of this work is to theoretically show that $\algname$ can converge to the global optimal policy using only localized information, we also verify its empirical performance on a simulated networked MARL problem. In particular, we design an example of a spreading process over a network where the optimal policy depends on the global states of each agent (not just the local information).  We run $\algname$ on this example and the results highlight that $\algname$ can perform well even outside the region suggested by the theory. 

\subsection{Related Literature}

Reinforcement Learning (RL) studies a dynamical environment where the agent decides its current action based on the current state and the current state/action affects the distribution of its next state. This environment is usually modeled as a \textit{Markov Decision Process} (MDP) where the key assumption is the \textit{Markov Property}, which means the current state and action are statistically sufficient for deciding the next state (see, e.g., \citet{sutton2018reinforcement}). Specifically, an MDP (with discounted accumulative reward) can be characterized as a tuple $\langle \mathcal{S}, \mathcal{A}, P, \gamma, r\rangle$, where $\mathcal{S}$ and $\mathcal{A}$ denote the state/action space. The transition function $P: \mathcal{S} \times \mathcal{A}\to \Delta_\mathcal{S}$ characterizes the transition probability $\mathbf{P}(s'\mid s, a)$ from the current state/action pair to next state $s'$. $\gamma \in (0, 1)$ denotes the discount factor, and the reward function $r: \mathcal{S} \times \mathcal{A} \to [0, \bar{r}]$ is nonnegative for the agent with the state/action pair $(s, a)$. 

In single-agent RL, the goal of the agent is to learn a policy $\zeta: \mathcal{S} \to \Delta(\mathcal{A})$ to maximize the expected discounted accumulative reward
$J(\zeta) :=  \mathbb{E}_{s(0) \sim \rho, a(t) \sim \zeta(\cdot\mid s(t))}\left[\sum_{t=0}^T \gamma^t r(s(t), a(t))\right].$ Two classic learning algorithms have been proposed to learn the optimal policy when the transition probabilities are known: value iteration and policy iteration \citep{sutton2018reinforcement}. \textit{Value Iteration} (VI) iteratively updates the estimated optimal value function by the Bellman equation and derives the final policy from the learned optimal value function. In contrast, \textit{Policy Iteration} (PI) works by evaluating the current policy and iteratively updating the policy using this policy's value function. Our algorithm, $\algname$, 
can be viewed as a generalization of PI to a networked setting where the transition probabilities are unknown. Lower bound results show tabular RL algorithms inevitably suffer from \textit{the curse of dimensionality} when the transition probabilities are unknown (see, e.g., \citet{jin_is_2018}), which means the sample complexity (i.e., the number of samples to find a near-optimal policy) grows with respect to the size of state space $\mathcal{S}$ and action space $\mathcal{A}$.

Multi-Agent Reinforcement Learning (MARL) generalizes the single-agent RL setting to a situation where the global MDP evolves based on the joint action/policy of $n$ agents. The generalization to include more agents is necessary for solving  practical problems that involve large-scale networks and/or games. Specifically, each agent $i$ has its local action space $\mathcal{A}_i$ and its local reward function $r_i$, and the global action space $\mathcal{A} = \mathcal{A}_1 \times \cdots \times \mathcal{A}_n$. At every time step $t$, agent $i$ receives a partial observation $o_i(s(t))$ to decide its local action $a_i(t)$. More details about this setting can be found in \Cref{sec:model_pre}. Depending on each agent's learning objective, MARL can be divided to two categories: cooperative or competitive MARL (see \cite{zhang2019multi} for a survey). In \textit{cooperative MARL}, all agents work together to optimize a shared global objective \citep{qu2020scalable, qu2020scalablelocal, lin2021multi}, which is also the setting studied in this work. In \textit{competitive MARL}, each agent optimizes its local accumulative reward, and the goal is to find an approximate Nash Equilibrium \citep{ding2022independent, leonardos2021global}. 

A major challenge for cooperative MARL is that the size of the global action space $\mathcal{A}$ can grow exponentially with respect to the number of agents $n$, which makes centralized training intractable. Fully distributed training is not practical either, because its reward and transition probability depend on other agents' policies. To address this challenge and derive finite-sample complexity bounds that are not exponential in $n$, a common assumption made by previous works \citep{doan2019finite, suttle2020multi} is that the global state is observable to all agents (i.e., $o_i(s) = s$), which thus eliminates the need for communicating local states in the network. This works in special settings such as cooperative navigation and Predator-Prey where a joint state space is easy to observe or all agents share the same state space \citep{lowe2017multi, zhang2018fully}. Nevertheless, such an assumption is not practical in more general MARL settings like the one studied in our paper, where each agent has its own state space and the global state space $\mathcal{S} = \mathcal{S}_1 \times \cdots \times \mathcal{S}_n$ is exponentially large with respect to $n$. In this setting, huge communication costs are usually unavoidable and thus it is more challenging and less studied.

Our work is most related to a class of cooperative MARL problems where the agents are located in a network graph $\mathcal{G}$ \citep{qu2020scalablelocal, qu2020scalable, lin2021multi}, which makes similar structural assumptions on the MDP as this paper, i.e., each agent has its own local state/action space and the local transition probabilities depend only on the agent's neighbors. Each agent $i$ observes the local states of the agents whose graph distance to $i$ is less than or equal to $\kappa$ to decide its local action (i.e., $o_i(s) = s_{\cN_i^\kappa}$, where $\cN_i^\kappa := \{j \mid d_{\mathcal{G}}(i, j) \leq \kappa\}$. $d_{\mathcal{G}}(i, j)$ denotes the length of the shortest path between $i$ and $j$ on $\mathcal{G}$). As a remark, all local policies become centralized when the dependence parameter $\kappa$ exceeds the diameter of graph $\mathcal{G}$. Previous works \citep{qu2020scalablelocal, qu2020scalable, lin2021multi} propose decentralized policy iteration algorithms to learn local policies for each agent\footnote{The local policy defined in these papers focuses on a local agent that chooses local actions based only on the agent's state.} and prove finite-time convergence bounds. However, there is still an open gap between the localized policy and centralized policy where each agent makes decisions based on the global state. Specifically, before our work, there is no result on how large the gap between the optimal $\kappa$-hop localized policy and the optimal centralized policy is, nor how fast it decays with respect to $\kappa$. Besides, even if the objective is defined with respect to the $\kappa$-hop localized policy class, the algorithms proposed by \citet{qu2020scalablelocal, qu2020scalable, lin2021multi} are only guaranteed to converge to stationary points of their local objective functions rather than global optima of the objective.

Another challenge widely encountered in RL is that the policy might rapidly become deterministic during the training process, which further leads to a slow convergence. To speed up the convergence and also encourage exploration for the training algorithm to escape suboptimal points, entropy regularization has often been added to the value function.  This approach has yielded both good empirical performance \citep{williams1991function,mnih2016asynchronous,vieillard2020leverage} and strong theoretical guarantees \citep{agarwal2020optimality,mei2020global,cen2021fast}. In this paper, we focus on the entropy regularized problem in multi-agent reinforcement learning.

\section{Model \& Preliminaries}\label{sec:model_pre}

In this section, we introduce the model we study and provide preliminaries that are used throughout our paper. 
In particular, we introduce networked multi-agent Markov decision processes in \Cref{subsec:mamdp}. %
Then, in \Cref{subsec:entropy_reg_value_and_q}, we define the (state) value function and action-state value function (i.e., $\Qtau$ function) in the multi-agent setting with an entropy regularization. In \Cref{subsec:spatial_decay_property}, we introduce a novel class of polices where the dependence of local policies on other agents' actions decay polynomially as the graph distance increases.  This property is key to our analysis. Note that a summary table of notation is given in \Cref{tab:notation} in the appendix for the reader's reference.

\subsection{Multi-Agent Markov Decision Process}\label{subsec:mamdp}

We consider a network of $n$ agents that are associated with an undirected graph $\mathcal{G} = (\mathcal{N},\mathcal{E})$, where  $\mathcal{N}=\{1,\ldots,n\}$ is the set of nodes and $\mathcal{E}\subseteq \mathcal{N}\times\mathcal{N}$ is the set of edges. We denote the state-space and the action-space of agent $i$ by $\mathcal{S}_i$ and $\mathcal{A}_i$, respectively, and they are both finite sets. %
We also denote $\actionmax=\max_{i\in\mathcal{N}}|\cA_i|$. 
The global state is denoted as $s = (s_1,\ldots,s_n)\in \mathcal{S}:=\mathcal{S}_1\times\cdots\times \mathcal{S}_n$ and similarly the global action is denoted as $a=(a_1,\ldots,a_n)\in\mathcal{A}:=\mathcal{A}_1\times\cdots\times\mathcal{A}_n$. 
At time $t\geq 0$, given current state $s(t)$ and action $a(t)$, for each $i\in\mathcal{N}$, the next individual state $s_i(t+1)$ is independently generated and is only dependent on its neighbors' states and its own action:
\begin{align}
  P(s(t+1)\mid s(t),a(t)) = \prod_{i=1}^n P_i(s_i(t+1)\mid s_{\mathcal{N}_i}(t),a_i(t)),  \label{eq:transition_factor}
\end{align}
where $\cN_i=\{i\}\cup \{j\in\mathcal{N}\mid (i,j)\in\mathcal{E}\}$ denotes the neighborhood of $i$ (including $i$ itself) and $s_{\mathcal{N}_i}$ represents the states of the agents in $\mathcal{N}_i$. In addition, for integer $\khop\geq 0$, we use $\nik$ to denote the $\khop$-hop neighborhood of $i$, i.e., the nodes of which the graph distance to $i$ have length less than or equal to $\khop$. We also use $\nminusik=\cN/\nik$ to denote the agents that are not in $\nik$. We use $s_{\mathcal{N}_i^\kappa}$ and $a_{\mathcal{N}_i^\kappa}$ to denote the states and actions of the agents in $\mathcal{N}_i^\kappa$ respectively, use $s_{-i}, a_{-i}$ to denote the states and actions of all agents other than $i$, and we use $f(\kappa) = \sup_{i\in\cN} |\nik|$ to denote the size of the largest $\kappa$-hop neighborhood. 

Each agent is associated with a policy $\zeta_i$, which maps each global state $s\in\mathcal{S}$ to a probability distribution supported on the set of local actions $\mathcal{A}_i$, and we use $\Delta_{\mathcal{A}_i|\mathcal{S}}$ to denote the space $\zeta_i$ lies in. Each agent, conditioned on observing $s(t)$, takes an action $a_i(t)$ according to $\zeta_i(\cdot|s(t))$.
We use $\zeta(a|s) = \prod_{i=1}^n\zeta_i(a_i|s)$ to denote the joint policy, which is the product of all the individual policies. Since $\zeta$ is uniquely determined by the tuple of the $\zeta_i$'s, we also slightly abuse the notation and denote $\zeta = (\zeta_1,\zeta_2,\cdots,\zeta_n)\in \Delta_{\mathcal{A}_1|\mathcal{S}} \times \Delta_{\mathcal{A}_2|\mathcal{S}} \times \cdots \times \Delta_{\mathcal{A}_n|\mathcal{S}} : =\Delta_{\mathrm{policy}}$.

\subsection{Value Function and $Q$-Function}\label{subsec:entropy_reg_value_and_q}

Each agent $i\in\mathcal{N}$ is associated with a stage reward function $r_i(s_i,a_i)$ that depends on its local state and action. The global stage reward is defined as $r(s,a) = \frac{1}{n}\sum_{i=1}^n r_i(s_i,a_i)$. 
We assume that all rewards $r_i$ are bounded above by $\bar{r}$ throughout the paper. 
Given any joint policy $\zeta\in\Delta_{\mathrm{policy}}$, we define the entropy regularized value function at state $s$ as
\begin{align}\label{eq:value_func}
    \Vtau^\zeta (s)=  \E_{a(t) \sim \zeta(\cdot|s(t))}\bigg[ \sum_{t=0}^\infty \gamma^t \big[r(s(t),a(t)) -\tau \log(\zeta(a(t)|s(t)))\big]  \;\bigg|\; s(0) = s\bigg],
\end{align}
where $\tau> 0$ is a tunable parameter. 
Based on $V^\zeta(\cdot)$, we further define the objective function  as 
\[J(\zeta) = \E_{s \sim \initialState}[ \Vtau^\zeta(s) ],\] 
where $\initialState(\cdot)$ is a given initial state distribution. 
In this work, we aim to find a global optimal policy $\zeta^*\in \Delta_{\mathrm{policy}}$ that maximizes the objective function $J(\zeta)$.  

In both the definition of the value function and the objective function, there is a weighted entropy term $-\tau\log(\zeta(a(t)|s(t)))$ with positive weight $\tau$. Entropy regularization has gained popularity in the RL literature for both practical and theoretical reasons. Practically, it is known that adding regularization encourages randomness and exploration of the policy \citep{williams1991function,mnih2016asynchronous,vieillard2020leverage}, which is necessary in RL. To see this, observe that the optimal policy becomes the uniform policy as the parameter $\tau$ goes to infinity. Theoretically, due to the strong concavity of the negative entropy function, it has been shown that entropy regularization helps improving the convergence rate of several RL algorithms, e.g., natural policy gradient  \citep{cen2021fast,cayci2021linear}. Moreover, we show in \Cref{sec:algo_design} that the entropy regularization term plays an important role in our multi-agent RL setting as it will affect the suboptimality gap between localized policies and the best centralized policy.

Given the definition of the value function \eqref{eq:value_func}, the $\Qtau$ function of a joint policy $\zeta$ is defined as 
\begin{align}\label{eq:q_func}
    \Qtau^\zeta(s,a) = r(s,a) + \gamma \mathbb{E}_{s'\sim P(\cdot|s,a)} \Vtau^\zeta(s'). 
\end{align}
The value function and the $\Qtau$-function satisfy the following relationship
\begin{align}\label{eq:Vtau_Qtau_relation}
    \Vtau^{\zeta}(s) &=\mathbb{E}_{a\sim \zeta(\cdot|s) }\left[  r(s,a)  -  \tau\log \zeta(a|s) +  \gamma \mathbb{E}_{s'\sim P(\cdot|s,a)} \Vtau^\zeta(s')\right]\notag\\
    &=\mathbb{E}_{a\sim \zeta(\cdot|s) }\left[ \Qtau^\zeta(s,a) -  \tau \log \zeta(a|s) \right],
\end{align}
where the first line follows from the Bellman equation for $V^\zeta(\cdot)$.
In view of the additive structure of the stage rewards $r(s,a)$, we also define in the following the local value function and the local $\Qtau$ function for all $i\in\mathcal{N}$: 
\begin{align} 
     \Vtau^\zeta_i (s) &=  \E_{a(t) \sim \zeta(\cdot|s(t))}\bigg[\sum_{t=0}^\infty \gamma^t \left[ r_i(s_i(t),a_i(t) ) -n\tau \log(\zeta_i(a_i(t)|s(t)))\right]  \bigg| s(0) = s\bigg],  \label{eq:local_value_func} \\
     \Qtau^\zeta_i(s,a) &= r_i(s_i,a_i) + \gamma \mathbb{E}_{s'\sim P(\cdot|s,a)} \Vtau^\zeta_i(s'). \label{eq:local_q_func}
\end{align}
We sometimes use the vector $\mathbf{Q} = (Q_1,\ldots,Q_n)$ to denote the complete profile of all the local $Q$ functions. 
It is clear from the definition that 
\begin{align*}
    \Vtau^\zeta(s) = \frac{1}{n} \sum_{i=1}^n   \Vtau^\zeta_i(s),\quad\text{and}\quad 
    \Qtau^\zeta(s,a) = \frac{1}{n} \sum_{i=1}^n   \Qtau^\zeta_i(s,a).
\end{align*}

Similar to the standard Bellman optimal operator, 
we define the Bellman optimal operator with entropy regularization $\mathcal{T}: \mathbb{R}^{|\mathcal{S}|}\mapsto\mathbb{R}^{|\mathcal{S}|} $ as
\begin{align}
    [\mathcal{T} \Vtau](s) = \max_{\zeta\in \Delta_{\mathrm{policy}}} \mathbb{E}_{a\sim \zeta(\cdot|s) }\left[  r(s,a) -  \tau \log \zeta(a|s)+ \gamma \mathbb{E}_{s'\sim P(\cdot|s,a)} \Vtau(s') \right],\;\forall\; \Vtau\in\mathbb{R}^{|\mathcal{S}|},\;s\in\mathcal{S}. \label{eq:bellman_operator}
\end{align}
The following result regarding the properties of $\mathcal{T}(\cdot)$ is an immediate extension to that of the standard Bellman optimal operator. For completeness, a proof is presented in \Cref{subsec:proof_exist_optimal_policy_entropy}. 

\begin{proposition}\label{prop:exist_optimal_policy_entropy} 
$\mathcal{T}(\cdot)$ is a contraction mapping with respect to the $\ell_\infty$-norm, with contraction factor $\gamma$. In addition, suppose a policy $\zeta^*$ satisfies $V^{\zeta^*}=\mathcal{T}(V^{\zeta^*})$, then $\zeta^*$ is an optimal policy.
\end{proposition}

\subsection{Spatial Decay Properties}\label{subsec:spatial_decay_property}
Throughout this paper, we use spatial decay properties heavily, i.e. the fact that a certain quantity decays as the distance between two agents increases. Spatial decay properties have been investigated in the literature in various forms, including the exponential decay of $Q$-functions in \cite{qu2020scalablelocal,qu2020scalable} and correlation decay in combinatorial optimization \cite{gamarnik2013correlation,gamarnik2014correlation,bamieh2002distributed}. In this paper, we introduce the classes of $(\nu,\mu)$-decay polices and $(\nu,\mu)$-decay $Q$-functions. The decay properties of these quantities eventually enable us to control the optimality gap (compared to the best centralized policy) when restricting the agents to using only $\kappa$-hop localized policies (to be defined in \Cref{def:kappa-hop}). 

Compared to the exponential decay of $Q$-functions in the literature, our work considers a different type of decay whose rate is polynomial. Further, we extend the decay property from $Q$-functions to policies which, broadly speaking, is similar in concept to \citet{shin2022near}, which studies a similar decaying policy class in linear dynamical systems.

To start, we first introduce the $(\nu,\mu)$-decay matrices.

\begin{definition}[$(\nu,\mu)$-decay matrix]\label{def:matrix_decay}
For $\nu,\mu>0$, a matrix $A \in\mathbb{R}^{n\times n}$

is said to be $(\nu,\mu)$-decay with respect to graph $\mathcal{G} = (\mathcal{V}, \mathcal{E})$ with $\mathcal{V} = \{1,2,\cdots,n\}$ if every entry of $A$ is non-negative and 
\begin{align*}
    \max\bigg\{ \sup_i \sum_{j=1}^n A_{ij}(\dist(i,j)+1)^\mu,\sup_{j}\sum_{i=1}^n A_{ij}(\dist(i,j)+1)^\mu \bigg\}&\leq \nu, %
\end{align*}
where $\dist(i,j)$ is the graph distance between $i$ and $j$ in $\mathcal{G}$. 
\end{definition}

Intuitively speaking, for a $(\nu,\mu)$-decay matrix $A$, $A_{ij}$ decays polynomially as the graph distance between $i$  and $j$ increases. To see this, observe that we have for all $i\in \cN$ and $\kappa\in\mathbb{N}$: 
\begin{align}\label{eq:mu_decay_imply_distance_decay}
    \sum_{j:\dist(i,j)>\kappa} A_{ij} \leq \sum_{j:\dist(i,j)>\kappa} A_{ij}\frac{(\dist(i,j)+1)^{\mu}}{(\kappa+1)^\mu}\leq \frac{\nu}{(\kappa+1)^\mu}. 
\end{align}
Based on \Cref{def:matrix_decay}, we next define the $(\nu,\mu)$-decay policy class as follows.
\begin{definition}[$(\nu,\mu)$-decay policy]\label{def:mu_decay_policy}
Let $\zeta=(\zeta_1,\ldots,\zeta_n)$ be a joint policy, where $\zeta_i$ is the policy of agent $i\in \{1,2,\cdots,n\}$. The interaction matrix of $\zeta$, denoted by $Z^\zeta\in\mathbb{R}^{n\times n}$, is defined as
\begin{align*}
    Z_{ij}^\zeta = \max_{s_{-j}} \max_{s_j,s_j'} \TV(\zeta_i(\cdot|s_j,s_{-j}),\zeta_i(\cdot|s_j',s_{-j}) ),\quad\text{for all } i,j\in \{1,2,\cdots,n\}. 
\end{align*}
A policy $\zeta$ is said to be ($\nu$,$\mu$)-decay if $Z^\zeta$ is a ($\nu,\mu$)-decay matrix. 
\end{definition}

In a $(\nu,\mu)$-decay policy, agent $i$'s action $a_i$ depends on the state of agent $j$ in a manner that the dependence decays polynomially in the graph distance between $i$ and $j$. 

The next definition introduces $\sigma$-regular policies.
\begin{definition}\label{def:sigma_regular}
A distribution $d=(d_1,\ldots,d_M)$ over a set $\{1,2,\ldots,M\}$ is called $\sigma$-regular if $\max_{m,m'\in[M]}\log(d_m/d_{m'})\leq\sigma$. A joint policy $\zeta=(\zeta_1,\ldots,\zeta_n)$ is $\sigma$-regular if $\forall i, s\in\mathcal{S}$, the distribution $\zeta_i(\cdot|s)$ is $\sigma$-regular.
\end{definition}

Through out the paper, we use $\Delta_{\nu,\mu,\sigma}$ to denote the set of all joint policies that are $(\nu,\mu)$-decay and $\sigma$-regular.
Similarly, we also define a class of decaying $Q$-functions as follows.

\begin{definition}[($\nu$,$\mu$)-decay $\Qtau$-function class]\label{def:Z_Q}
Let $Q_i\in\mathbb{R}^{|\mathcal{S}||\mathcal{A}|}$, $i\in \mathcal{N}$ be the local $Q$-functions defined in \eqref{eq:local_q_func}, and let $\mathbf{Q} = \{Q_i\}_{i\in\mathcal{N}}$. The interaction matrix of $\mathbf{Q}$, denoted by $Z^{\mathbf{Q}}\in\mathbb{R}^{n\times n}$, is defined as
\begin{align*}
    Z^{\mathbf{Q}}_{ij} = \max_{s_{-j},a_{-j}}\max_{s_j,a_j,s_j',a_j'}|\Qtau_i(s_j,a_j, s_{-j},a_{-j}) - \Qtau_i(s_j',a_j', s_{-j},a_{-j})|, \quad\forall\; i,j\in\{1,2,\cdots,n\}. 
\end{align*}
For any given $\nu$ and $\mu$, a local $Q$ function tuple $\mathbf{Q} = \{Q_i\}_{i\in\mathcal{N}}$ is said to be $(\nu,\mu)$-decay if $Z^{\mathbf{Q}}$ is a $(\nu,\mu)$-decay matrix. 
\end{definition}

One benefit of the $(\nu,\mu)$-decay property of $\mathbf{Q} = \{Q_i\}_{i\in\mathcal{N}} $ is that it allows ``truncation'' of the local $Q$ functions. Specifically, given a positive integer $\beta$, and a local $Q$-function tuple $\mathbf{Q} = \{Q_i\}_{i\in\mathcal{N}} $, we can define the following truncated local $Q$-functions:
\begin{align}
    \hat{Q}_i(s_\nib,a_\nib) = \sum_{s_{\mathcal{N}_{-i}^\beta}, a_{\mathcal{N}_{-i}^\beta}} w(s_{\mathcal{N}_{-i}^\beta}, a_{\mathcal{N}_{-i}^\beta}; s_\nib,a_\nib) Q_i(s_\nib,s_{\mathcal{N}_{-i}^\beta},a_\nib a_{\mathcal{N}_{-i}^\beta}) \label{eq:truncated_q_def}
\end{align}
where weight parameters $w(s_{\mathcal{N}_{-i}^\beta}, a_{\mathcal{N}_{-i}^\beta}; s_\nib,a_\nib)\geq0$  satisfy $ \sum_{s_{\mathcal{N}_{-i}^\beta}, a_{\mathcal{N}_{-i}^\beta}} w(s_{\mathcal{N}_{-i}^\beta}, a_{\mathcal{N}_{-i}^\beta}; s_\nib,a_\nib) = 1$.

It can be shown that despite the reduction in dimension, the truncated local $Q$ functions are a good approximate of the original local $Q$ functions under the $(\nu,\mu)$-decay property. 

\begin{proposition}[Adapted from Lemma 4 in \cite{qu2020scalablelocal}]\label{prop:truncated_q_error}
Suppose the local $Q$ functions $\{Q_i\}_{i\in\mathcal{N}}$ satisfy $(\nu,\mu)$ decay property. Then, $\forall i$,
\begin{align*}
    \sup_{s,a} \Big| Q_i(s,a) - \hat{Q}_i(s_\nib,a_\nib) \Big| \leq \frac{\nu}{(\beta+1)^\mu}. 
\end{align*}
\end{proposition}
\noindent This truncation of the local $Q$ functions have been adopted in the literature \citep{qu2020scalablelocal,qu2020scalable}. We also use this in our approach.

\section{Algorithm Design }\label{sec:algo_design}

We now present the design of Localized Policy Iteration, i.e., $\algname$. As discussed in the introduction, it is impractical to implement $\zeta_i$ because it needs access to the global state $s=(s_1,\ldots,s_n)$. To see this, note that it requires $\Omega(|\mathcal{S}|) = \Omega(\prod_{i=1}^n|\mathcal{S}_i|)$ parameters to even specify such a policy, which is impractically large and hard to store. Moreover, each agent would need access to the states of all the agents in order to implement such a policy, which is often hard to achieve in a large networked system due to communication challenges. 

To address these issues, our focus is on localized policies, where agent $i$'s action depends only on the states of the agents who are close to $i$ in the network graph $\mathcal{G}$. Given this restriction on the policies, one may immediately ask what is lost when we restrict to localized policies compared with the centralized policies where the agents are allowed access to the global state. We show in \Cref{subsec:almost-optimal} that, under a structural assumption, using $\kappa$-hop localized policies is almost as good as using centralized policies, where the parameter $\kappa$ captures the level of localization. Then, in \Cref{subsec:algo}, we introduce our \algname\ framework that learns a near optimal $\kappa$-hop localized policy.

\subsection{Performance Gap between Localized Policies and Centralized Policies}\label{subsec:almost-optimal}

We begin by formally defining the class of $\kappa$-hop localized policies we consider. Recall that $\kappa>0$ is a positive integer, and we use $s_{\mathcal{N}_i^\kappa}$ to denote the states of the agents in $\mathcal{N}_i^\kappa$.

\begin{definition}[$\kappa$-hop policies]\label{def:kappa-hop}
A policy $\zeta = (\zeta_1,\ldots,\zeta_n) \in\Delta_{\mathrm{policy}}$ is called a $\kappa$-hop (localized) policy if the following equation holds for all $i\in\mathcal{N}$, $s_{\nik}\in \mathcal{S}_{\nik}$, and $s_{\nminusik},s'_{\nminusik} \in\mathcal{S}_{\nminusik}$:
\[  \zeta_i(\cdot\mid s_{\nik},s_{\nminusik}) = \zeta_i(\cdot\mid s_{\nik},s'_{\nminusik}) .\]
In other words, there exists $\hat{\zeta}_i: \mathcal{S}_{\nik} \mapsto \Delta_{\mathcal{A}_i}$ such that 
$\zeta_i(\cdot|s_{\nik},s_{\nminusik}) = \hat{\zeta}_i(\cdot|s_{\nik})$.
\end{definition}

From \Cref{def:kappa-hop}, we see that, when using a $\kappa$-hop localized policy, each agent only needs to know the states of the agents that are within its $\kappa$-hop neighborhood. Note that $\kappa$ is a tunable parameter in practice, and captures the trade-off between communication, optimality, and computational complexity. Such a policy is more practical to implement than a centralized policy because of both the reduction in dimension and the reduction in communication. 

To state our result on the performance gap, the following definition regarding the system transition matrix is needed. Following~\citet{qu2020scalable}, we define a matrix $C\in\mathbb{R}^{n\times n}$ that characterizes the interaction strength between agents via the total variation of the transition probabilities. 
\begin{align}\label{def:C_mat_tv}
    {C}_{ij} = \left\{ \begin{array}{ll}
0, & \text{ if }j\notin \cN_i,\\
\sup_{s_{\cN_i/j},a_i}\sup_{s_j,s_j'} \TV(  P_i(\cdot|s_j, s_{\cN_i/j},a_i) , P_i(\cdot| s_j', s_{\cN_i/j},a_i)   ),  & \text{ if } j\in \cN_i/{i},\\
\sup_{s_{\cN_i/i}}\sup_{s_i,s_i',a_i,a_i'} \TV(  P_i(\cdot|s_i, s_{\cN_i/i},a_i) , P_i(\cdot| s_i', s_{\cN_i/i},a_i')   ),  & \text{ if } j=i.
\end{array} \right.
\end{align}

Our first main result states that $\kappa$-hop polices are nearly as good as centralized policies, with a gap that decays polynomially in $\kappa$. 
\begin{theorem}\label{thm:subopt_dist_policy}
Suppose that $\gamma<0.8$, $\tau \geq 6\bar{r} \frac{4-3\gamma}{4-5\gamma}\actionmax^2 e$, and $\sum_{j\in \cN_i}C_{ij}<1/2$ for all $i\in \{1,2,\cdots,n\}$. Then there exists a $\kappa$-hop policy $\hat{\zeta}^{*,\kappa}$ that satisfies 
\begin{align*}
    J(\zeta^*) - J(\hat{\zeta}^{*,\kappa}) \leq  {\frac{\bar{r} (4-3\gamma)}{(1-\gamma)(4-5\gamma)} \frac{1}{(\kappa+1)^\mu},} 
\end{align*}
where 
\begin{align*}
\mu= \min\bigg\{\log_2 \frac{\tau(4-5\gamma)}{6\bar{r}\actionmax^2 e(4-3\gamma)} ,\log_2 \frac{1}{2 \sup_{i\in\mathcal{N}}\sum_{j\in \cN_i}C_{ij}}\bigg\}.
\end{align*}
\end{theorem}

The above theorem states that when $\kappa$ increases, the optimality gap of $\kappa$-hop localized policies decays polynomially in $1/\kappa$. This suggests that, even for $\kappa$-hop policies with a relatively small $\kappa$, one can achieve reasonably good performance and also avoid the curse of dimensionality and the  communication issue in large networked systems.  \Cref{thm:subopt_dist_policy} motivates us to learn $\kappa$-hop localized policies, which is presented in the next section.

We note that the polynomial decay rate $\mu$ in \Cref{thm:subopt_dist_policy} depends on several factors. The factor $\sum_{j\in \cN_i} C_{ij}$ can be viewed as the interaction strength between node $i$ and its neighbors. The smaller the total interaction strength is for all agents, the larger the decay parameter $\mu$. This is intuitive since weaker interaction means that $\kappa$-hop localized policies should perform better as there is less coupling between the nodes. Another important factor is $\tau$.  The larger $\tau$ is, the larger the decay rate $\mu$. Intuitively, a larger $\tau$ generally means entropy regularization will play a more important role, which effectively ``dampens'' the interaction among agents. To understand this, note that the policy that maximizes the entropy is the uniform policy, and therefore, the larger the entropy regularization is, the less incentive there is for an agent's policy to depend on the states of far away agents.  

It is also interesting to compare our result to the results in the Linear Quadratic Control (LQC) setting \citep{bamieh2002distributed,motee2008optimal,shin2022near}, where it has been shown that the performance gap between the optimal $\kappa$-hop policy and the optimal centralized policy decays exponentially in $\kappa$. This is a faster decay rate than the polynomial rate in our result. One reason for this difference is that the LQC setting is inherently a linear setting,  whereas our MARL setting is inherently a combinatorial setting, which is typically more complicated. It remains an interesting question to understand the fundamental reason behind the discrepancy between the LQC setting and our setting.  

Lastly, we note that, for \Cref{thm:subopt_dist_policy} to hold, we need a lower bound on $\tau$ and an upper bound on $\sup_i\sum_{j\in \cN_i} C_{ij}$. 
We believe these two bounds are hard to avoid.  Bounds like these are common in the spatial decay literature in combinatorial settings. As an example, in \cite{gamarnik2013correlation,gamarnik2014correlation}, the ratio $I_2/I_1$ is assumed to be small enough, where $I_1$ can be viewed as the bias of each agent towards a specific action, and $I_2$ can be viewed as the interaction strength between neighboring agents. Put in the context of our paper, $I_1$ corresponds to the $\tau$ parameter, as the entropy regularization can be viewed as a bias towards the uniform policy, and $I_2$ corresponds to the total interactive strength $\sup_i\sum_{j\in \cN_i} C_{ij}$. As a result, our assumptions on $\tau$ and  $\sup_i\sum_{j\in \cN_i} C_{ij}$ are consistent with those in the literature. On a different note, we also have an upper bound on $\gamma<0.8$, which we believe is an artifact of the proof. In \Cref{sec:experiment}, we show our approach also works when $\gamma>0.8$.

\subsection{Algorithm Design: Localized Policy Iteration } \label{subsec:algo}

\begin{algorithm}[t]
\caption{Localized Policy Iteration ($\algname$)}\label{alg:SVI}
\begin{algorithmic}[1]
\FOR{$m = 0, 1, 2, \cdots$}
    \STATE Sample initial global state $s(0) \sim \initialState$.\label{alg:SAC_Outline:Critic_Start}
    
    \hrulefill\\ \textit{\color{gray}Policy Evaluation} 
    \FOR{$t = 0, 1, \cdots, T$}
        \STATE Each agent $i$ takes action $a_i(t) \sim \hat{\zeta}_i^m(\cdot \mid s_{\nik}(t))$ to obtain the next global state $s(t+1)$. \label{algo:run_policy}
        \STATE Each agent $i$ records $s_{\nib}(t), a_{\nib}(t)$, $r_i(t):=r_i(s_i(t),a_i(t))$. 
        \label{algo:collect_trajectory}
    \ENDFOR
    \STATE Each agent $i$ conducts policy evaluation subroutine to estimate its truncated local $Q$-function\\ $\hat{Q}_i^m \leftarrow \policyevaluation(\beta,\{s_{\nib}(t), a_{\nib}(t), r_i(t)\}_{t=0}^T, \hat{\zeta}_i^m)$. \label{algo:policy_evaluation_return}
    \hrulefill\\ \textit{\color{gray}Soft Policy Improvement} 
    \FOR{$p = 0, 1, \cdots, p_{\max}$} 
    \label{alg:policy_improvement_start}
        \STATE Each agent $i$ runs the following iterative procedure to calculate a policy, where $\hat{\pi}_i^{0}$ is initialized at a uniformly random policy. \label{algo:policy_improvement}\\
        For all $(a_{\nik},s_{\nik})\in\mathcal{A}_{\nik}\times \mathcal{S}_{\nik}$, update %
        \begin{align}
        \mathcal{Q}^i(a_i,a_{\nik/i},s_{\nik}) & =    \frac{1}{n}\sum_{j\in \nik} \hat{Q}_j^m([\ext{s_{\nik}}]_{\njb},[\ext{a_{\nik}}]_{\njb \backslash i},a_i),\label{eq:algo_local_aggregate_q}\\
        \hat{\pi}_i^{p+1}(a_i\mid s_{\nik}) &=  \frac{\big(\hat{\pi}_i^{p}(a_i \mid s_{\nik})\big)^{1 - \eta \tau}}{Z_i^{p}(s_{\nik})} \exp\Big(\eta \mathbb{E}_{a_{j}\sim\hat{\pi}_{j}^{p}([\ext{s_{\nik}}]_{\njk}),j\in \nik/\{i\}}\big[\mathcal{Q}^i(a_i,a_{\nik/i},s_{\nik})\big]\Big), \label{eq:algo_policyimprovement} 
        \end{align}%
        where $Z_i^{p}(s_{\nik})$ is the normalization term.
    \ENDFOR \label{alg:policy_improvement_end}
    \STATE Each agent sets $\hat{\zeta}_i^{m+1} \leftarrow \hat{\pi}_i^{p_{\max}+1}$.  \label{algo:policy_improvement_return}
\ENDFOR
\end{algorithmic}
\end{algorithm}

The details of \algname\ are presented in \Cref{alg:SVI}.  It is a policy iteration style algorithm where each agent updates a $\kappa$-hop policy $\hat{\zeta}_i^m$, with $m$ being the iteration counter. Within each iteration, the algorithm is divided into two major steps: policy evaluation and policy improvement.  Each is described in detail below, followed by a discussion of the communication and computation requirements of the algorithm.

\textbf{Policy evaluation.} Each agent implements the current policy (Line \ref{algo:run_policy}) to collect samples (Line \ref{algo:collect_trajectory}). Then, in Line~\ref{algo:policy_evaluation_return}, each agent $i$ estimates a \emph{truncated} local $Q$ function (defined in \eqref{eq:truncated_q_def}). Such a truncated local $Q$-function is much smaller in dimension than the full local $Q$-function, and the error caused by the truncation is small (cf. \Cref{prop:truncated_q_error}). We note that Line~\ref{algo:policy_evaluation_return} uses a subroutine \policyevaluation, which we leave unspecified for now because our algorithm can accommodate many popular policy evaluation schemes such as Temporal Difference (TD) learning. Further, our final convergence guarantee holds as long as the policy evaluation subroutine satisfies a exactness property (\Cref{def:beta-hop-state-aggregation-policy-eva}) to be defined in \Cref{sec:convergence}. Moreover, in \Cref{subsec:policy_evaluation}, we provide a version of TD learning as the \policyevaluation\ subroutine that can learn the truncated local $Q$ functions and satisfy \Cref{def:beta-hop-state-aggregation-policy-eva}.

\textbf{Policy improvement.} The policy improvement step runs from Line~\ref{alg:policy_improvement_start} to Line~\ref{alg:policy_improvement_end}. It is an iterative procedure with $p_{\max}$ iterations, and each agent updates a $\kappa$-hop policy $\hat{\pi}_i^p$. The core step is for the agents to collectively implement \eqref{eq:algo_local_aggregate_q} and \eqref{eq:algo_policyimprovement}, where each agent $i$ sets $\hat{\pi}_i^{p+1}$ to be the softmax policy according to $\mathcal{Q}^i$. We explain this step in detail below.

 \textit{Calculating $\mathcal{Q}^i$ in \eqref{eq:algo_local_aggregate_q}.} $\mathcal{Q}^i$ is an \emph{locally aggregated $Q$ function}, where it averages over the truncated local $Q$ functions of agent $i$'s $\kappa$ hop neighborhood: $\{\hat{Q}_j\}_{j\in\nik}$. We note that $\mathcal{Q}^i$ only depends on the states and actions of of $i$'s $\kappa$-hop neighborhood $s_\nik,a_\nik$, but in the mean while $\hat{Q}_j$ depends on the state and action of node $j$'s $\beta$-hop neighborhood, which might not be in $\nik$. Therefore, when evaluating $\hat{Q}_j$, for the nodes $\ell \in \njb/\nik $, we use a default value for their state and action denoted as $s^{\mathrm{default}}_{\ell},a^{\mathrm{default}}_{\ell}$. More formally, we define the following extension operator: 
     \begin{definition}[Extension Operator]\label{def:ext_operator}
Define a state tuple $s_{\mathcal{J}} = (s_j)_{j\in\mathcal{J}}\in\cS_{\mathcal{J}}$ for a subset of agents $\mathcal{J}\subseteq\cN$. We define a corresponding global state $\ext{s_{\mathcal{J}}}$ based on $s_{\mathcal{J}}$:
\begin{align}
    [\ext{s_{\mathcal{J}}}]_j=
    \begin{cases}
   s_j&\text{if } j\in\mathcal{J},\\
    s^{\mathrm{default}}_j&\text{if} j\not\in\mathcal{J}.
    \end{cases}
\end{align}
 where $s^{\mathrm{default}}$ is a default global state. The same notation is used to extend local actions to global actions.
\end{definition}
\noindent We note that, throughout this paper, the default state $s^{\mathrm{default}}$ is fixed and can be any state tuple in $\mathcal{S}$. Its value does not affect our final guarantee. With \Cref{def:ext_operator}, we can write the evaluation of $\hat{Q}_j $ as $\hat{Q}_j( [\ext{s_{\nik}}]_\njb,  [\ext{a_{\nik}}]_\njb)$. This gives rise to \eqref{eq:algo_local_aggregate_q}. The rationale of calculating $\mathcal{Q}^i$ in this way is that, as shown later in the proof, we have for any $s,a_{-i}, a_i,a_i'$
\begin{align*}
    \mathcal{Q}^i(a_i,a_{\nik/i},s_{\nik}) - \mathcal{Q}^i(a_i',a_{\nik/i},s_{\nik}) \approx \hat{Q}^m(a_i, a_{-i}, s) -  \hat{Q}^m (a_i', a_{-i}, s)
\end{align*}
where $\hat{Q}^m(s,a) = \frac{1}{n} \sum_{j\in\mathcal{N}}\hat{Q}_j^m(s_{\njb},a_{\njb}) $, the average of the truncated local $Q$ functions for iteration $m$. In other words, $\mathcal{Q}^i$ is a good estimate of $\hat{Q}^m$ (up to a function that does not depend on $a_i$), which in turn is a good estimate of the true $Q$ function $Q^{\hat{\zeta}^m}$. Therefore, we can then conduct a multiplicative weights policy update based on $\mathcal{Q}^i$, which we discuss next. 

\emph{Multiplicative weights policy update \eqref{eq:algo_policyimprovement}.} The basic idea of \eqref{eq:algo_policyimprovement} is for each agent $i$ to update its policy using the multiplicative weights algorithm, with $\mathcal{Q}^i$ as the score for each action $a_i$. Note that $\mathcal{Q}^i$ not only depends on $a_i$, but also $a_{\nik/i}$, so when conducting the update, we take the expectation of $a_j$ using the current policy of agent $j$, $\hat{\pi}_j^p$. When doing so, $\hat{\pi}_j^p$ depends on states that are outside $\nik$, and for these states, we set it to a default value (the same as the default value used in the calculation of $\mathcal{Q}^i$). We show later in the proof (\Cref{lem:policy_improvement_error}) that the update \eqref{eq:algo_policyimprovement} is approximately solving the maximization in the Bellman optimal operator \eqref{eq:bellman_operator}.

\textbf{Computation and communication.} The major computational burden of $\algname$ lies in the steps that compute the truncated local $Q$ functions and the $\kappa$-hop polices. The complexity of these scales with the largest state-action space size of the $\kappa$ and $\beta$-hop neighborhood in the network. Therefore, our algorithm can avoid the exponential computation burden associated with the exponentially large state and action spaces. 

In terms of communication, each agent only needs local communication with agents in the $\max(\beta,\kappa)$-hop neighborhood, as opposed to centralized communication. Specifically, in Line~\ref{algo:collect_trajectory} each agent $i$ needs to receive information on the states and actions of agents in the $\beta$-hop neighborhood $s_{\nib}(t),a_{\nib}(t)$. In Line~\ref{algo:policy_improvement} (eq. \eqref{eq:algo_policyimprovement}), each agent $i$ needs to know $\{\hat{\pi}_j^{p}\}_{j\in\nik}$ as well as the truncated $Q$-functions $\{\hat{Q}_j^m
\}_{j\in\nik}$ for the agents in the $\kappa$-hop neighborhood.

\section{Convergence Analysis}\label{sec:convergence}

This section provides a convergence guarantee for \algname. We first present our requirements for the \policyevaluation\ subroutine in \Cref{def:beta-hop-state-aggregation-policy-eva}. Then, under these requirements, we present our main convergence result for \algname\ in \Cref{thm:convergence_svi}. Finally, in \Cref{subsec:policy_evaluation}, we provide a specific policy evaluation subroutine, \localizedtd, which is a variation of the approach proposed in \cite{lin2021multi}, and show that it can meet the \Cref{def:beta-hop-state-aggregation-policy-eva}.
\subsection{Convergence of Localized Policy Iteration}
Our requirement for the \policyevaluation\ subroutine is stated formally in \Cref{def:beta-hop-state-aggregation-policy-eva}. Intuitively, the policy evaluation process involves two parts of errors: a bias is introduced because we only estimate the truncated local $Q$ functions (cf. \eqref{eq:truncated_q_def}) using information collected within the $\beta$-hop neighborhood rather than global network. Such a truncation introduces a bias that is on the order of $O(\nu'/(\beta + 1)^{\mu})$ because of the $(\nu', \mu)$-decay property of the true local Q function (\Cref{prop:truncated_q_error}). In addition to the bias caused by the truncation, a stochastic error also exists due to the inherent stochasticity of the observed samples. Under proper assumptions, the stochastic error decays to zero as the number of samples $T$ increases. Since the bias part of the error is `inevitable', we require the policy evaluation algorithm to achieve an approximation error that is in the same order of the bias with high probability after a finite number of iterations, formally stated in \Cref{def:beta-hop-state-aggregation-policy-eva}. 

\begin{definition}\label{def:beta-hop-state-aggregation-policy-eva}
A policy evaluation algorithm is said to be $(\sigma,\nu',\mu)$-exact if the following condition holds: For any $\sigma$ regular policy $\hat{\zeta}$ whose corresponding local $Q$ functions $\{Q_i^{\hat{\zeta}}\}$ is $(\nu',\mu)$ decay, 

and $\delta \in (0, 1)$, there exists a function $c_{pe}(\delta,\sigma, \nu',\mu)$, such that for any $\beta\in\mathbb{N}$, there exits a $T_{eval}(\delta,\sigma, \nu',\mu, \beta)\in\mathbb{N}$ such that when the length of the trajectory supplied to the subroutine $T \geq T_{eval}(\delta,\sigma, \nu',\mu, \beta)$ steps, with probability at least $1 - \delta$, the following inequality holds
\[ \sup_{(s, a) \in \mathcal{S}\times \mathcal{A}}\abs{Q_i^{\hat{\zeta}}(s, a) - \hat{Q}_i(s_{\nib}, a_{\nib})} \leq \frac{c_{pe}(\delta,\sigma,\nu',\mu)}{(\beta + 1)^\mu}, \forall i\in \mathcal{N},\]
where $Q_i^{\hat{\zeta}}$ is the true local $Q$-functions under policy $\hat{\zeta}$, and $\hat{Q}_i$ is the output of the policy evaluation subroutine. When the context is clear, we may drop the parenthesis of $c_{pe}, T_{eval}$. 
\end{definition}

Using the above definition, we can now state our main convergence result for \Cref{alg:SVI}. 
\begin{theorem}\label{thm:convergence_svi}
Suppose $\gamma<0.8$, $\tau\geq 40\bar{r} \frac{4-3\gamma}{4-5\gamma}\actionmax^2 e$, and $\sum_{j\in \cN_i}C_{ij}<1/2$ for all $i\in \{1,2,\cdots,n\}$. 
Let $\nu' = \frac{4-3\gamma}{4-5\gamma}\bar{r}$, $\mu=\min\Big\{\log_2 \frac{\tau(4-5\gamma)}{40\bar{r}\actionmax^2 e(4-3\gamma)} ,\log_2 \frac{1}{2 \sup_{i\in\mathcal{N}}\sum_{j\in \cN_i}C_{ij}}\Big\}$, and $\tilde\sigma=\frac{2\nu'}{\tau n}$. Suppose the \emph{\policyevaluation}\ subroutine in \Cref{alg:SVI} is $(\tilde{\sigma},\nu',\mu)$-exact with constants $c_{pe}(\cdot), T_{eval}(\cdot)$ (cf. \Cref{def:beta-hop-state-aggregation-policy-eva}). Fixing any $\kappa$, take the algorithm parameters as $\beta = \frac{\kappa+1}{2}(\frac{2f(\kappa) c_{pe}}{\nu'})^{\frac{1}{\mu}}$, $\eta=\frac{1}{\tau}$, and $p_{\max}\geq -\log_2{\frac{4 + c_{pe}/(2^{\mu}\nu')}{3(\kappa/2+1)^\mu}}$. Then, given $M\in\mathbb{N}$ and $\delta\in(0,1)$, when the trajectory input to the \emph{\policyevaluation}\ subroutine $T \geq T_{eval}(\delta/M, \tilde{\sigma},\nu',\mu,\beta)$, we have with probability at least $1-\delta$,
\[J(\zeta^*) - J(\hat{\zeta}^M) \leq \gamma^{M} \Vert V^{\hat{\zeta}^0} - V^*\Vert_\infty + \frac{3(4-3\gamma)\bar{r} }{(1-\gamma)^2(4-5\gamma)} \bigg(4 + \frac{c_{pe} (4-5\gamma)}{2^{\mu}(4-3\gamma)\bar{r}}\bigg) \frac{1}{(\kappa/2+1)^\mu}. \]
\end{theorem}

As shown in \Cref{thm:convergence_svi}, \algname\ converges geometrically in $M$, with a steady state error on the order of $O(\frac{1}{(\kappa+1)^\mu})$, i.e., the steady state error decays polynomially in $\kappa$. This steady state error is a result of both the policy evaluation error in \Cref{def:beta-hop-state-aggregation-policy-eva}, and the fact that we are using a truncated $\kappa$-hop policy as opposed to a global policy. 

\textbf{Sample complexity. }As we discussed in \Cref{subsec:algo}, the \policyevaluation\ subroutine can be chosen to be any algorithm that satisfies \Cref{def:beta-hop-state-aggregation-policy-eva} and the sample complexity of $\algname$ depends on the specific choice. In \Cref{subsec:policy_evaluation}, we describe a specific \policyevaluation\ subroutine, $\beta$-hop Localized TD(0), that satisfies the requirement. Under this specific subroutine, 
we derive the sample complexity of $\algname$ in the following corollary.

\begin{corollary}\label{cor:convergence_with_localized_td}
Let the assumptions of \Cref{thm:convergence_svi} and \Cref{thm:beta-hop-policy-evaluation} hold. For any $\varepsilon>0$ and $\delta\in(0,1)$, when using $\beta$-hop \localizedtd\ as the \emph{\policyevaluation}\ subroutine, the sample complexity of \Cref{alg:SVI} to achieve $\varepsilon$-optimality is in the order of
\begin{align}\label{eq:sample_complexity}
\Theta\left( \frac{1}{\varepsilon^2}\frac{  f(\kappa)^2 \log(1/\varepsilon) \log [f(\beta) \log(1/\varepsilon)] }{\xi(\tilde\sigma,\beta)^2}\right), 
\end{align} 
where $\xi(\tilde{\sigma},\beta)$ is defined in \Cref{network-assp:geometric-mixing} and means the smallest probability the state-action pairs in $\beta$-hop neighborhoods are visited.
\end{corollary}
Based on the convergence result in \Cref{thm:convergence_svi}, to find an $\varepsilon$-optimal policy, we only need to set $M = \Theta(\log(1/\varepsilon))$ and $\kappa = \Theta((1/\varepsilon)^{1/\mu} )$, and correspondingly $\beta = \Theta(\kappa (f(\kappa))^{1/\mu} )$. According to \eqref{eq:T_eval} that we show later in \Cref{thm:beta-hop-policy-evaluation}, the iteration complexity of  $\beta$-hop \localizedtd\ is $T_{eval} = \Theta( \frac{(\beta+1)^{2\mu}  \log (f(\beta)  M/\delta)}{\xi(\tilde{\sigma},\beta)^2} ) = \Theta(\frac{ (1/\varepsilon)^2 f(\kappa)^2 \log [f(\beta) \log(1/\varepsilon)] }{\xi(\tilde\sigma,\beta)^2})  $. Thus the sample complexity of $\algname$ follows the simple calculation $MT_{eval}$ and is given in \eqref{eq:sample_complexity}. In all the $\Theta(\cdot)$ notations above, we only keep the dependence on the network size $n$ and the $\varepsilon$ parameter, which involves $\varepsilon$ itself and by extension, the $\kappa,\beta,M$ parameters. 

In this sample complexity bound, we have the $\frac{1}{\varepsilon^2}$ factor that is standard in the RL literature. In addition, instead of depending on the global-state action space size (exponential in $n$), we have a square dependence on $\frac{1}{\xi(\tilde{\sigma},\beta)}$, where $\xi(\tilde{\sigma},\beta)$ is defined in \Cref{network-assp:geometric-mixing} and means the smallest probability the state-action pairs in $\beta$-hop neighborhoods are visited. As such, $\frac{1}{\xi(\tilde{\sigma},\beta)}$ scales with the largest state action space size of the $\beta$-hop neighborhoods, as opposed to the entire network. Therefore, $\algname$ is much more scalable and implementable than other methods that enjoy global convergence such as centralized tabular RL methods \citep{bertsekas1996neuro}. This advantage is especially pronounced in sparse networks where $\beta$-hop neighborhoods are much smaller than the entire network.
Lastly, beyond the $\frac{1}{\varepsilon^2}$ factor and $\frac{1}{(\xi(\tilde\sigma,\beta))^2}$ factor and a few $\log$ factors, our bound also contains a factor $f(\kappa)^2$, where we recall $f(\kappa)$ is the size of the largest $\kappa$-hop neighborhood. This factor reflects the complexity of the $\kappa$-hop policy we try to learn.

\subsection{Policy Evaluation via $\beta$-hop Localized TD(0)} \label{subsec:policy_evaluation}

\begin{algorithm}[ht]
        \caption{$\beta$-hop Localized TD(0) (Agent $i$)}\label{alg:beta-hop-TD-0}
        \begin{algorithmic}[1]
        \REQUIRE Parameter $\beta$, a sequence of $\beta$-hop state, $\beta$-hop action, and local reward tuples $\{s_{\nib}(t), a_{\nib}(t), r_i(t)\}$, local policy $\hat{\zeta}_i^m$, and a sequence of learning rate $\{\alpha_t\}$.
        \STATE Initialize $\hat{Q}_i^0$ to be all zero vector.
        \STATE Set $\hat{r}_i(t) \gets r_i(t) - n \tau \mathbb{E}_{a_i \sim \hat{\zeta}_i^m(\cdot \mid s_{\nik}(t))} \log \hat{\zeta}_i^m(a_i \mid s_{\nik}(t))$.
        \FOR{$t = 1, \cdots, T$}
            \STATE Update the local estimation $\hat{Q}_i$ with step size $\alpha_{t-1}$,
            {\begin{equation*}
                \begin{aligned}
                &\hat{Q}_i^t \big(s_{\nib}(t-1), a_{\nib}(t-1)\big) =\\
                &(1 - \alpha_{t-1})\hat{Q}_i^{t-1} \big(s_{\nib}(t-1), a_{\nib}(t-1)\big) + \alpha_{t-1}\big(\hat{r}_i(t) + \gamma \hat{Q}_i^{t-1} \big(s_{\nib}(t), a_{\nib}(t)\big)\big),\\
                &\hat{Q}_i^t \big(s_{\nib}, a_{\nib}\big) = \hat{Q}_i^{t-1} \big(s_{\nib}, a_{\nib}\big)\text{ for }\big(s_{\nib}, a_{\nib}\big) \not = \big(s_{\nib}(t-1), a_{\nib}(t-1)\big).
                \end{aligned}
            \end{equation*}}
            \label{alg:SAC_Outline:Critic_End}
        \ENDFOR
        \STATE For every $(s_{\nib}, a_{\nib})$ pair, set $\hat{Q}_i^T(s_{\nib}, a_{\nib}) \gets \hat{Q}_i^T(s_{\nib}, a_{\nib}) + n \tau \mathbb{E}_{a_i \sim \hat{\zeta}_i^m(\cdot \mid s_{\nik})} \log \hat{\zeta}_i^m(a_i \mid s_{\nik})$ and return $\hat{Q}_i^T(\cdot)$
        \end{algorithmic}
    \end{algorithm}

To provide a concrete example of the \policyevaluation\, subroutine, we introduce and analyze \localizedtd\ in this subsection. The pseudo-code of \localizedtd\ is given in \Cref{alg:beta-hop-TD-0}. It is a variation of a policy evaluation subroutine used in \cite{lin2021multi}, where each agent conducts temporal difference learning using the states and actions of its local $\beta$-hop neighborhood. Note that compared to \cite{lin2021multi}, we also add entropy regularization in the update (see Line \ref{alg:SAC_Outline:Critic_End} of \Cref{alg:beta-hop-TD-0}).  

In this section, we show that \Cref{alg:beta-hop-TD-0} satisfies \Cref{def:beta-hop-state-aggregation-policy-eva} by utilizing the convergence results in  \cite{lin2021multi}. We first state the assumption needed. 

\begin{assumption}\label{as:MC}
	There exists a policy $\zeta_b=(\zeta_{b,1},\zeta_{b,2},\cdots,\zeta_{b,n})\in\Delta_{\textrm{policy}}$ such that the Markov chain $\{s(t)\}_{t\geq 0}$ induced by $\zeta_b$ is irreducible and aperiodic.
\end{assumption}
Based on \Cref{as:MC}, we show that all $\sigma$-regular policies are sufficiently explorative. 
\begin{lemma}\label{network-assp:geometric-mixing}
Define $\mathcal{Z} := \mathcal{S} \times \mathcal{A}$. Under \Cref{as:MC}, for any $\sigma$-regular policy $\hat{\zeta}$, the induced Markov chain $\{z(t) := \left(s(t), a(t)\right)\}$ is irreducible and aperiodic, hence admits a unique stationary distribution $d^{\hat{\zeta}} \in \Delta_{\mathcal{Z}}$ with strictly positive components. Further, there exists positive constants $K_1(\sigma)$ and $K_2(\sigma)\geq 1$ depending on $\sigma$ such that 
\[\forall z' \in \mathcal{Z}, \forall t \geq 0, \sup_{\mathcal{K} \subseteq \mathcal{Z}}\Bigg|\sum_{z \in \mathcal{K}}d_z^{\hat{\zeta}} - \sum_{z \in \mathcal{K}}\mathbb{P}(z(t) = z \mid z(0) = z')\Bigg| \leq K_1 e^{-t/K_2},\]
Further, given any $\beta\in\mathbb{N}$, for each agent $i$ and $z' \in \mathcal{Z}_{\nib}$, define $d^{\hat{\zeta}}_{i, \beta}(z') = \sum_{z \in \mathcal{Z}: z_{\nib} = z'} d^{\hat{\zeta}}(z)$, which is the marginal stationary distribution for the state-actions in $\nib$, and define $\xi(\sigma,\beta) := \inf_{i \in \mathcal{N}, z' \in \mathcal{Z}_{\nib}} d^{\hat{\zeta}}(z')$ as the minimum probability in this distribution. 
\end{lemma}

\Cref{network-assp:geometric-mixing} is an immediate implication of a result on the sufficient exploration of non-deterministic policies stated in \Cref{ap:exploration}. Note that the only assumption we need here is that there exists a policy $\zeta_b$ such that the induced Markov chain $\{S_k\}$ is irreducible and aperiodic (\Cref{as:MC}). This is in contrast with the analysis in existing literature, where  \Cref{network-assp:geometric-mixing} is directly assumed to hold \citep{qu2020scalablelocal}. Therefore, our proof of \Cref{network-assp:geometric-mixing} is of independent interest in the literature of policy evaluation.

\Cref{network-assp:geometric-mixing} enables us to apply Theorem 3.2 in \cite{lin2021multi} to show that  $\beta$-hop Localized TD(0)  is a $(\sigma, \nu',\mu)$-exact policy evaluation algorithm.

\begin{theorem}\label{thm:beta-hop-policy-evaluation}
Suppose the policy $\hat{\zeta}$ we want to evaluate is $\sigma$ regular and under the policy, the local $Q$-functions satisfy the $(\nu', \mu)$-decay property. Suppose $K_1(\sigma), K_2(\sigma), \xi(\sigma,\beta)$ are the constants in \Cref{network-assp:geometric-mixing}. Recall that the local reward at every time step is upper bounded by $\bar{r}$. Let the step size of \Cref{alg:beta-hop-TD-0} be $\alpha_t = \frac{H}{t + t_0}$ with $t_0 = \max(4H, 2K_2 \log T)$, and $H = \frac{2}{(1 - \gamma) \xi(\sigma,\beta)}$. Then,  \Cref{alg:beta-hop-TD-0} is $(\sigma,\nu',\mu)$-exact with constant $c_{pe}(\delta,\sigma,\nu',\mu) = \frac{2 \nu'}{1 - \gamma}$ and $T_{eval}(\delta,\sigma,\nu',\mu, \beta)$ is upper bounded by{\small
\begin{align}
    \tilde{O}\Bigg(\frac{(\beta + 1)^{2\mu} (\bar{r} + \tau \log \actionmax)^2 K_2(\sigma) \log(f(\beta)K_2(\sigma)/\delta)}{(\nu')^2(1 - \gamma)^4(\xi(\sigma,\beta))^2} + \frac{(\beta + 1)^\mu (\bar{r} + \tau \log \actionmax) (K_1(\sigma) + 1)(K_2(\sigma) + 1)}{\nu'(1 - \gamma)^2(\xi(\sigma,\beta))^2}\Bigg), \label{eq:T_eval}
\end{align}}
where we recall $f(\beta) = \sup_{i \in \mathcal{N}} \big|\mathcal{N}_i^\beta\big|$ is the size of the largest $\beta$-hop neighborhood.
\end{theorem}

\section{Convergence Proof for Localized Policy Iteration}
We now prove our main results for the convergence of $\algname$. The key technical idea underlying our analysis is a novel closure property for a class of policies with spatially decaying properties.  We introduce this property first, and then apply it to prove the results in the previous sections. 

\subsection{Key Idea: Closure of Decay Policy Class under Policy Iteration} \label{subsec:keyidea}
We first present a prototype algorithm in \Cref{alg:prototype}, which is an exact policy iteration algorithm that runs exact policy evaluation followed by exact policy improvement to update the global policy. While this algorithm is not practical, it is useful in developing a generic framework to analyze the convergence of $\algname$ (\Cref{alg:SVI}). %
For that purpose, in this subsection we prove an important ``closure'' property of the prototype algorithm (\Cref{alg:prototype}): when starting at an initial policy in a decay policy class $\Delta_{\nu,\mu,\sigma}$ (defined in~\Cref{def:mu_decay_policy}), the policy at every iteration of the prototype algorithm remains in the same $\Delta_{\nu,\mu,\sigma}$ class. This property  provides the foundation of our proof for both \Cref{thm:subopt_dist_policy} and \Cref{thm:convergence_svi}. %

\begin{algorithm}[ht]
\caption{Exact Policy Iteration (Prototype)}\label{alg:prototype}
\begin{algorithmic}[1]
\STATE \textbf{Input:} initial global policy $\zeta_0$.
\FOR{$m = 0, 1, 2, \cdots, M-1$}
    \STATE Calculate the $\Qtau$-function for policy $\zeta^{m}$ as $\Qtau^{\zeta^{m}}$
    \STATE Updates the global policy $\zeta^{m+1}=(\zeta_{1}^{m+1},\ldots,\zeta_{n}^{m+1})$, where
    \begin{align*}
    \{\zeta_{i}^{m+1}(\cdot|s)\}_{i\in\cN}=\argmax_{\pi_1(\cdot|s)\in \Delta_{\mathcal{A}_1}, \ldots, \pi_n(\cdot|s)\in \Delta_{\mathcal{A}_n}} \mathbb{E}_{a_1\sim \pi_1(\cdot|s),\ldots, a_n\sim \pi_n(\cdot|s)}\bigg[  \Qtau^{\zeta^{m}}(s,a) -  \sum_{i=1}^n \tau \log \pi_i(a_i|s)  \bigg].
    \end{align*}
\ENDFOR
\STATE \textbf{Output:} $\zeta^{M}$
\end{algorithmic}
\end{algorithm}

\textbf{Closure of $\Delta_{\nu,\mu,\sigma}$ under \Cref{alg:prototype}.} In \Cref{thm:convergence_prototype}, we formally establish the property that when starting with $\zeta^0\in\Delta_{\nu,\mu,\sigma}$,  the iterates in \Cref{alg:prototype}  remain in the policy class $\Delta_{\nu,\mu,\sigma}$. The proof of \Cref{thm:convergence_prototype} can be found in \Cref{sec:proof_of_closure}.

\begin{theorem}\label{thm:convergence_prototype}
Suppose $\gamma<0.8$ and $\tau \geq 6\bar{r} \frac{4-3\gamma}{4-5\gamma}\actionmax^2 e$, and $\forall i, \sum_{j\in \cN_i}C_{ij}<1/2$. Define
\begin{align*}
\nu = 1/2, \quad \mu= \min\bigg\{\log_2 \frac{\tau(4-5\gamma)}{6\bar{r}\actionmax^2 e(4-3\gamma)} ,\log_2 \frac{1}{2 \sup_{i\in\mathcal{N}}\sum_{j\in \cN_i}C_{ij}}\bigg\}, \quad \sigma = \frac{\bar{r}(4-3\gamma)}{(4-5\gamma)n\tau}.
\end{align*}
Then, the following holds. 
\begin{itemize}
    \item [(a)] When initial policy satisfies $\zeta^0\in\Delta_{\nu,\mu,\sigma}$, all iterates $\zeta^m$ in \Cref{alg:prototype} will be in $\Delta_{\nu,\mu,\sigma}$. 
    \item [(b)] $Q^{\zeta^m}$ is $(\nu',\mu)$-decay with $ \nu' = \bar{r} \frac{4-3\gamma}{4-5\gamma}$. 
    \item [(c)] $V^{\zeta^m}$ converges to the unique fixed point $V^*$ of $\mathcal{T}$ with geometric rate, i.e. $\Vert V^{\zeta^m} - V^* \Vert_\infty \leq \gamma^{m} \Vert V^{\zeta^0} - V^* \Vert_\infty$.
\end{itemize}
\end{theorem}

The results in \Cref{thm:convergence_prototype} demonstrate that the defined policy class $\Delta_{\nu,\mu,\sigma}$ is closed under the policy iteration operator in \Cref{alg:prototype}. This key observation provides a pathway for proving our main theoretical results, as we explain below.  

\textbf{Pathway to prove \Cref{thm:subopt_dist_policy}.} Based on the closure of the policy class $\Delta_{\nu,\mu,\sigma}$, the optimal policy is also in $\Delta_{\nu,\mu,\sigma}$. We show that if we ``truncate'' the optimal policy to a $\kappa$-hop policy, the performance loss is on the order of $O(\frac{1}{(\kappa+1)^\mu})$. This directly leads to a proof of \Cref{thm:subopt_dist_policy}, which we detail in \Cref{subsec:proof_dist_policy_subopt}. 

\textbf{Pathway to prove \Cref{thm:convergence_svi}.} A second consequence of \Cref{thm:convergence_prototype} is that we can view \Cref{alg:SVI} as an inexact version of the prototype algorithm (\Cref{alg:prototype}), which reduces the complexity in implementation, communication, and training. Specifically, \Cref{alg:SVI} introduces the following types of errors: 
\begin{itemize}[leftmargin=*]
    \item \textit{Error caused by truncated $Q$-functions}: In \Cref{alg:SVI}, we only learn $\beta$-hop truncated versions of the $Q$-functions. By \Cref{thm:convergence_prototype}, the true $Q$-functions are $(\nu',\mu)$-decay, and as a result, the truncation causes an error on the order of $O(\frac{1}{(\beta+1)^\mu} )= O(\frac{1}{(\kappa+1)^\mu})$ (Noticing $\kappa\leq \beta$). 
    \item \textit{Error caused by truncated policies}: We only learn $\kappa$-hop truncated policies, which also causes an $O(\frac{1}{(\kappa+1)^\mu})$ error. 
    \item \textit{Error caused by inexact policy evaluation}: Since we use finite samples to estimate the truncated $Q$-functions, this causes statistical errors depending on the sample size. However, with high probability, this error diminishes to $0$ as the number of samples $T$ increases. 
    \item \textit{Error caused by inexact policy improvement}: In \Cref{alg:SVI}, we cannot directly implement the $\arg\max$ procedure as in \Cref{alg:prototype}. Instead, we use an iterative procedure \eqref{eq:algo_policyimprovement}, which we later show corresponds to the mirror descent algorithm for solving the $\arg\max$. This causes an optimization error, which diminishes to $0$ as the number of iterations $p_{\max}$  for \eqref{eq:algo_policyimprovement} increases. 
\end{itemize}
As discussed above, the first two types of errors can be bounded by the order of $O(\frac{1}{(\kappa+1)^\mu})$ which depends on the size of the neighborhood in our localized approximation of the policy and the $Q$ functions, while the last two types of errors can be made arbitrarily small as long as the algorithm is run for sufficiently many steps. This eventually leads to a proof of \Cref{thm:convergence_svi}, which we detail in \Cref{subsec:convergence_proof}.

\subsection{Proof of \Cref{thm:subopt_dist_policy}: Near-optimality of $\kappa$-hop Policies}\label{subsec:proof_dist_policy_subopt}

As discussed in \Cref{subsec:keyidea}, the optimal policy $\zeta^*$ is a $\sigma$-regular $(\nu,\mu)$-decay policy with the values of $\nu,\mu,\sigma$ given by the condition in \Cref{thm:convergence_prototype}. To prove \Cref{thm:subopt_dist_policy}, we simply truncate the optimal policy to the $\kappa$-hop neighborhood. Specifically, the ``truncated'' policy is defined as
\begin{align}
    \hat{\zeta}_i^*(\cdot|s_{\nik}) = \zeta_i^*(\cdot|s_{\nik},s^{\mathrm{default}}_{\nminusik}), \forall i.
\end{align}
For any $i\in\mathcal{N}$, $s\in\mathcal{S}$, for the simplicity of notations, we re-order the set $\cN$ as follows
\begin{align*}
\{l_1,\ldots,l_{|\nminusik|},l_{|\nminusik|+1},\ldots,l_n\},
\end{align*}
where $\{l_1,\ldots,l_{|\nminusik|}\}=\nminusik$
and $\{l_{|\nminusik|+1},\ldots,l_n\}=\nik$. We also use $\cN_{[l_j]}$ to denote $\{l_1,\ldots,l_j\}$ and $\cN/\cN_{[l_j]}$ to denote $\{l_{j+1},\ldots,l_n\}$.  Then, using the triangle inequality we have
\begin{align}\label{eq:tv_dist_triangle_decompose}
    \TV(\hat{\zeta}_i^*(\cdot|s_{\nik}),\zeta_i^*(\cdot|s))&\leq    \sum_{j=|\nminusik|}^{1}\TV(\zeta_i^*(\cdot|s_{\cN_{[l_j]}}^{\mathrm{default}},s_{\cN/\cN_{[l_j]}}),\zeta_i^*(\cdot|s_{\cN_{[l_{j-1}]}}^{\mathrm{default}},s_{\cN/\cN_{[l_{j-1}]}})),
\end{align}
where we define $\cN_{l_0}=\Phi$ to be the empty set, which implies for $j=1$, $\cN/\cN_{[l_{j-1}]}=\cN$ and $s_{\cN/\cN_{[l_{j-1}]}}=s$ is the global state. By the definition of the interaction matrix in \Cref{def:mu_decay_policy}, we immediately have
\begin{align} \label{eq:proof_thm1:tv_bound}
 \sup_{i\in\mathcal{N}}\sup_{s\in\mathcal{S}}   \TV(\hat{\zeta}_i^*(\cdot|s_{\nik}),\zeta_i^*(\cdot|s))
 & \leq \sup_{i\in\cN}\sum_{j\in \cN_{-i}^\kappa } Z^{\zeta^*}_{ij} \nonumber \\
 &\leq \frac{1}{(\kappa+1)^\mu} \sup_{i\in\cN}\sum_{j\in \cN_{-i}^\kappa } (\dist(i,j)+1)^\mu Z_{ij}^{\zeta^*}\nonumber\\
 &\leq \frac{\nu}{(\kappa+1)^\mu},
\end{align}
where the second inequality is due to \eqref{eq:mu_decay_imply_distance_decay} and the last inequality holds because $\zeta^*$ is in the $(\nu,\mu)$-decay class. This further indicates $\hat{\zeta}_i^*$ is $O(\frac{1}{(\kappa+1)^\mu})$ close to $\zeta_i^*$. 

To proceed, we introduce the following performance difference bound that bounds the value functions of two policies by the $\TV$ distance between the two policies: 
\begin{lemma}\label{lem:performance_diff_tv}
Suppose $\zeta,\tilde{\zeta}\in\Delta_{\mathrm{policy}}$ are  $\sigma$-regular, and further, $Q^{\tilde{\zeta}}$ is $(\nu',\mu)$-decay. Then, we have 
\begin{align}
   \Vert V^\zeta-V^{\tilde{\zeta}}\Vert_\infty\leq \frac{1}{1-\gamma} \bigg(\tau\sigma + \frac{\nu'}{n}\bigg) \sum_{i=1}^n \sup_{s\in\cS}\TV(\tilde\zeta_i(\cdot|s),\zeta_i(\cdot|s)).
\end{align}
\end{lemma}
With the above \Cref{lem:performance_diff_tv}, we have
\begin{align*}
    J(\zeta^*) - J(\hat{\zeta}^*) &= \mathbb{E}_{s \sim\initialState}\left( V^{\zeta^*}(s) - V^{\hat{\zeta}^*}(s)\right)
    \leq \Vert V^{\zeta^*} - V^{\hat{\zeta}^*}\Vert_\infty \\
    &\leq \frac{1}{1-\gamma} \bigg(\tau\sigma + \frac{\nu'}{n}\bigg) \sum_{i=1}^n \sup_{s\in\cS}\TV(\zeta_i^*(\cdot|s),\hat{\zeta}_i^*(\cdot|s))
    \leq \frac{1}{1-\gamma} (n\tau\sigma + \nu') \frac{\nu}{(\kappa+1)^\mu}\\
    &\leq {\frac{\bar{r} (4-3\gamma)}{(1-\gamma)(4-5\gamma)} \frac{1}{(\kappa+1)^\mu},} 
\end{align*}
where in the last inequality, we have plugged in { $n\tau\sigma = \nu' = \bar{r} \frac{4-3\gamma}{4-5\gamma}$ and $\nu=1/2$. } 
This demonstrates that the $\kappa$-hop policy $\hat{\zeta}^*$ is $O(\frac{1}{\kappa^\mu})$-optimal, and setting $\hat{\zeta}^{*,\kappa}$ as $\hat{\zeta}^*$ concludes the proof of \Cref{thm:subopt_dist_policy}.

\subsection{Proof of \Cref{thm:convergence_svi}: Convergence Analysis of Localized Policy Improvement} \label{subsec:convergence_proof}

In this section, we prove the convergence of  $\algname$. 
Our proof uses an induction argument, and the induction assumption is that, at the $m$-th iteration of \Cref{alg:SVI}, the policy $\hat{\zeta}^m$ is $\tilde\sigma$ regular and its local $Q$-functions $\{ Q_i^{\hat{\zeta}^m}\}_{i=1}^n$ is $(\nu',\mu)$-decay, with constants $\tilde\sigma,\nu',\mu$ given in \Cref{thm:convergence_svi}'s statement. This induction assumption is clearly true for $m=0$ because of \Cref{lem:policy-q} and the fact that $\hat{\zeta}^0$ is a uniform policy. 
 
Now under this induction assumption, we first show that the local $Q$ functions $\hat{Q}_i^m$ returned by the policy evaluation step (i.e., the $\policyevaluation$ subroutine in Line~\ref{algo:policy_evaluation_return}) of \Cref{alg:SVI} is a good approximation of the true local $Q$-functions $Q_i^{\hat{\zeta}^m}$ of the policy $\hat{\zeta}^m$. 

More specifically, recall that we assumed the policy evaluation step in \Cref{alg:SVI}  is $(\tilde{\sigma},\nu',\mu)$-exact as defined in \Cref{def:beta-hop-state-aggregation-policy-eva}. This means for any given $\delta \in (0, 1)$ and $\beta\in\mathbb{N}$, there exists constants $T_{eval}(\frac{\delta}{M},\tilde{\sigma},\nu',\mu, \beta), c_{pe}(\frac{\delta}{M},\tilde{\sigma},\nu',\mu)$ such that when the input trajectory to  $\policyevaluation$ has length $T \geq T_{eval}(\frac{\delta}{M}, \tilde{\sigma},\nu',\mu,\beta)$ steps, its output $\hat{Q}_i^m$ satisfies
\begin{align}\label{lem:policy_evaluation_error}
   \forall i, \sup_{(s, a) \in \mathcal{S}\times \mathcal{A}}\abs{Q_i^{\hat{\zeta}^m}(s, a) - \hat{Q}_i^m(s_{\nib}, a_{\nib})} \leq \frac{c_{pe}(\frac{\delta}{M},\tilde{\sigma},\nu',\mu)}{(\beta + 1)^\mu}:=\epsilon
\end{align}
with probability at least $1 - \frac{\delta}{M}$. In the rest of the proof, we  write $c_{pe}(\frac{\delta}{M},\tilde{\sigma},\nu',\mu)$ as $c_{pe}$ for simplicity. 

Given \eqref{lem:policy_evaluation_error}, we next show that the policy improvement step (Line \ref{algo:policy_improvement_return}) of \Cref{alg:SVI}  returns a $\tilde\sigma$ regular policy $\hat\zeta^{m+1}$ that is $O(\frac{1}{\kappa^\mu})$ close to the policy given by the optimal Bellman operator, and further $\{Q_i^{\hat\zeta^{m+1}}\}$ is $(\nu',\mu)$ decay. 
More precisely, we present the following \Cref{lem:policy_improvement_error}, the full proof of which can be found in \Cref{sec:proof_policy_improvement_error}. 

\begin{lemma}\label{lem:policy_improvement_error} Consider the settings of  \Cref{thm:convergence_svi}. Suppose at iteration $m$, the local $Q$-functions $\{ Q_i^{\hat{\zeta}^m}\}_{i=1}^n$ is $(\nu',\mu)$-decay and \eqref{lem:policy_evaluation_error} is true. %
Then, the policy improvement part (Line \ref{algo:policy_improvement_return}) will return a policy $\hat{\zeta}^{m+1}$ that satisfies
\begin{align}
     \mathcal{T} \Vtau^{\hat{\zeta}^m} -  \Vtau^{\hat{\zeta}^{m+1}} \leq  \frac{3(4-3\gamma)\bar{r} }{(1-\gamma)(4-5\gamma)} \frac{4 + c_{pe} (4-5\gamma)/(2^{\mu}(4-3\gamma)\bar{r})}{(\kappa/2+1)^\mu} \mathbf{1} := \frac{c_{pi}}{(\kappa/2+1)^\mu} \mathbf{1}, \label{eq:policy_improvement_error}
\end{align}
where $\mathcal{T}$ is the global Bellman optimal operator with entropy regularization defined in \eqref{eq:bellman_operator}. Further, $\hat{\zeta}^{m+1}$ is $\tilde\sigma$ regular and $Q^{\hat{\zeta}^{m+1}}$ is $(\nu',\mu)$-decay. 
\end{lemma}

Note that \Cref{lem:policy_improvement_error} implies that $\hat{\zeta}^{m+1}$ is $\tilde\sigma$-regular and $Q^{\hat{\zeta}^{m+1}}$ is $(\nu',\mu)$-decay. Therefore, conditioned on the induction assumption for $m$, the induction assumption holds for $m+1$ with probability at least $1-\frac{\delta}{M}$. As a consequence, using a union bound, we must have with probability at least $1-\delta$, for all $m=0,\ldots,M-1$, the induction assumption, and by extension \eqref{eq:policy_improvement_error}, is true. We now condition on this event to show the final convergence bound.

For any $m=0,\ldots, M-1$, by \eqref{eq:policy_improvement_error}, 
\begin{align*}
0\leq   V^* - V^{\hat{\zeta}^{m+1}}   &=  V^* - \mathcal{T} V^{\hat{\zeta}^{m}} + \mathcal{T} V^{\hat{\zeta}^{m}}- V^{\hat{\zeta}^{m+1}}\leq \Vert \mathcal{T} V^{\hat{\zeta}^{m}} - V^* \Vert_\infty \mathbf{1} +  \frac{c_{pi}}{(\kappa/2+1)^\mu} \mathbf{1} .
\end{align*}
Taking the infinity norm, we get
\begin{align*}
\Vert V^* - V^{\hat{\zeta}^{m+1}}\Vert_\infty&\leq \Vert \mathcal{T} V^{\hat{\zeta}^{m}} - V^* \Vert_\infty  +  \frac{c_{pi}}{(\kappa/2+1)^\mu} \leq \Vert   \mathcal{T} V^{\hat{\zeta}^{m}} - \mathcal{T} V^* \Vert_\infty + \frac{c_{pi}}{(\kappa/2+1)^\mu}\\
&\leq \gamma \Vert V^{\hat{\zeta}^{m}} - \Vtau^*\Vert_\infty + \frac{c_{pi}}{(\kappa/2+1)^\mu}\leq \gamma^{m+1} \Vert V^{\hat{\zeta}^0} - V^*\Vert_\infty + \frac{1}{1-\gamma}  \frac{c_{pi}}{(\kappa/2+1)^\mu},
\end{align*}
where the third inequality is due to \Cref{prop:exist_optimal_policy_entropy}. The desired convergence bound follows by noticing $ J(\zeta^*) - J(\hat{\zeta}^{M})\leq \Vert V^* - V^{\hat{\zeta}^M}\Vert_{\infty}$.

\section{Experiments}\label{sec:experiment}
Though the main focus of this paper is the theoretical convergence of the proposed localized algorithm $\algname$ to the global optimal policy, we also provide some preliminary  experiments to demonstrate its empirical advantages in practice. 

\subsection{Experimental Setup}
We first describe an example of networked MARL problems, where $n$ agents cooperatively work together to control the spread of a global process, e.g., information leakage or an infectious disease.

In particular, we consider $n$ agents that are connected through a line graph. For each agent $i\in[n]$, its state is a $2$-dimensional vector $s_i = (s^1_{i},s^2_{i})$, where $s^1_{i}$ and $s^2_{i}$ take binary values. Each agent has a binary action space, $a_i\in \{0, 1\}$. The value of $s^1_{i}$ in state $s_i$ obeys the following transition rule that is not affected by the action:
\begin{align*}
    \bar{s}^1_i(t+1)=
    \begin{cases}
    1 &  \text{if one of } s^1_{i-1}(t), s^1_i(t), s^1_{i+1}(t) = 1,\\
    0 & \text{otherwise}.
    \end{cases}\\
    P(s^1_i(t+1)=1|\bar{s}^1_i(t+1))=
    \begin{cases}
     1-p_1 &  \text{if } \bar{s}^1_i(t+1) = 1,\\
    0 & \text{otherwise},
    \end{cases}
\end{align*}
where $p_1\in(0,1)$. To interpret this, each agent first deterministically enters an intermediate state $\bar{s}^1_i(t+1)$ based on its own and neighbors' states in the previous time step. The true $s_i^1(t+1)$ of agent $i$ will be activated (set to $1$) with probability $1-p_1$  or $0$ according to the intermediate state.
Similarly, the value of $s_i^2$ in state $s_i$ obeys the following transition rule determined by the action:
\begin{align*}
    P\big( \bar{s}^2_i(t+1)=1|s^2_i(t),a_i(t)\big)=
    \begin{cases}
        1 &  \text{if } s^2_{i}(t)= 1,\\
        p_{\text{eff}} & \text{if } s^2_{i}(t)= 0, a_i(t) = 1,\\
        0 &  \text{otherwise}.
    \end{cases}
\end{align*}
\begin{align*}
    P\big(s^2_i(t+1)=1|\bar{s}^2_i(t+1)\big)=
    \begin{cases}
        1-p_2 &  \text{if } \bar{s}^2_i(t+1) = 1,\\
        0 & \text{otherwise}.
    \end{cases}
\end{align*}
Here $\bar s_i^2(t+1)$ is an intermediate state, which will be activated determnistically if the agent's previous state is already activated, i.e, $s_i^2(t)=1$, or activated with probability $p_{\text{eff}}$ if $s_i^2(t)=0$ but the agent takes action $1$. And the true $s_i^2(t+1)$ remains activated with probability $1-p_2$.  

This model defines a simple form of propagating/spreading dynamics on a graph, with the agents' actions providing a mechanism to control the spread.  We decompose  the reward into two terms, one depending only on the state $s_i=(s_i^1,s_i^2)$ and the other depending only on the action, 
$ r_i(s_i,a_i) = r_i^s (s_i^1,s_i^2) + r_i^a(a_i)$, where the state reward $r_i^s(\cdot)$ and action reward $r_i^a(\cdot)$ are defined as
\begin{align*}
    r_i^s(s^a_i,s^b_i) = 
    \begin{cases}
    0  & \text{if } s_i^1 =1, s_i^2 = 0, \\
    1  & \text{otherwise}.
    \end{cases}
    \qquad
    r_i^a(a_i) =  
    \begin{cases}
    1  & \text{ if } a_i=0, \\
    1 - c  & \text{ if } a_i=1 .
    \end{cases}
\end{align*}
It is easy to verify that states $(s_i^1, s_i^2) = (0,0), (1,1), (0,1)$ have a high state reward $r_i^s$, and $(s_i^1, s_i^2) = (1,0)$ has a low reward (or incurs a penalty). This corresponds to an agent incurring a penalty if a protection action is not taken before the dynamics reaches the agent. Furthermore, the protection action itself incurs a cost $c$ in the action reward $r_i^a$. 

Intuitively, if an agent does not observe any of its $\kappa$-hop neighbors to have $s^1(t)=1$, it should take action $a(t)=0$ to get the largest action reward. However, if some of its neighboring agents have $s^1(t)=1$ and the cost $c$ of taking action $1$ is relatively small,  it should take action $1$ to increase $s^2$ in order to get a reward $r^s=1$ even if its $s^1$ changes to $1$ at some time. If the probability $p_{\text{eff}}$ that action $a(t)=1$ being effective is smaller, the agent should start taking action $a=1$ earlier, when its nearest agent has $s^1(t)=1$ is at a larger distance, to make sure that it has a high probability to have $s^2=1$ when $s^1$ propagates to its position. In this sense, the optimal policy should depend on not only the local state but the states of other agents as well, and we can expect a large performance difference between localized policy with small $\kappa$  and the global optimal policy. 

\paragraph{Parameter Settings}
In our experiments, we use a variation of $\beta$-hop Localized TD(0) which has a constant step size of $0.1$. We set $\beta=\kappa$ for simplicity.  We set the environment parameters to be $p_1=0.6$, $p_2=0.7$, $c=0.3$ and $p_{\text{eff}}=0.4$. We set the training parameters to be $\gamma=0.95$ and $\eta=0.05$. We train $\algname$ for $50$ outer loops and in the inner loop we perform policy update (Line \ref{algo:policy_improvement} of \Cref{alg:SVI}) for $p_{\max}=10$ steps. %
For each experiment, we repeat it for $10$ times and plot the median in a solid line and $25/75$ percent performance values in the shaded area.  %

\subsection{Experimental Results}

\begin{figure*}[t]
    \centering
    \subfigure[$\tau=0.02$]{\includegraphics[width=0.32\textwidth]{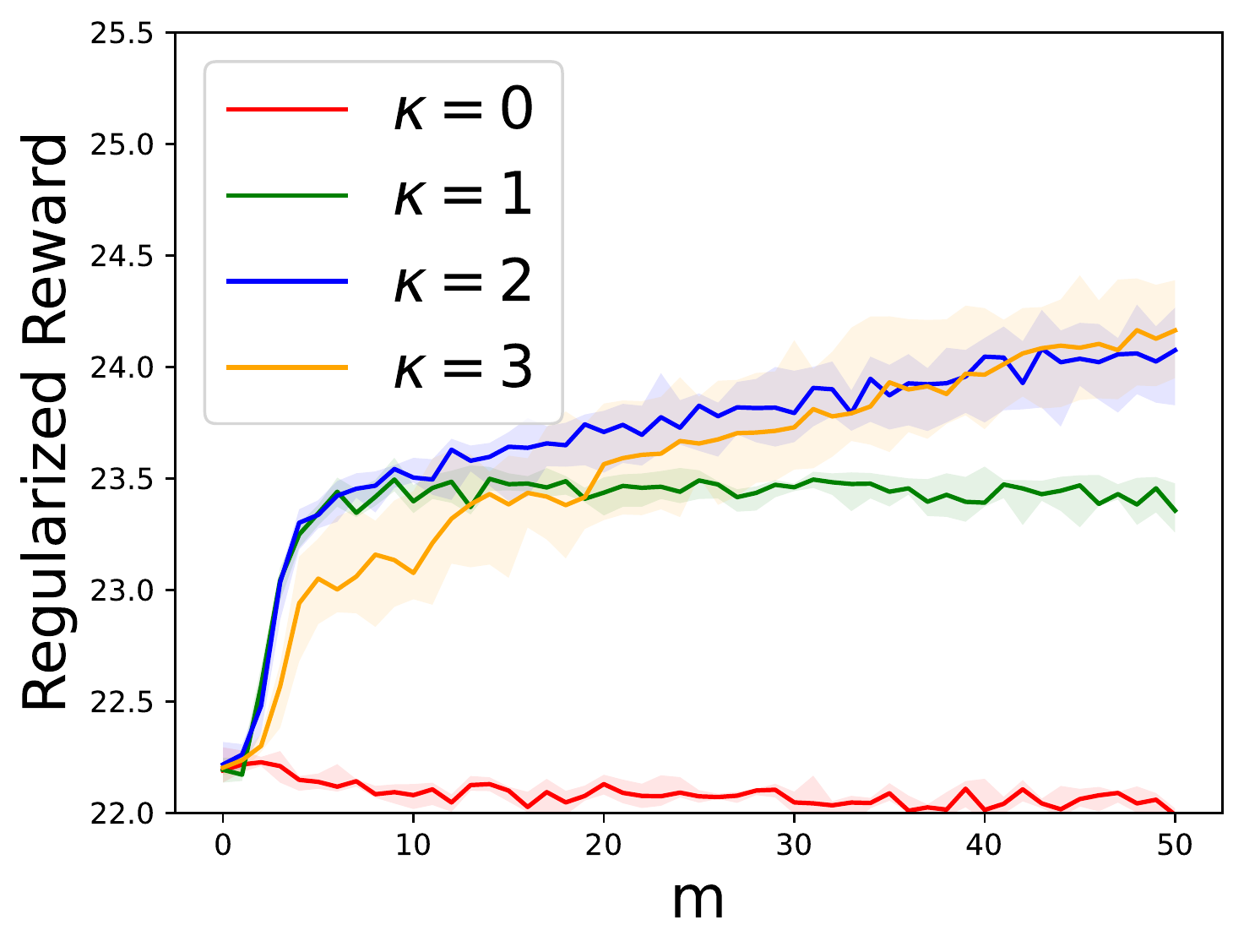}}
    \subfigure[$\tau=0.05$]{\includegraphics[width=0.32\textwidth]{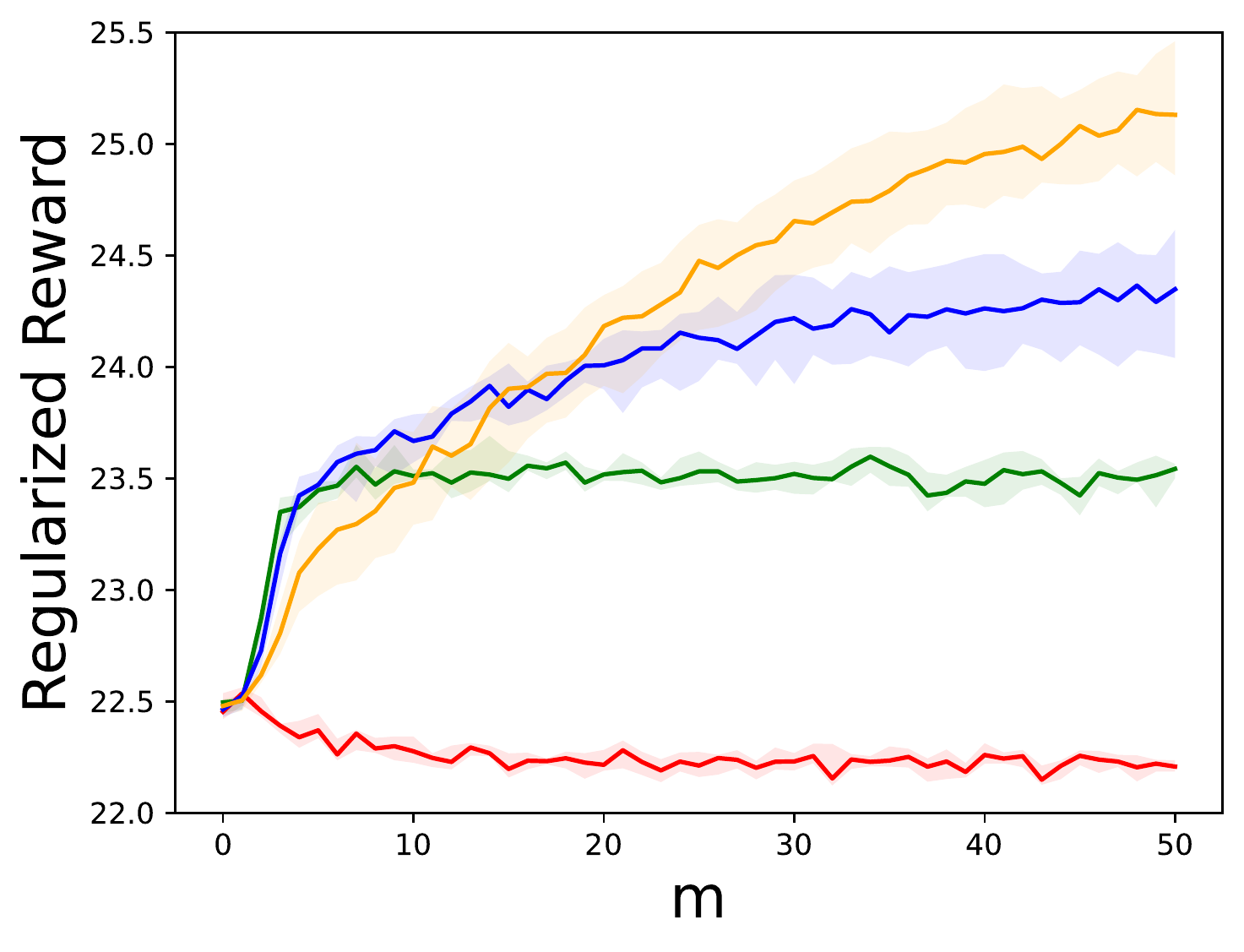}}
    \subfigure[$\tau=0.1$]{\includegraphics[width=0.32\textwidth]{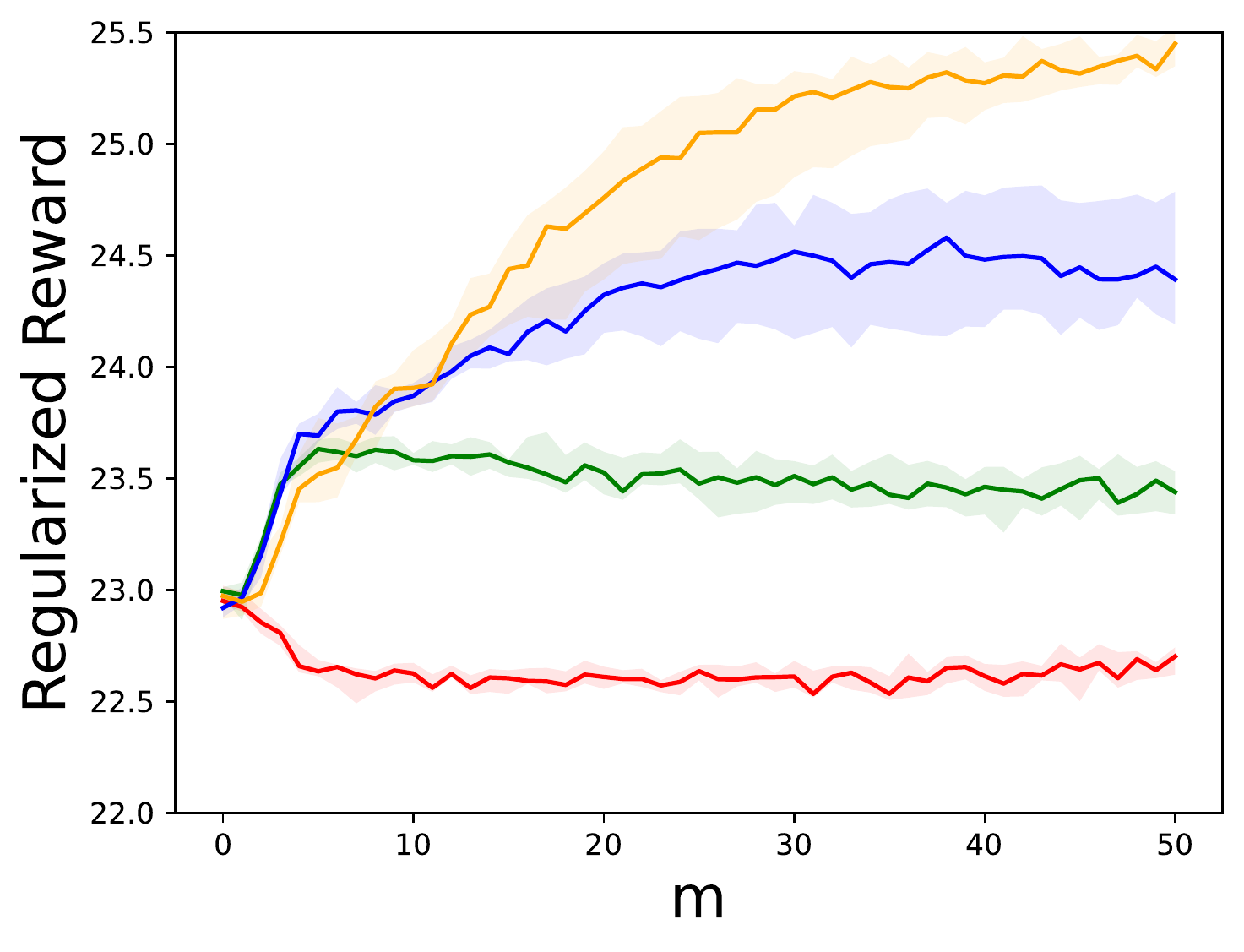}}
    \caption{Comparison of the performance of $\algname$ for different $\kappa$ under different levels of entropy regularization. %
    \label{fig:comparison_differ_tau}}
\end{figure*}
In \Cref{fig:comparison_differ_tau}, we compare the performance of $\algname$ with different choices of $\kappa$, which represents how much local information we used in the local policy learning on each agent. We also present different levels of entropy regularization in each subfigure of \Cref{fig:comparison_differ_tau} by setting $\tau=0.02,0.05$ and $0.1$. We set $n=8$ in the line graph. It can be seen that a larger $\kappa$ leads to a better total regularized reward but a lower convergence rate, which is consistent with our theoretical findings. %

\begin{figure*}[t]
    \centering
    \subfigure[$n=5$]{\includegraphics[width=0.32\textwidth]{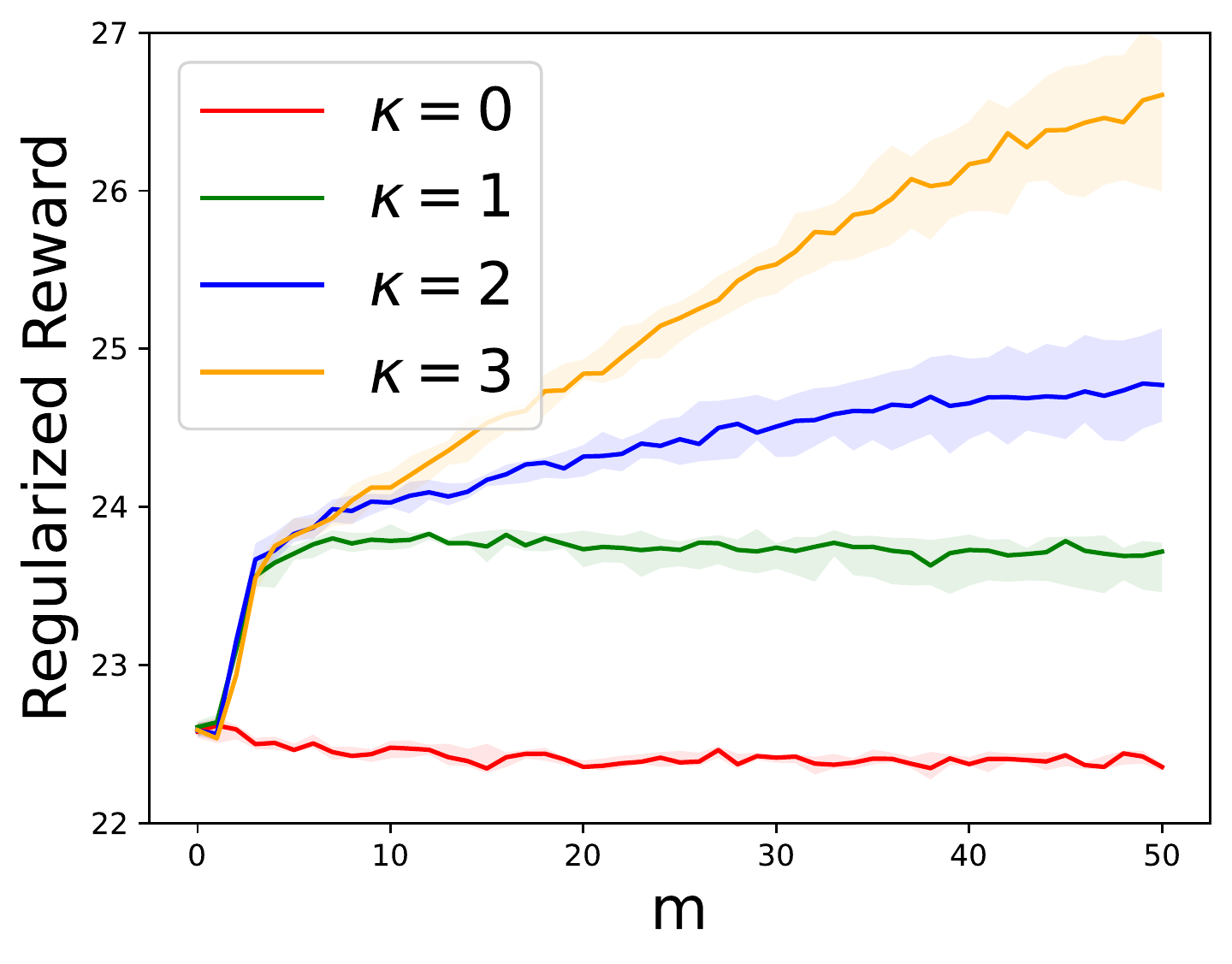}}
    \subfigure[$n=10$]{\includegraphics[width=0.32\textwidth]{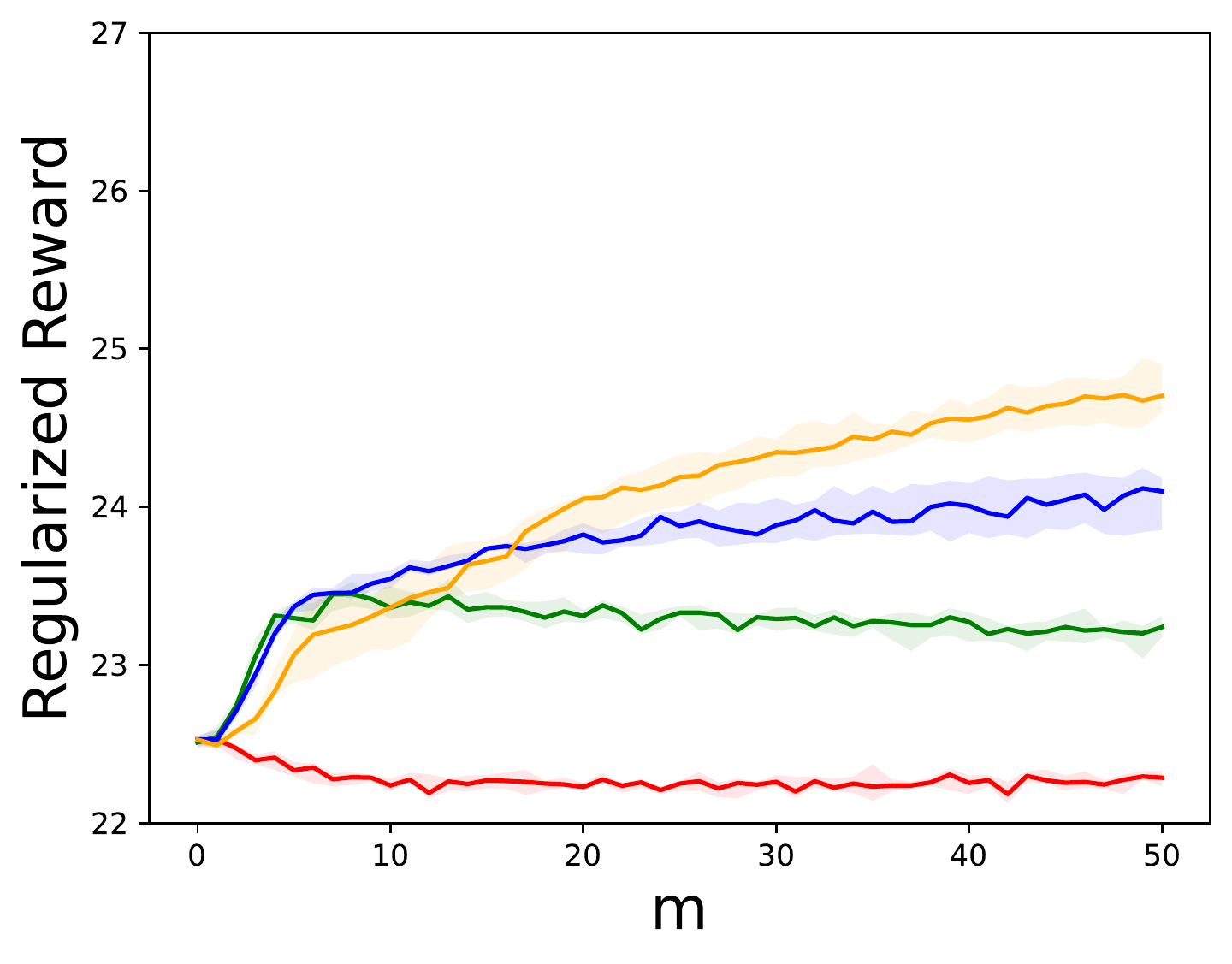}}
    \subfigure[$n=15$]{\includegraphics[width=0.32\textwidth]{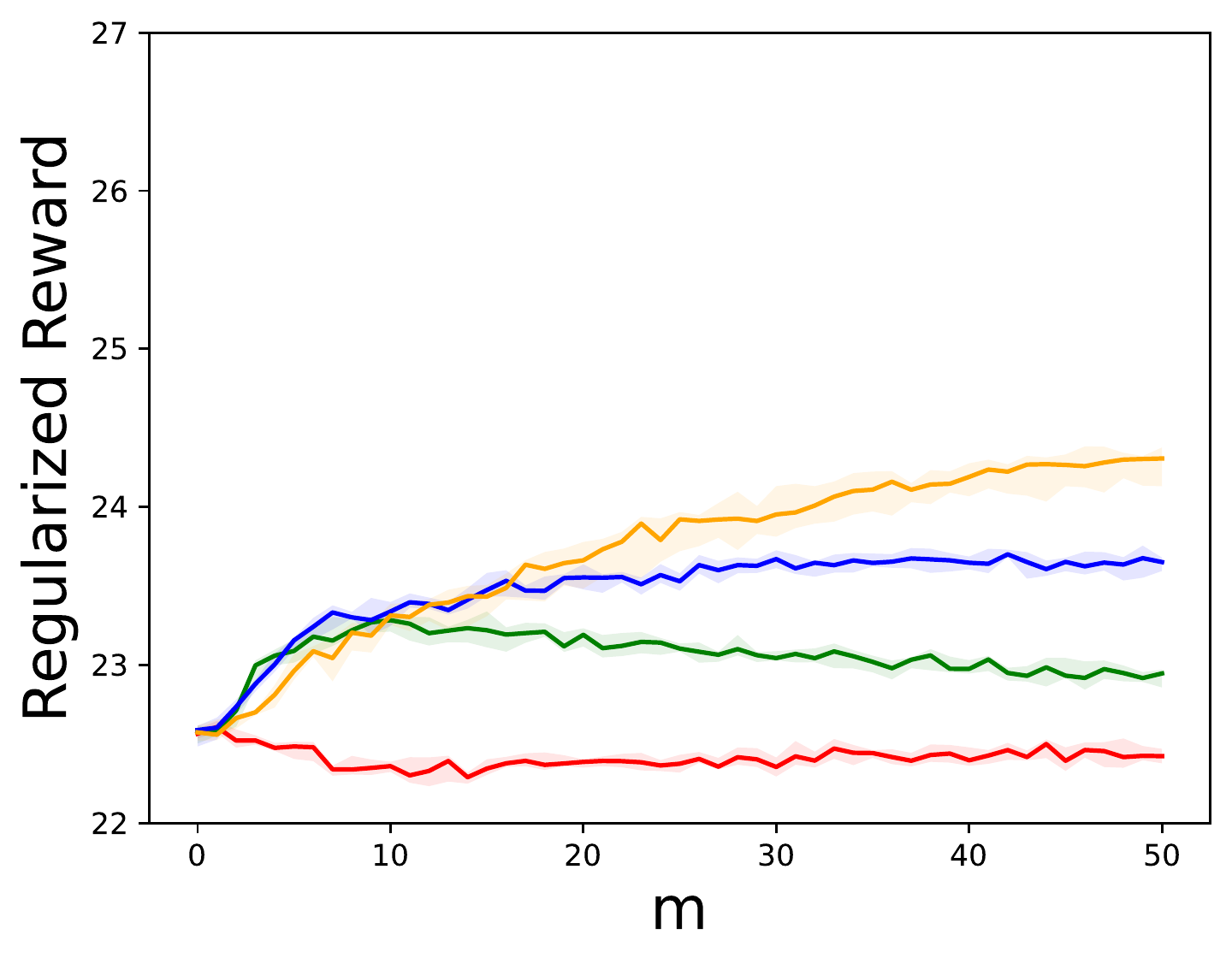}}
    \caption{Comparison of the performance of $\algname$ for different $\kappa$ in the settings with different network sizes. %
    \label{fig:comparison_differ_networksize}}
\end{figure*}

In \Cref{fig:comparison_differ_networksize}, we illustrate that our algorithm works for various network sizes by comparing the training curves among different network sizes $n=5, 10$ and $15$ while keeping $\tau$ as a constant. In each plot in \Cref{fig:comparison_differ_networksize}, larger $\kappa$ always leads to a better total regularized reward. Also, by comparing these three plots, we can see similar convergence rates for different network sizes. This shows that our algorithm has similar per-agent efficiency for networks of different scales.  

We note that some of our settings violate the bounds on $\tau$ and $\eta$ in our theorems, while $\algname$ still performs well.  Thus, in practice, one can expect our algorithm to be applicable to a larger problem class than suggested by our theoretical analysis, which could be conservative.

\section{Conclusion and Future Work}
This paper studies a cooperative MARL problem in networked systems. We provided the first theoretical guarantee on the performance gap between localized polices that make decisions only based on the information within a neighborhood of each agent and the optimal centralized policy. We achieved this by showing that a $\kappa$-hop localized policy where each agent chooses actions based on the states of its $\kappa$-hop neighbors is nearly globally optimal. Our analysis relies on a novel characterization of a class of spatial decaying policies. Further, we proposed a specific algorithm, Localized Policy Iteration ($\algname$), that learns a policy in this class. $\algname$ is communication-efficient and scales to large networked systems due to its localized implementation. We further provided a finite-sample analysis of $\algname$ showing that it converges to the globally optimal centralized policy and achieves $\varepsilon$-optimality using only local information of the $\kappa=\Theta(\text{poly}(1/\varepsilon))$-hop neighborhood  of each agent. It is worth noting that $\algname$ is the first localized algorithm that converges to the centralized globally optimal policy in the networked MARL setting. Finally, we conducted numerical experiments on a MARL problem where multiple agents on a graph cooperate together to control the spreading dynamics of a global process such as information leakage or disease transmission. %
There are many interesting questions motivated by this work. For example, it remains an open problem whether the polynomial dependence of the optimality gap on the neighborhood size $\kappa$ can be improved to an exponential decaying gap, which has been shown in the LQC setting by \citet{shin2022near}. It is also interesting to explore whether our analysis extends to the competitive MARL setting, where the objectives of agents are different from each other.

\appendix

\section{Notation}
The notations used in this paper are summarized in \Cref{tab:notation}.

\begin{table}[H]
    \centering
    \caption{Notations}
    \label{tab:notation}    
    \begin{tabular}{cp{12cm}}
    \toprule
     Symbol                                     & Definition \\
     \midrule
     $|\cS|$                                    & The cardinality of state space $\cS$. Similar notations apply to $\cA$, $\cS_i$, and $\cA_i$.\\ 
          $\actionmax$                               & The maximum cardinality for local action spaces, i.e., $\actionmax=\argmax_{i} |\cA_i|.$\\
     $[n]$                                      & Index set $\{1,2,\ldots,n\}$.\\
     $\beta,\kappa$.                            & The neighborhood sizes of the approximation for policy evaluation and policy improvement respectively in \Cref{alg:SVI}.\\
     $\cN_i$.                                   & The set including node $i$ and all its neighbors on the graph.\\
     $\cN_{i}^{\kappa}$                         & The set including node $i$ and all its $\kappa$-hop neighbors on the graph.\\
     $\cN_{-i}$                                 & The complement set of $\cN_i$, i.e., $\cN/\cN_i$. Similarly, $\cN_{-i}^{\kappa}=\cN/\nik$.\\
     $f(\kappa)$ & The size of the largest $\kappa$-hop neighborhood, $f(\kappa) = \sup_{i\in\cN}|\nik|$.  \\
     $s_{\nik}$ & The states of agents in $\nik$. Similarly for $s_{\nminusik}$ and other subscripts. \\
          $\ext{s_{\cJ}}$                            & The extension of a local state $s_\cJ$ that are supported on a set of indices $\cJ$ to a global state. See \Cref{def:ext_operator}.\\
     $\rho$ & The initial state distribution. \\
     $\tau$ & The entropy regularization parameter. \\
     $\Delta_{\nu,\mu,\sigma}$                  & The set of $(\nu,\mu)$-decay and $\sigma$-regular policy. See \Cref{def:mu_decay_policy} and \Cref{def:sigma_regular}.\\
          $\Delta_{\cA_i}$ & the probability simplex over the set $\cA_i$.\\
          $\Delta_{\cA_i|\cS}$ & The space of distributions over $\cA_i$ that can depend on elements in $\cS$. This is the space agent $i$'s policy lies in. \\
               $\Delta_{\mathrm{policy}}$& The space of all possible joint policies, cf. \Cref{subsec:mamdp}.\\ 
     $\zeta,\pi,\hat{\zeta},\hat{\pi}$ & Letters reserved for policies. Letter $\pi$ is used for policies in the inner-loop multiplicative weight updates. The hat indicates the policy is a $\kappa$-hop policy. \\
     $V^\zeta, Q^\zeta$ & Value and $Q$ functions for a given policy $\zeta$, cf. \eqref{eq:value_func} and \eqref{eq:q_func}. \\
     $V_i^\zeta, Q_i^\zeta $& Local value and $Q$ functions of agent $i$ for a given policy $\zeta$, cf. \eqref{eq:local_value_func} and \eqref{eq:local_q_func}.\\
          $\mathbf{ Q}$                              & Vector $(Q_1,\ldots,Q_n)$ that denotes the complete profile of all local $Q$ functions. \\
     $C$                                        & The matrix that characterizes the total variation of transition probabilities. See \eqref{def:C_mat_tv}.\\
     $Z^{\zeta}$                                & The interaction matrix of a policy $\zeta=(\zeta_1,\ldots,\zeta_n)$. See \Cref{def:mu_decay_policy}.\\
     $Z^{\bf Q}$                                & The interaction matrix of the local $Q$ functions ${\bf Q}=\{Q_i\}_i^n$. See \Cref{def:Z_Q}.\\
     $H^{ Q}$                                & The second-order interaction matrix of a $Q$ function. See \Cref{def:q_interaction}.\\
     $O(\cdot)$                                 & We write $f(n)=O(g(n))$ if there is a constant $c$ such that $f(n)\leq c g(n)$ for all large enough $n$.\\
     $\Omega(\cdot)$                            & We write $f(n)=\Omega(g(n))$ if there is a constant $c$ such that $f(n)\geq c g(n)$ for all large enough $n$.\\
    $\Theta(\cdot)$                            & We write $f(n)=\Theta(g(n))$ if both $f(n)=O(g(n))$ and $f(n) = \Omega(g(n))$. \\
    \bottomrule     
    \end{tabular}
\end{table}

\section{Proof of \Cref{thm:convergence_prototype}: Closure of Policy Class $\Delta_{\nu,\mu,\sigma}$ under Policy Improvement}\label{sec:proof_of_closure}
To prove the closure of policy class $\Delta_{\nu,\mu,\sigma}$ under policy improvement (i.e., \Cref{thm:convergence_prototype}), our first step is to show that for a policy $\zeta\in\Delta_{\nu,\mu,\sigma}$, its local $Q$ functions $\mathbf{Q}^\zeta = \{ Q_i^\zeta\}_{i\in\mathcal{N}}$ will also be $(\nu',\mu)$-decay for some $\nu'$ (\Cref{lem:policy-q}). Secondly, we show that if the local $Q$ functions are $(\nu',\mu)$-decay, conducting policy improvement with respect to it will lead to a $(\nu'',\mu)$ decay policy (\Cref{lem:q-policy}), with $\nu''\leq\nu$ (under the parameter settings of \Cref{thm:convergence_prototype}). Combining the two lemmas will lead to the closure property, i.e. part (a) and (b) of \Cref{thm:convergence_prototype}. The proofs of \Cref{lem:policy-q} and \Cref{lem:q-policy} can be found in \Cref{subsec:proof_lemma_closure}. 

We first present \Cref{lem:policy-q} below, showing decaying policies have decaying local $Q$ functions. 

\begin{lemma}\label{lem:policy-q}
For a policy $\zeta=(\zeta_1,\ldots,\zeta_n)\in \Delta_{\nu,\mu,\sigma}$, the interaction matrix of its corresponding local $Q$-functions $\mathbf{Q}^\zeta = \{ Q_i^\zeta\}_{i\in\mathcal{N}}$ is given by
\begin{align*}
    Z^{\mathbf{Q}^\zeta}  \leq \bar{r} I +\gamma  (\bar{r}+n\tau\sigma)\sum_{t=0}^\infty \gamma^t Z^\zeta (C+Z^\zeta)^t C,
\end{align*}
where $Z^\zeta$ is the interaction matrix for policy $\zeta$ and  $C=[C_{ij}]_{i,j=1,\ldots,n}$ is defined in \eqref{def:C_mat_tv}, and the ``$\leq$'' above is element-wise. Furthermore, if we assume $\forall i,\sum_{j\in \cN_i} 2^\mu C_{ij}\leq 1/2$ and $\nu \leq 1/2$, then $\mathbf{Q}^\zeta$ is $(\nu',\mu)$-decay with $\nu' = \bar{r} + \frac{\gamma(\bar{r} +n\tau\sigma)}{4(1-\gamma)}$. 
  
\end{lemma}

We next present \Cref{lem:q-policy}, showing that given local $Q$ functions that are $(\nu',\mu)$-decay, the policy improvement step will result in a $(\nu'',\mu)$-decay policy for some $\nu''$ given below.

\begin{lemma}\label{lem:q-policy}
Consider $\Qtau$-functions that are $(\nu',\mu)$-decay. For $\forall s$, let $\zeta(\cdot|s)=(\zeta_1(\cdot|s),\ldots,\zeta_n(\cdot|s))$ be the solution to the following policy improvement step: 
\begin{align}\label{eq:entropy_regu_Q}
 \max_{\zeta_1(\cdot|s)\in \Delta_{\mathcal{A}_1}, \ldots, \zeta_n(\cdot|s)\in \Delta_{\mathcal{A}_n}} \mathbb{E}_{a_1\sim \zeta_1(\cdot|s),\ldots, a_n\sim \zeta_n(\cdot|s)}\bigg[  \Qtau(s,a) -  \sum_{i=1}^n \tau \log \zeta_i(a_i|s)  \bigg].
\end{align}
Then, when $\tau\geq 3\times 2^{\mu+1}\nu' \actionmax^2 e$, $\zeta$ is a $\sigma'=\frac{\nu'}{n\tau}$ regular and $(\nu'',\mu)$-decay policy, where $\nu''=\frac{2^{\mu+1}\nu' \actionmax^2e^{\frac{\nu'}{n\tau}}}{\tau-2^{\mu+1}\nu' \actionmax^2 e^{\frac{\nu'}{n\tau}}}$. 
\end{lemma}

We now prove parts (a) (b) of \Cref{thm:convergence_prototype} by induction. Consider $\zeta^m\in\Delta_{\nu,\mu,\sigma}$. Under the parameter settings in \Cref{thm:convergence_prototype}, we have $2^\mu \sup_i \sum_{j\in\cN_i} C_{ij}\leq 1/2$, and $\nu=1/2$. As a result, we can apply \Cref{lem:policy-q} to $\zeta^m$ and get $\mathbf{Q}^{\zeta^m}$ is $(\nu',\mu)$ decay, where we have utilized the easy to check fact that $\bar{r} + \frac{\gamma(\bar{r} +n\tau\sigma)}{4(1-\gamma)} = \nu'$ (using $\sigma = \frac{\nu'}{n\tau}$). Then, notice the parameter settings in \Cref{thm:convergence_prototype} leads to $\tau\geq 2^\mu 6 \frac{4-3\gamma}{4-5\gamma}\bar{r}\actionmax^2 e \geq 3 \times 2^{\mu+1}\nu' \actionmax^2 e^{\frac{\nu'}{n\tau}}  $. Therefore, we can now apply \Cref{lem:q-policy} to $Q^{\zeta^m}$ to show that $\zeta^{m+1}$ is $\sigma$ regular and $(\nu'',\mu)$ decay with $\nu''\leq 1/2 = \nu$. As a result, $\zeta^{m+1}\in\Delta_{\nu,\mu,\sigma}$ and the induction is finished. 

This concludes the proof for part of (a) and (b) of \Cref{thm:convergence_prototype}. For part (c), note that
\begin{align*}
0\leq   V^* - V^{{\zeta}^{m+1}}   &=  V^* - \mathcal{T} V^{\zeta^{m}} + \mathcal{T} V^{\zeta^{m}}- V^{\zeta^{m+1}}\leq \Vert \mathcal{T} V^{\zeta^{m}} - V^* \Vert_\infty \mathbf{1}.
\end{align*}
Taking the infinity norm, we get
\begin{align*}
\Vert V^* - V^{\zeta^{m+1}}\Vert_\infty&\leq \Vert \mathcal{T} V^{\zeta^{m}} - V^* \Vert_\infty  \\
&\leq \Vert   \mathcal{T} V^{\zeta^{m}} - \mathcal{T} V^* \Vert_\infty \\
&\leq \gamma \Vert V^{\zeta^{m}} - \Vtau^*\Vert_\infty \\
&\leq \gamma^{m+1} \Vert V^{\zeta^0} - V^*\Vert_\infty ,
\end{align*}
where the third inequality is due to \Cref{prop:exist_optimal_policy_entropy}.

\section{Proof of Lemmas in Appendix \ref{sec:proof_of_closure} }\label{subsec:proof_lemma_closure}
\subsection{Proof of \Cref{lem:policy-q}}
Before we present the proof of \Cref{lem:policy-q}, we lay out some useful lemmas that will be frequently used in our analysis.

Firstly, we use the following lemma regarding the summation and product of decaying matrices, the proof of which can be found in \Cref{subsec:decay_matrix_proof}. 
\begin{lemma}\label{lem:decay_matrix}
Let $A\in\mathbb{R}^{n\times n}$ be a $(\nu,\mu)$-decay matrix and $A'\in\mathbb{R}^{n\times n}$ be a $(\nu',\mu)$-decay matrix. Then we have the following results.
\begin{itemize}
    \item [(a)] $ c A+c' A'$ is a $(c\nu + c'\nu',\mu)$-decay matrix, where $c,c'\geq 0$ are arbitrary constants.
    \item [(b)] $AA'$ is a $(\nu\nu',\mu)$-decay matrix. 
\end{itemize}
\end{lemma}

Next, we will use the Lipschitz property of the entropy function, the proof of which is deferred to \Cref{subsec:entropy_lipschitz_proof}. 
\begin{lemma}\label{lem:entropy_lipschitz}
Let $d,d'$ be two $\sigma$-regular distributions over the same set $\{1,2,\cdots,M\}$. Then we have $| H(d) - H(d')| \leq \sigma \TV(d,d')$, where $H(\cdot)$ is the entropy function. 
\end{lemma}

Our proof also uses the following two lemmas, whose proofs can be found in \cite{qu2020scalable}. 
\begin{lemma}[Lemma 5 in \citet{qu2020scalable}]\label{lem:diff_dist_f}
Let $f(\cdot)$ be a function that maps $\prod_{i\in V}\mathcal{Z}_i$ to $\mathbb{R}$, where $\mathcal{Z}_i$ is a finite set for all $i$.
Let $P_i$ and $\tilde{P_i}$  be two distributions on $\mathcal{Z}_i$, and further let $P$ (and $\tilde{P}$) be the product distribution of $P_i$ (and $\tilde{P}_i$). Then we have
\[\big|\E_{z\sim P} f(z) - \E_{z\sim \tilde{P}} f(z) \big| \leq \sum_{i\in V} \TV(P_i, \tilde{P}_i) \delta_i (f) ,\]
 	where $\delta_i(f)= \sup_{z_{V/i}} \sup_{z_i, z_i'} |f(z_i,z_{V/i}  ) -f(z_i',z_{V/i}  )  | $.
\end{lemma}

\begin{lemma}[Lemma 3,4 in \citet{qu2020scalable}]\label{lem:markov_chain_exp_decay}
Consider a Markov Chain with state $z=(z_1,\ldots,z_n)\in \mathcal{Z}=\mathcal{Z}_1\times\cdots\times \mathcal{Z}_n$, where each $\mathcal{Z}_i$ is some finite set. Suppose its transition probability factorizes as
\[ P(z(t+1)|z(t)) = \prod_{i=1}^n P_i(z_i(t+1)|z(t))\]
and further, define $C^z$ to be the following matrix,
\[C_{ij}^z =   
\sup_{z_{-j}}\sup_{z_
j,z_j'} \TV(  P_i(\cdot|z_j, z_{-j}) , P_i(\cdot| z_j', z_{-j})   )  \]
Then, fixing a pair of $(i,j)$ for any $z = (z_j,z_{-j})$, $z'=(z_j',z_{-j})$, the following holds. 
\begin{itemize}
    \item[(a)] We have, \[ \TV(\pi_{t,i},\pi'_{t,i}) \leq  [(C^z)^t  e_{j}]_i,\]
where $\pi_{t,i}$ is the distribution of $z_i(t)$ given $z(0)=z$, and $\pi'_{t,i}$ is the distribution of $z_i(t)$ given $z(0) = z'$, and $e_{j}$ is the indicator vector in which the $j$'th entry is $1$ and all other entries are $0$. 
    \item[(b)] For any function $f:\mathbb{R}^{\mathcal{Z}}\rightarrow\mathbb{R}$, we have
    \[ \left|\mathbb{E}_{z\sim \pi_t} f(z) - \mathbb{E}_{z\sim \pi_t'} f(z)\right| \leq \sum_{\ell=1}^n [(C^z)^t e_j]_\ell \delta_\ell(f) \]
    where $\delta_\ell(f) =\sup_{z_{-\ell}} \sup_{z_\ell,z_{\ell}'} |f(z_\ell,z_{-\ell}) - f(z_{\ell}',z_{-\ell})|$ and $\pi_{t}$ is the distribution of $z(t)$ given $z(0)=z$, and $\pi'_{t}$ is the distribution of $z(t)$ given $z(0) = z'$
\end{itemize}
\end{lemma}
We start our proof by showing the following \Cref{lem:cij_mupolicy}, where we treat the Markov Chain in \Cref{lem:markov_chain_exp_decay} as the Markov Chain induced by policy $\zeta$ and bound the $C^z$ matrix. The proof of \Cref{lem:cij_mupolicy} is deferred to \Cref{subsec:cij_mupolicy_proof}. 

\begin{lemma}\label{lem:cij_mupolicy}
In the settings of \Cref{lem:markov_chain_exp_decay}, we set the Markov Chain as the induced Markov Chain of our MDP when using policy $\zeta$, and we treat $z_i = s_i$ and $z=s$. The transition probabilities of this induced Markov Chain is given by,
\[ P(z(t+1)|z(t)) = \prod_{i=1}^n \underbrace{\sum_{a_i(t)\in \mathcal{A}_i}\big[ P_i(s_i(t+1)|s_{\cN_i}(t),a_i(t)) \zeta_i(a_i(t)| s(t))\big]}_{:= P_i^\zeta(s_i(t+1)|s(t))}  \]
the resulting Markov chain's interaction strength parameter $C^z = [C_{ij}^z]_{i,j=1,\ldots,n}$ (as defined in \Cref{lem:markov_chain_exp_decay}) satisfies
\begin{align*}
    C_{ij}^z =  \sup_{s_{-j}}\sup_{s_j,s_j'} \TV(  P_i^\zeta(\cdot|s_j, s_{-j}) , P_i^\zeta(\cdot| s_j', s_{-j})   ) \leq Z^\zeta_{ij} + C_{ij},
\end{align*}
where $[Z^\zeta_{ij}]_{i,j=1,\ldots,n}$ is the interaction matrix of policy $\zeta$ (cf. \Cref{def:mu_decay_policy}), and $C=[C_{ij}]_{i,j=1,\ldots,n}$ is defined in \eqref{def:C_mat_tv}. 
\end{lemma}

Now we start proving \Cref{lem:policy-q}.
\begin{proof}[Proof of \Cref{lem:policy-q}]
Given the relationship between $\Qtau_i^\zeta$ and $\Vtau_i^\zeta$ in \eqref{eq:local_q_func}, it is easy to compute the interaction matrix for the local value functions $\mathbf{\Vtau}^\zeta = \{V_i^\zeta\}_{i\in\mathcal{N}}$, which we define as below:
\begin{align}\label{eq:interact_V_proof_pi_Q}
   Z_{ij}^{\mathbf{\Vtau}^\zeta} = \sup_{s_{-j}}\sup_{s_j,s_j'}|\Vtau_i^\zeta(s_j, s_{-j}) - \Vtau_i^\zeta(s_j', s_{-j})|.
\end{align}
Fixing $i,j,s_j,s_j',s_{-j}$, define $\pi_t$ (or $\pi_t'$) to be the distribution of $s(t)$ given $s(0) = (s_j,s_{-j})$ (or $s(0) = (s_j',s_{-j})$). Let $r_i^\zeta(s) = \mathbb{E}_{a_i\sim \zeta_i(\cdot|s)} [r_i (s_i,a_i)]$. Also, let $H_i^\zeta(s) = - \sum_{a_i\in\mathcal{A}_i} \zeta_i(a_i|s) \log \zeta_i(a_i|s) $ be the entropy of local policy $\zeta_i(\cdot|s)$ as a function of $s$. Then, by \eqref{eq:local_value_func}, we have
\begin{align*}
    &|\Vtau_i^\zeta(s_j, s_{-j}) - \Vtau_i^\zeta(s_j', s_{-j})| \\
    &= \Bigg| \E_{a(t) \sim \zeta(\cdot|s(t))}\bigg[ \sum_{t=0}^\infty \gamma^t \left[ r_i(s_i(t),a_i(t) ) -n\tau \log(\zeta_i(a_i(t)|s(t)))\right]  \bigg|s(0) = (s_j,s_{-j})\bigg]   \\
    & \quad - \E_{a(t) \sim \zeta(\cdot|s(t))}\bigg[ \sum_{t=0}^\infty \gamma^t \left[ r_i(s_i(t),a_i(t) ) -n\tau \log(\zeta_i(a_i(t)|s(t)))\right]  \bigg|s(0) = (s_j',s_{-j})\bigg]   \Bigg|\\
    &\leq  \sum_{t=0}^\infty \gamma^t \left| \mathbb{E}_{s\sim \pi_t} [r_i^\zeta(s) + n\tau H_i^\zeta(s)] - \mathbb{E}_{s\sim \pi_t'} [r_i^\zeta(s) + n\tau H_i^\zeta(s)] \right|  \\
    &\leq \sum_{t=0}^\infty \gamma^t \sum_{\ell=1}^n [(C^z)^t e_j]_\ell \delta_\ell( r_i^\zeta(\cdot) + n\tau H_i^\zeta(\cdot)),
\end{align*}
where in the last inequality we have used Statement (b) of \Cref{lem:markov_chain_exp_decay}. Now let's compute $\delta_\ell( r_i^\zeta(\cdot) + n\tau H_i^\zeta(\cdot))$. For any $s_\ell,s_\ell',s_{-\ell}$, we have by definition that
\begin{align*}
   | r_i^\zeta(s_\ell,s_{-\ell}) - r_i^\zeta(s_\ell',s_{-\ell}) | &= | \mathbb{E}_{a_i\sim \zeta_i(\cdot|s_\ell,s_{-\ell})} r_i(s_i,a_i) - \mathbb{E}_{a_i\sim \zeta_i(\cdot|s_\ell',s_{-\ell})} r_i(s_i,a_i)| \\
   &\leq \TV(\zeta_i(\cdot|s_\ell,s_{-\ell}), \zeta_i(\cdot|s_\ell',s_{-\ell})) \bar{r}\\
   &\leq Z^\zeta_{i\ell} \bar{r},
\end{align*}
where $Z^{\zeta}$ is the interaction matrix of global policy $\zeta$. Similarly, we have by \Cref{lem:entropy_lipschitz},
\begin{align*}
    | H_i^\zeta(s_\ell,s_{-\ell}) - H_i^\zeta(s_\ell',s_{-\ell}) |\leq \sigma \TV(\zeta_i(\cdot|s_\ell,s_{-\ell}),\zeta_i(\cdot|s_\ell',s_{-\ell}))\leq \sigma Z^{\zeta}_{i\ell},
\end{align*}
which immediately implies that
\[\delta_\ell( r_i^\zeta(\cdot) + n\tau H_i^\zeta(\cdot)) \leq Z^\zeta_{i\ell} (\bar{r}+n\tau\sigma).  \]
As a result, 
\begin{align*}
    |\Vtau_i^\zeta(s_j, s_{-j}) - \Vtau_i^\zeta(s_j', s_{-j})|  \leq \sum_{t=0}^\infty \gamma^t \sum_{\ell=1}^n [(C^z)^t e_j]_\ell Z^\zeta_{i\ell} (\bar{r}+n\tau\sigma) = (\bar{r}+n\tau\sigma)\sum_{t=0}^\infty \gamma^t [Z^\zeta (C^z)^t]_{ij}. 
\end{align*}
By the definition in \eqref{eq:interact_V_proof_pi_Q}, we further have 
\begin{align}\label{eq:bound_ZV}
   Z^{\mathbf{\Vtau}^\zeta} \leq   (\bar{r}+n\tau\sigma)\sum_{t=0}^\infty \gamma^t Z^\zeta (C^z)^t. 
\end{align}   
It remains to convert this result to the interaction matrix of the local $\Qtau$-functions $\mathbf{\Qtau}^\zeta=\{Q_i^\zeta \}_{i\in\mathcal{N}}$. Recall that, 
 \[    \Qtau^\zeta_i(s,a) = r_i(s_i,a_i) + \gamma \mathbb{E}_{\bar{s}\sim P(\cdot|s,a)}\big[ \Vtau^\zeta_i(\bar{s})\big]. \]
Therefore, we have
\begin{align*}
   &| \Qtau^\zeta_i(s_j,a_j,s_{-j},a_{-j}) - \Qtau^\zeta_i(s_j',a_j',s_{-j},a_{-j})| \\
   &\leq  \bar{r} \mathbf{1}(j=i) +\gamma \big| \mathbb{E}_{\bar{s}\sim P(\cdot|s_j,a_j,s_{-j},a_{-j})} \Vtau^\zeta_i(\bar{s}) - \mathbb{E}_{\bar{s}\sim P(\cdot|s_j',a_j',s_{-j},a_{-j})} \Vtau^\zeta_i(\bar{s})\big|\\
   &\leq \bar{r} \mathbf{1}(j=i)+\gamma \sum_{\ell\in N_j} C_{\ell j} Z^{\mathbf{\Vtau}^\zeta}_{i\ell},
\end{align*}
where in the second inequality we have used \Cref{lem:diff_dist_f} and the definition of $C_{\ell,j}$. 
So by \Cref{def:Z_Q} and \eqref{eq:bound_ZV} we further have
\begin{align}\label{eq:bound_ZQ}
    Z^{\mathbf{\Qtau}^\zeta} \leq \bar{r} I + \gamma Z^{\mathbf{\Vtau}^\zeta} C = \bar{r} I +\gamma  (\bar{r}+n\tau\sigma)\sum_{t=0}^\infty \gamma^t Z^\zeta (C^z)^t C.
\end{align}
This gives rises to the upper bound on  $Z^{\mathbf{\Qtau}^\zeta} $ in \Cref{lem:policy-q}. 

Further, consider the scenario where we assume $Z^{\zeta}$ is $(\nu,\mu)$-decay, $\nu\leq \frac{1}{2}$ and $\forall i, \sum_{j\in\cN_i}2^{\mu}C_{ij}\leq1/2$. This means $C$ is $(1/2,\mu)$-decay since by \eqref{def:C_mat_tv}, $C$ is a sparse matrix and $C_{ij} = 0$ when $\dist(i,j)>1$. Combining these with the fact that $C_{ij}^z\leq Z_{ij}^{\zeta}+C_{ij}$ (\Cref{lem:cij_mupolicy}), we can easily prove that $C^z$ is $(\nu+1/2,\mu)$-decay by \Cref{lem:decay_matrix}. Also note that $I$ is $(1,\mu)$-decay for any $\mu$. Then applying \Cref{lem:decay_matrix} to \eqref{eq:bound_ZQ}, we can show that
$Z^{\mathbf{\Qtau}^\zeta}$ is $(\nu',\mu)$-decay, where
\begin{align}
    \nu'= \bar r + \gamma(\bar r + n\tau\sigma)\sum_{t=0}^{\infty}\gamma^t\nu(\nu+1/2)^{t} \frac{1}{2}=\bar r+ \nu(\bar r + n\tau\sigma)\frac{\frac{1}{2}\gamma }{1-\gamma(\nu+1/2)}\leq\bar r + \frac{\gamma(\bar r + n\tau\sigma)}{4(1-\gamma)},\label{eq:proof_policy_q_nuprime}
\end{align}
where the inequality uses $\nu\leq 1/2$. This completes the proof. 
\end{proof}

\subsection{Proof of \Cref{lem:q-policy}}

We first define the notion of the second order interaction matrix for the $\Qtau$ function. %
\begin{definition}\label{def:q_interaction}
Consider a $\Qtau$-function $Q(s,a)$ defined on the global state-action pair. The second order interaction matrix $H^\Qtau = [H^\Qtau_{ij}]_{i,j=1,\ldots,n}$ is defined by
\begin{align*}
    H_{ij}^\Qtau = \sup_{z_i,z_j, z_i',z_j'} \sup_{z_{-(i,j)}}\Big| [\Qtau(z_i,z_j,z_{-(i,j)} ) - \Qtau(z_i',z_j,z_{-(i,j)})] - [\Qtau(z_i,z_j',z_{-(i,j)}) - \Qtau(z_i',z_j',z_{-(i,j)})] \Big|, 
\end{align*}
where $-(i,j)$ denotes the set of nodes outside of $i$ and $j$, and here for notational simplicity, we use $z$ to represent state-action pairs, e.g. we use $z_i$ to represent $(s_i,a_i)$ and similarly $z_{-(i,j)} = (s_{-(i,j)},a_{-(i,j)})$. 
\end{definition}

Our first result is \Cref{lem:q_second_interaction_matrix} on the decay property of the second order interaction matrix of the global $Q$ function of a given policy. The proof is deferred to \Cref{subsec:q_second_interaction_matrix_proof}

\begin{lemma}\label{lem:q_second_interaction_matrix}
Suppose the local $Q$ functions $\mathbf{\Qtau} = \{Q_i\}_{i\in\mathcal{N}}$ is $(\nu',\mu)$-decay. Then, for the corresponding global $Q$ function $Q = \frac{1}{n} \sum_{i\in \mathcal{N}} Q_i$ its second order interaction matrix $H^\Qtau$ is $(2^{\mu+1}\nu',\mu)$-decay. 
\end{lemma}

We will also frequently use the following auxilliary result, the proof of which is deferred to \Cref{subsec:log_tv_proof}. 
\begin{lemma} \label{lem:log_tv}
    For two distributions $d,d'$ on $\{1,2,\ldots,m\}$, suppose for all $i,j$, we have $|\log \frac{d_i}{d_j} - \log\frac{d'_i}{d'_j}| \leq \epsilon$, and $\max(|\log \frac{d_i}{d_j} |,| \log\frac{d'_i}{d'_j}| )\leq c$.   Then, $\TV(d,d')\leq \frac{1}{2} m^2 e^c(e^\epsilon-1)$.
\end{lemma}

Now we prove \Cref{lem:q-policy}. %
For each state $s$ we consider the following entropy regularized optimization problem over $\Delta_{A_1}\times\cdots\Delta_{A_n}$, 
\begin{align} \label{eq:policy_improvement_optimization}
 (\zeta_1(\cdot|s),\ldots,\zeta_n(\cdot|s))\leftarrow   \argmax_{\pi_1(\cdot|s),\ldots,\pi_n(\cdot|s)} \mathbb{E}_{a_i\sim \pi_i(\cdot|s)} [\Qtau(s,a_1,\ldots,a_n)]+ \tau 
 \sum_{i=1}^n H(\pi_i(\cdot|s)), 
\end{align}
where $H$ is the entropy function. When the context is clear, we also use $\pi_i(s)$ to denote the distribution $\pi_i(\cdot|s)$. We also will aggregate the subscripts, e.g. $\pi(s)$ will denote the collection of $(\pi_i(s))_{i=1}^n$; $\pi_{-i}(s)$ denotes $(\pi_j(s))_{j\neq i}$. 
We consider the following multiplicative weights algorithm (where $p$ is the iteration counter):
\begin{align}
    \pi_i^{p+1} (a_i|s)\propto  \pi_i^p(a_i|s)^{1-\eta\tau}  \exp( \eta \E_{a_{-i}\sim \pi_{-i}^p(s)} \Qtau(s,a_{-i},a_i)), i=1,2,\ldots,n, \label{eq:mul_weights}
\end{align}
and we will set the initial distribution ($p=0$) $\pi_i^{0}(\cdot|s) $ to be the uniform distribution over $\mathcal{A}_i$.  The convergence of algorithm \eqref{eq:mul_weights} is shown in the following \Cref{lem:mul_weights_convergence} whose proof can be found in \Cref{subsec:proof_mul_weights}. 

\begin{lemma}\label{lem:mul_weights_convergence}
    When {$\tau>2\times 2^{\mu+1}\nu' \actionmax^2 e^{\nu'/n\tau}$ and $\eta=\frac{1}{\tau}$}, for each $s$, optimization problem \eqref{eq:policy_improvement_optimization} has a unique solution $(\zeta_1(\cdot|s),\ldots,\zeta_n(\cdot|s))$ and the algorithm \eqref{eq:mul_weights} will converge to it with a geometric convergence rate:
    \begin{align*}
         \sup_{i\in\cN} \TV(\pi_i^p(\cdot|s),\zeta_i(\cdot|s)) \leq \left(\frac{1}{\tau}2^{\mu+1}\nu' \actionmax^2 e^{\nu'/n\tau} \right)^{p}.
    \end{align*}
\end{lemma}

\begin{proof}[Proof of \Cref{lem:q-policy}]
Recall that \Cref{lem:q-policy} says that the $\zeta$ obtained by \eqref{eq:entropy_regu_Q} is $\sigma'$-regular and $(\nu'',\mu)$-decay. We show these two properties now. 

\textbf{Proof of $\sigma’$-regular}. By considering \eqref{eq:mul_weights}, we have for each $i$ and $a_i$
    \begin{align}
        \log \pi_i^{p+1}(a_i|s) = (1-\eta\tau) \log \pi_i^{p}(a_i|s)+ \eta\E_{a_{-i}\sim \pi_{-i}^p(s)} [\Qtau(s,a_{-i},a_i)] + c(s) , \label{eq:mul_weights_differential}
    \end{align}
where $c(s)$ is a normalization constant that only depends on $s$ and $\eta = \frac{\tau}{n}$. Therefore, for any 2 action pairs $a_i, \tilde{a}_i$, we have
    \begin{align*}
      & \overbrace{\log \pi_i^{p+1}(a_i|s)-\log \pi_i^{p+1}(\tilde{a}_i|s)}^{\xi_i^{p+1}(s)} \\
        &=  \underbrace{(1-\eta\tau ) (\log \pi_i^{p}(a_i|s) - \log \pi_i^{p}(\tilde{a}_i|s))}_{:=\xi_i^p(s)} +\eta\E_{a_{-i}\sim \pi_{-i}^p(s)} [\Qtau(s,a_{-i},a_i) - \Qtau(s,a_{-i},\tilde{a}_i)]  ,
    \end{align*}
where for notational simplicity, we denote $\log \pi_i^{p+1}(a_i|s)-\log \pi_i^{p+1}(\tilde{a}_i|s) = \xi_i^{p+1}(s)$ for now (where we have fixed an action pair $(a_i,\tilde{a}_i$). Note that, 
\begin{align*}
    |\Qtau(s,a_{-i},a_i) - \Qtau(s,a_{-i},\tilde{a}_i)| & \leq \frac{1}{n}\sum_{\ell=1}^n |\Qtau_\ell(s,a_{-i},a_i) - \Qtau_\ell(s,a_{-i},\tilde{a}_i) | \leq \frac{1}{n}\sum_{\ell=1}^n Z^\Qtau_{\ell i} \leq \frac{\nu'}{n}. 
\end{align*}
Next, we show that $|\xi_i^p(s)|\leq \sigma'$ by induction. The statement is clearly true for $p=0$ as $\pi_i^0(\cdot|s)$ is an uniform distribution. Suppose the statement is true for $p$, we have
\begin{align}\label{eq:sigma_regu_xi}
    |\xi_i^{p+1}(s)| \leq (1-\eta \tau) \sigma' + \eta \frac{\nu'}{n} \leq \sigma',
\end{align}
where we have used the fact that  $\sigma' = \frac{\nu'}{n\tau}$.

\textbf{Proof of $(\nu'',\mu)$-decay. } 
The bulk of the proof is to show that for any $p$, $\pi^p = (\pi_1^p,\ldots,\pi_n^p)$ is a $(\nu'',\mu)$-decay policy under the assumptions of \Cref{lem:q-policy}, in other words, the interaction matrix $Z^{\pi^k}$ is a $(\nu'',\mu)$-decay matrix. 
Now we fix a $j$ and consider \eqref{eq:mul_weights} for two values of $s$: $s = (s_j,s_{-j})$ and $s' = (s_j',s_{-j})$. The respective trajectories are $(\pi_i^p(\cdot|s))_{i=0,\ldots,n}$ and $(\pi_i^p(\cdot|s'))_{i=0,\ldots,n}$ where $p$ is the iteration counter, and the two have the same initializer, that is $\pi_i^0(\cdot|s) = \pi_i^0(\cdot|s')$ as we initialize both with the uniform distribution. Then, consider \eqref{eq:mul_weights_differential} for $s$ and $s'$, we get
\begin{align*}
  &  \xi_i^{p+1}(s) - \xi_i^{p+1}(s')\\
  & = (1-\eta\tau)(\xi_i^{p}(s) - \xi_i^{p}(s')) \\
    &\quad + \eta\E_{a_{-i}\sim \pi_{-i}^p(s)} [\Qtau(s,a_{-i},a_i) - \Qtau(s,a_{-i},\tilde{a}_i)] - \eta\E_{a_{-i}\sim \pi_{-i}^p(s')} [\Qtau(s',a_{-i},a_i) - \Qtau(s',a_{-i},\tilde{a}_i)]\\
    &= (1-\eta \tau )(\xi_i^{p}(s) - \xi_i^{p}(s')) \\
    &\quad + \underbrace{\eta\E_{a_{-i}\sim \pi_{-i}^p(s)} [\Qtau(s,a_{-i},a_i) - \Qtau(s,a_{-i},\tilde{a}_i)] - \eta\E_{a_{-i}\sim \pi_{-i}^p(s')} [\Qtau(s,a_{-i},a_i) - \Qtau(s,a_{-i},\tilde{a}_i)]}_{:=I_1}\\
    &\quad + \underbrace{ \eta\E_{a_{-i}\sim \pi_{-i}^p(s')} \Big([\Qtau(s,a_{-i},a_i) - \Qtau(s,a_{-i},\tilde{a}_i)] - [\Qtau(s',a_{-i},a_i) - \Qtau(s',a_{-i},\tilde{a}_i)] \Big)}_{:=I_2}.
\end{align*}
The absolute values of $I_1,I_2$ can be bounded as follows:
\begin{align*}
    |I_1|&\leq \eta \sum_{\ell\neq i} \TV(\pi_\ell^p(s),\pi_\ell^p(s'))H_{i \ell}^\Qtau,
\end{align*}
where we have used \Cref{lem:diff_dist_f} and the definition of second order interaction matrix $H^\Qtau$.  Similarly, for term $I_2$, we have
\begin{align*}
    |I_2|\leq \eta H_{ij}^\Qtau.
\end{align*}
This shows that,
 \begin{align*}
    | \xi_i^{p+1}(s) - \xi_i^{p+1}(s')|
 & \leq (1-\eta \tau)|\xi_i^{p}(s) - \xi_i^{p}(s')| + \eta \sum_{\ell\neq i} \TV(\pi_\ell^p(s),\pi_\ell^p(s'))H_{i \ell}^\Qtau+\eta H_{i j}^\Qtau.
\end{align*}
Denote vector $H_{:j}^\Qtau$ as the $j$'th column of $H^\Qtau$, $H_{i:}^\Qtau$ as the $i$'th row, and $H^\Qtau_{jj}=0$ for any $i,j\in[n]$. Let vector $v^p$ denote $v^p_i = \TV(\pi_i^p(s),\pi_i^p(s'))$.  Then it holds that
\begin{align*}
    | \xi_i^{p+1}(s) - \xi_i^{p+1}(s')| &\leq(1-\eta \tau)|\xi_i^{p}(s) - \xi_i^{p}(s')|+\eta H_{i:}^\Qtau v^p+\eta H_{i j}^\Qtau\\
    & \leq (1-\eta\tau)^{p+1}| \xi_i^{0}(s) - \xi_i^{0}(s')|+\sum_{k=0}^p \eta(1-\eta\tau)^{p-k}(H_{i:}^\Qtau v^k+H_{ij}^\Qtau).
\end{align*}
Recall the definition  $\xi_i^{p+1}(s)=\log\pi_{i}^{p+1}(a_i|s)-\log\pi_i^{p+1}(\tilde a_i|s)$ and that $|\xi_i^{p+1}(s)|\leq \nu'/\tau$ by \eqref{eq:sigma_regu_xi}. By  \Cref{lem:log_tv}, we obtain the following bound on $v_i^{p+1}=\TV(\pi_i^{p+1}(s),\pi_i^{p+1}(s'))$:
\begin{align*}
    v_i^{p+1}&\leq \frac{|\cA_i|^2e^{\sigma'}}{2} \bigg(\exp\bigg((1-\eta\tau)^{p+1}|\xi^0(s) - \xi^{0}(s')|+\sum_{k=0}^p \eta (1-\eta\tau)^{p-k} (H_{i:}^\Qtau v^k + H^\Qtau_{ij})\bigg)-1\bigg).
\end{align*}
where $|\cA_i|$ is the cardinality of $\cA_i$. Now we choose $\eta=1/\tau$ to simplify the presentation. Note that in general $\eta\leq 1/\tau$ works and the proof will be more involved. Then we obtain
\begin{align*}
    v_i^{p+1}&\leq |\cA_i|^2e^{\sigma'}/2 \big(\exp\big(\eta  (H_{i:}^\Qtau v^p + H^\Qtau_{ij})\big)-1\big).
\end{align*}
Denote $c_{\TV}=\actionmax^2e^{\sigma'}/2$. For all $i,j\in[n]$ and $p=0,1,\ldots$, we have that $v_i^p\leq 1$ and that  $H^\Qtau_{ij}\leq\sum_{k}H^\Qtau_{ik}\leq\sum_{k}H^\Qtau_{ik}(\dist(i,k)+1)^{\mu}\leq 2^{\mu+1}\nu'$ by \Cref{lem:q_second_interaction_matrix}.

Then it holds that $\eta(\sum_{k}H_{ik}^\Qtau v_k^p + H^\Qtau_{ij})\leq 1/2$ which is due to our choice of $\tau$ parameter satisfies $\tau=\frac{1}{\eta}\geq \nu'2^{\mu+3}$, 
which further implies
\begin{align*}
    v_i^{p+1}&\leq c_{\TV}\bigg(\exp\bigg(\eta  \bigg(\sum_{k}H_{ik}^\Qtau v_k^p + H^\Qtau_{ij}\bigg)\bigg)-1\bigg)\leq 2c_{\TV}\eta  \bigg(\sum_{k}H_{ik}^\Qtau v_k^p + H^\Qtau_{ij}\bigg).
\end{align*}
where the second inequality is due to the fact that for $0<x<c<1$, we have $e^x\leq 1+x/(1-c)$. Stacking all $v_i^{p+1}$ together yields 
\begin{align*}
     v^{p+1} \leq 2c_{\TV} \eta (H^\Qtau v^p + H^\Qtau_{:j}),
\end{align*}
where the absolute values and $\leq$ are interpreted element-wise. 
By \Cref{def:mu_decay_policy}, we have $Z_{ij}^{\pi^{p}}=\sup_{s_{-j}}\sup_{s_j,s_j'}v_i^p$. Note that for a given $\Qtau$ function, $H^\Qtau$ is fixed and independent of states or actions. 

Since the above inequality holds for any $s$ and $s'$ that only differ in the $j$-th entry, $\forall j\in[n]$, it immediately implies
\begin{align*}
    Z^{\pi^{p+1}} \leq 2 c_{\TV}\eta(H^\Qtau Z^{\pi^p}+H^\Qtau).
\end{align*}
By \Cref{lem:q_second_interaction_matrix} $H^\Qtau$ is $(2^{\mu+1}\nu',\mu)$-decay. Denote $\nu_H=2^{\mu+1}\nu'$. We will prove that $Z^{\pi_p}$ is $(\nu_p,\mu)$-decay by induction, where $\nu_p$ is to be determined later. First, $\pi_0$ is a uniform policy and thus $\nu_0=0$ for $Z^{\pi_0}$. By  \Cref{lem:decay_matrix}, it holds that
\begin{align*}
    \nu_{p+1}&\leq 2c_{\TV}\eta(\nu_H\nu_p+\nu_H)\leq (2c_{\TV}\eta\nu_H)^{p+1}\nu_0+\sum_{k=1}^{p+1}(2c_{\TV}\eta\nu_H)^k\leq \frac{2c_{\TV}\eta\nu_H}{1-2c_{\TV}\eta\nu_H},
\end{align*}
which implies that $Z^{\pi^{p}}$ is $(\frac{2^{\mu+2}\nu' c_{\TV}}{\tau-2^{\mu+2}\nu' c_{\TV}},\mu)$-decay for all $p$. Further by  \Cref{lem:mul_weights_convergence}, let $p\rightarrow\infty$, we have that $Z^{\zeta^*}$ is $(\nu'',\mu)$-decay and thus $\zeta^*\in\Delta_{\nu'',\mu,\sigma'}$, where $\nu''=\frac{2^{\mu+1}\nu'\actionmax^2e^{\frac{\nu'}{n\tau}}}{\tau-2^{\mu+1}\nu' \actionmax^2e^{\frac{\nu'}{n\tau}}}$. This completes the proof of \Cref{lem:q-policy}.

\end{proof}

Based on the proof of \Cref{lem:q-policy}, it is easy to see that the intermediate iterates in the iterative algorithm in \eqref{eq:mul_weights} are also $(\nu'',\mu)$-decay policies. We wrote this intermediate result below as a Corollary and we will use it in other parts of the proof. 
\begin{corollary}\label{cor:q_policy_pi_p_decay}
We have that under the same setting of \Cref{lem:q-policy} and under the multiplicative algorithm \eqref{eq:mul_weights}, $\pi^p$ is $\bigg(\frac{2^{\mu+1}\nu'\actionmax^2e^{\frac{\nu'}{n\tau}}}{\tau-2^{\mu+1}\nu' \actionmax^2e^{\frac{\nu'}{n\tau}}},\mu\bigg)$-decay. 
\end{corollary}

\section{Proofs of Supporting Lemmas}
In this section, we provide the proofs of the lemmas we used in proving the main results.

\subsection{Proof of  \Cref{prop:exist_optimal_policy_entropy}}\label{subsec:proof_exist_optimal_policy_entropy}

Since $|\sup_{x}f(x)-\sup_{x'}g(x')|\leq \sup_{x}|f(x)-g(x)|$, we have for any $V_1,V_2\in\mathbb{R}^{|\mathcal{S}|}$ and $s\in\mathcal{S}$ that
\begin{align*}
    |[\mathcal{T}(V_1)](s)-[\mathcal{T}(V_2)](s)|
    \leq \;&\gamma \sup_{\zeta\in \Delta_{\mathrm{policy}}}|  \mathbb{E}_{\zeta} \left[ V_1(s(1))-V_2(s(1))\mid s(0)=s\right] |\\
    \leq \;&\gamma \sup_{\zeta\in \Delta_{\mathrm{policy}}} \mathbb{E}_{\zeta} \left[ |V_1(s(1))-V_2(s(1))|\mid s(0)=s\right] \\
    \leq \;&\gamma \|V_1-V_2\|_\infty.
\end{align*}
Since the previous inequality holds for all $s\in\mathcal{S}$, we have the desired contraction property:
\begin{align*}
    \|\mathcal{T}(V_1)-\mathcal{T}(V_2)\|_\infty\leq \gamma \|V_1-V_2\|_\infty.
\end{align*}

We next show that if a policy $\zeta^*$ is such that $V^{\zeta^*}=\mathcal{T}(V^{\zeta^*})$, then $\zeta^*$ is an optimal policy. For any policy $\zeta\in\Delta_{\text{policy}}$, we have for all $s\in\mathcal{S}$ that
\begin{align*}
    \Vtau^{\zeta^*} (s)=\;& (\mathcal{T}\Vtau^{\zeta^*})(s) \\
    \geq \;&\mathbb{E}_{\zeta}\left[  r(s(0),a(0))  -   \tau \log \zeta(a(0)|s(0)) +  \gamma  \Vtau^{\zeta^*}(s(1))\;\middle|\; s(0)=s\right]\\
    =\;&  \mathbb{E}_{\zeta}\left[  r(s(0),a(0))  -  \tau \log \zeta(a(0)|s(0)) \mid s(0)=s\right]+\gamma \mathbb{E}_{\zeta} \left[\Vtau^{\zeta^*}(s(1)) \;\middle|\; s(0)=s\right] \\
    \geq\;& \mathbb{E}_{\zeta}\left[  r(s(0),a(0))  -  \tau \log \zeta(a(0)|s(0)) \mid s(0)=s\right] \\
    &+ \gamma \mathbb{E}_{\zeta}\left[  r(s(1),a(1))  -   \tau \log \zeta(a(1)|s(1)) \mid s(0)=s\right] + \gamma^2 \mathbb{E}_{\zeta} \left[\Vtau^{\zeta^*}(s(2))\;\middle|\; s(0)=s\right]\\
    \geq \;&\sum_{t=0}^\infty\gamma^t\mathbb{E}_{\zeta}\left[  r(s(t),a(t))  -  \tau \log \zeta(a(t)|s(t)) \mid s(0)=s\right]\\
    = \;&\Vtau^{\zeta}(s). 
\end{align*}
As a result, we have for all $\zeta\in\Delta_{\text{policy}}$ that
\begin{align*}
    J(\zeta^*)=\sum_{s\in\mathcal{S}}\initialState(s)V^{\zeta^*}(s)\geq \sum_{s\in\mathcal{S}}\initialState(s)V^{\zeta}(s)=J(\zeta),
\end{align*}
hence $\zeta^*$ is an optimal policy.

\subsection{Proof of \Cref{lem:performance_diff_tv}}
For each $i\in[n]$, we have
\begin{align*}
   & |V_i^\zeta(s)-V_i^{\tilde{\zeta}}(s)|\\
&= \Big| \mathbb{E}_{a\sim \zeta(\cdot|s)}[r_i(s_i,a_i)-n\tau\log\zeta_i(a_i|s)+\gamma\mathbb{E}_{s'\sim P(\cdot|s,a)}V_i^\zeta(s')  ] \\
&\qquad-\mathbb{E}_{a\sim \tilde\zeta(\cdot|s)}[r_i(s_i,a_i)- n\tau\log\tilde\zeta_i(a_i|s)+\gamma\mathbb{E}_{s'\sim P(\cdot|s,a)}V_i^{\tilde\zeta}(s') ]\Big|\\
&\leq \underbrace{\left| \mathbb{E}_{a\sim \zeta(\cdot|s)}[r_i(s_i,a_i)+\gamma\mathbb{E}_{s'\sim P(\cdot|s,a)}V_i^\zeta(s') ]-\mathbb{E}_{a\sim \tilde\zeta(\cdot|s)}[r_i(s_i,a_i) +\gamma\mathbb{E}_{s'\sim P(\cdot|s,a)}V_i^\zeta(s') ] \right|}_{I_1} \\
&\qquad+\underbrace{\left|\mathbb{E}_{a\sim \tilde\zeta(\cdot|s)}[r_i(s_i,a_i)+\gamma\mathbb{E}_{s'\sim P(\cdot|s,a)}V_i^\zeta(s') ]-\mathbb{E}_{a\sim \tilde\zeta(\cdot|s)}[r_i(s_i,a_i) +\gamma\mathbb{E}_{s'\sim P(\cdot|s,a)}V_i^{\tilde\zeta}(s') ]\right|}_{I_2}\\
 & \qquad+\underbrace{n\tau | H(\zeta_i(\cdot|s)) - H(\tilde\zeta_i(\cdot|s))|}_{I_3} \\
&\leq  \gamma\Vert V_i^\zeta-V_i^{\tilde\zeta}\Vert_\infty + n\tau\sigma \TV(\tilde\zeta_i(\cdot|s),\zeta_i(\cdot|s)) + \sum_{j=1}^n\TV(\zeta_j(\cdot|s),\tilde\zeta_j(\cdot|s)) Z^{Q^{\tilde{\zeta}}}_{ij},
\end{align*}
where in the last inequality, we have used  \Cref{lem:diff_dist_f} (for term $I_1$), the definition of $Z^{Q^{\tilde{\zeta}}}$ in  \Cref{def:Z_Q} (for term $I_2$), and \Cref{lem:entropy_lipschitz} (for term $I_3$) respectively. Therefore, we have
\begin{align*}
    \Vert V_i^\zeta - V_i^{\tilde{\zeta}}\Vert_\infty\leq  \frac{1}{1-\gamma} \left(n\tau\sigma \sup_{s\in\cS}\TV(\tilde\zeta_i(\cdot|s),\zeta_i(\cdot|s)) + \sum_{j=1}^n \sup_{s\in\cS} \TV(\zeta_j(\cdot|s),\tilde\zeta_j(\cdot|s)) Z^{Q^{\tilde{\zeta}}}_{ij}  \right).
\end{align*}
Taking the average over $i$, we have, 
\begin{align*}
  \Vert V^\zeta-V^{\tilde{\zeta}}\Vert_\infty
&\leq \frac{1}{1-\gamma} \left( \tau\sigma \sum_{i=1}^n \sup_{s\in\cS}\TV(\tilde\zeta_i(\cdot|s),\zeta_i(\cdot|s)) +\frac{1}{n} \sum_{i=1}^n\sum_{j=1}^n \sup_{s\in\cS}\TV(\zeta_j(\cdot|s),\tilde\zeta_j(\cdot|s))  Z^{Q^{\tilde{\zeta}}}_{ij} \right)\\
&\leq \frac{1}{1-\gamma} \bigg(\tau\sigma + \frac{\nu'}{n}\bigg) \sum_{i=1}^n \sup_{s\in\cS}\TV(\tilde\zeta_i(\cdot|s),\zeta_i(\cdot|s))  ,
\end{align*}
where in the last inequality, we have used $\sum_{i\in\cN} Z_{ij}^{Q^{\tilde{\zeta}}} \leq \nu'$.

\subsection{Proof of \Cref{lem:decay_matrix}} \label{subsec:decay_matrix_proof}

Due to the symmetry in  \Cref{def:matrix_decay}, we only show the $(\nu,\mu)$-decay property of a matrix for each row sums. The proof for column sums follows the same process. 
\begin{itemize}
    \item[(a)] For any $c,c'\geq 0$, it holds that
\begin{align*}
    \sum_{j=1}^n (cA_{ij}+c'A_{ij}')(\dist(i,j)+1)^{\mu}\leq c\nu+c'\nu'.
\end{align*}
\item[(b)] Since $1\leq\dist(i,j)+1\leq(\dist(i,k)+1)(\dist(k,j)+1)$ for all $i,j,k\in[n]$, we have
\begin{align*}
    \sum_{j}[AA']_{ij}(\dist(i,j)+1)^{\mu}&=\sum_{j}\sum_{k}a_{ik}a_{kj}'(\dist(i,j)+1)^{\mu}\\
    &\leq \sum_{j}\sum_{k}a_{ik}a_{kj}'(\dist(i,k)+1)^{\mu}(\dist(k,j)+1)^{\mu}\\
    &= \sum_{k}a_{ik} (\dist(i,k)+1)^{\mu} \sum_{j}a_{kj}'(\dist(k,j)+1)^{\mu}\\
    &\leq \nu\nu'.
\end{align*}
\end{itemize}

\subsection{Proof of  \Cref{lem:entropy_lipschitz}}\label{subsec:entropy_lipschitz_proof}
Define $h(t) = H(d+t(d'-d))$ for all $t\in [0,1]$. Then we have
\begin{align*}
    |H(d) - H(d')| &= |h(1) - h(0)|\\
    &= \left|\int_{0}^1 h'(t) dt\right| \\
    &=\left|\int_{0}^1 \langle \nabla H(d+t(d'-d)),  d' - d\rangle dt\right|.
\end{align*}
For simplicity of notation, denote $d'' = d + t(d'-d)$. Since the set of all $\sigma$-regular distributions over $\{1,2,\cdots,M\}$ is a polytope (hence convex),
the distribution $d''$ as a convex combination between $d$ and $d'$ is also $\sigma$-regular. Let $d''_{\max }= \max_m d''_m$ and $d''_{\min} = \min_m d''_m$. Then, we have,
\begin{align}
    \left| \langle \nabla H(d''),  d' - d\rangle \right| & =\left|\sum_{m=1}^M (1 + \log d_m'')(d_m - d_m') \right|\nonumber\\
    &=\left|\sum_{m=1}^M \bigg( -\frac{\log d''_{\max} + \log d''_{\min}}{2} + \log d_m''\bigg)(d_m - d_m') \right|\label{eqeq:1}\\
    &\leq \sum_{m=1}^M \left| -\frac{\log d''_{\max} + \log d''_{\min}}{2} + \log d_m'' \right| \left|d_m - d_m'\right|\nonumber\\
    &\leq \sum_{m=1}^M \left| \frac{\log d''_{\max} - \log d''_{\min}}{2}  \right| \left|d_m - d_m'\right|\label{eqeq:2}\\
    &\leq \sigma \TV(d,d'),\nonumber
\end{align}    
where Eq. (\ref{eqeq:1}) is due to the fact that $\sum_{m=1}^Mc(d_m-d_m')=0$ for any constant $c\in\mathbb{R}$, and Eq. (\ref{eqeq:2}) is due to the fact that 
\begin{align*}
    -(\log d''_{\max} - \log d''_{\min})/2\leq-(\log d''_{\max} + \log d''_{\min})/2+ \log d_m''\leq(\log d''_{\max} - \log d''_{\min})/2.
\end{align*}

\subsection{Proof of  \Cref{lem:cij_mupolicy}}\label{subsec:cij_mupolicy_proof}
To calculate $C_{ij}^z$, we consider the following two cases. 

\textit{Scenario 1:} $j\not\in \cN_i$. 
The probability of the next state $\bar{s}_i$ 
given the current state $(s_j,s_{-j})$ is
\[ \sum_{a_i\in \mathcal{A}_i} P_i(\bar{s}_i| s_{\cN_i},a_i)  \zeta_i(a_i | s_j,s_{-j}) = P^\zeta_i(\bar{s}_i|s_j,s_{-j}).\]
Therefore, when $s_j$ is changed to $s_j'$, the TV distance is,
\begin{align*}
&\TV( P^\zeta_i(\cdot|s_j,s_{-j}), P^\zeta_i(\cdot|s_j',s_{-j}))\\
& = \frac{1}{2} \sum_{\bar{s}_i} \Big| \sum_{a_i\in \mathcal{A}_i} P_i(\bar{s}_i| s_{\cN_i},a_i) \zeta_i(a_i | s_j,s_{-j}) -\sum_{a_i\in \mathcal{A}_i} P_i(\bar{s}_i| s_{\cN_i},a_i) \zeta_i(a_i | s_j',s_{-j})  \Big| \\
  &\leq \frac{1}{2}  \sum_{a_i} \left(\sum_{\bar{s}_i}   P_i(\bar{s}_i| s_{\cN_i},a_i) \right)\Big|\zeta_i(a_i | s_j,s_{-j}) - \zeta_i(a_i | s_j',s_{-j}) \Big| \\
  &= \TV(\zeta_i(\cdot | s_j,s_{-j}) , \zeta_i(\cdot | s_j',s_{-j}))  \\
  &\leq   Z^{\zeta}_{ij},
  \end{align*}
  where in the last inequality we have used the definition of the interaction matrix $Z_{ij}^\zeta$ of policy $\zeta$ in \Cref{def:mu_decay_policy}. 
  
\textit{Scenario 2: $j\in \cN_i$:} When $s_j$ is changed to $s_j'$, we have 
\begin{align*}
    &\TV( P^\zeta_i(\cdot|s_j,s_{-j}), P^\zeta_i(\cdot|s_j',s_{-j}))\\
& = \frac{1}{2} \sum_{\bar{s}_i} \Big| \sum_{a_i\in \mathcal{A}_i} P_i(\bar{s}_i| s_{\cN_i/j},s_j,a_i) \zeta_i(a_i | s_j,s_{-j}) -\sum_{a_i\in \mathcal{A}_i} P_i(\bar{s}_i| s_{\cN_i/j},s_j',a_i) \zeta_i(a_i | s_j',s_{-j})  \Big| \\
&\leq  \frac{1}{2} \sum_{\bar{s}_i}  \sum_{a_i\in \mathcal{A}_i}\Big| P_i(\bar{s}_i| s_{\cN_i/j},s_j,a_i) \zeta_i(a_i | s_j,s_{-j}) -P_i(\bar{s}_i| s_{\cN_i/j},s_j',a_i) \zeta_i(a_i | s_j',s_{-j})  \Big|\\
&\leq \frac{1}{2} \sum_{\bar{s}_i}  \sum_{a_i\in \mathcal{A}_i}\Big| P_i(\bar{s}_i| s_{\cN_i/j},s_j,a_i) \zeta_i(a_i | s_j,s_{-j}) -P_i(\bar{s}_i| s_{\cN_i/j},s_j,a_i) \zeta_i(a_i | s_j',s_{-j})  \Big|\\
&\qquad + \frac{1}{2} \sum_{\bar{s}_i}  \sum_{a_i\in \mathcal{A}_i}\Big| P_i(\bar{s}_i| s_{\cN_i/j},s_j,a_i) \zeta_i(a_i | s_j',s_{-j}) -P_i(\bar{s}_i| s_{\cN_i/j},s_j',a_i) \zeta_i(a_i | s_j',s_{-j})  \Big|\\
&\leq \frac{1}{2} \sum_{a_i\in \mathcal{A}_i} \sum_{\bar{s}_i}  P_i(\bar{s}_i| s_{\cN_i/j},s_j,a_i) \Big|  \zeta_i(a_i | s_j,s_{-j}) -\zeta_i(a_i | s_j',s_{-j})  \Big|\\
&\qquad + \frac{1}{2} \sum_{a_i\in \mathcal{A}_i} \zeta_i(a_i | s_j',s_{-j}) \sum_{\bar{s}_i}   \Big| P_i(\bar{s}_i| s_{\cN_i/j},s_j,a_i)  -P_i(\bar{s}_i| s_{\cN_i/j},s_j',a_i)   \Big| \\
&= \TV(\zeta_i(\cdot|s_j,s_{-j}), \zeta_i(\cdot|s_j',s_{-j})) +  \sum_{a_i\in \mathcal{A}_i} \zeta_i(a_i | s_j',s_{-j}) \TV(P_i(\cdot| s_{\cN_i/j},s_j,a_i),P_i(\cdot| s_{\cN_i/j},s_j',a_i))\\
&\leq Z^{\zeta}_{ij} + C_{ij}.
\end{align*}

\subsection{Proof of  \Cref{lem:q_second_interaction_matrix}}\label{subsec:q_second_interaction_matrix_proof}
First notice the following, 
\begin{align*}
        H_{ij}^\Qtau &\leq  \sup_{z_i,z_j, z_i',z_j'} \sup_{z_{-(i,j)}}  \frac{1}{n} \sum_{\ell=1}^n \Big| [\Qtau_\ell(z_i,z_j,z_{-(i,j)}) - \Qtau_\ell(z_i',z_j,z_{-(i,j)})] - [\Qtau_\ell(z_i,z_j',z_{-(i,j)}) - \Qtau_\ell(z_i',z_j',z_{-(i,j)})] .\Big|
\end{align*}
Clearly, by the definition of $(\nu',\mu)$-decay of $\Qtau$ function, we have
\begin{align*}
  &  \Big| [\Qtau_\ell(z_i,z_j,z_{-(i,j)}) - \Qtau_\ell(z_i',z_j,z_{-(i,j)})] - [\Qtau_\ell(z_i,z_j',z_{-(i,j)}) - \Qtau_\ell(z_i',z_j',z_{-(i,j)})] \Big|\\
    &\leq |\Qtau_\ell(z_i,z_j,z_{-(i,j)}) - \Qtau_\ell(z_i',z_j,z_{-(i,j)})| + |\Qtau_\ell(z_i,z_j',z_{-(i,j)}) - \Qtau_\ell(z_i',z_j',z_{-(i,j)})|\\
    &\leq 2 Z^\Qtau_{\ell i}. 
\end{align*}
Using a symmetric argument, we also have, 
\begin{align*}
      &  \Big| [\Qtau_\ell(z_i,z_j,z_{-(i,j)}) - \Qtau_\ell(z_i',z_j,z_{-(i,j)})] - [\Qtau_\ell(z_i,z_j',z_{-(i,j)}) - \Qtau_\ell(z_i',z_j',z_{-(i,j)})] \Big|\\
      &= \Big| [\Qtau_\ell(z_i,z_j,z_{-(i,j)}) - \Qtau_\ell(z_i,z_j',z_{-(i,j)})] - [\Qtau_\ell(z_i',z_j,z_{-(i,j)}) - \Qtau_\ell(z_i',z_j',z_{-(i,j)})] \Big|\\
    &\leq 2 Z^\Qtau_{\ell j}.
\end{align*}
As a result, 
\[  H_{ij}^\Qtau \leq \frac{1}{n} \sum_{\ell=1}^n 2 \min(Z^\Qtau_{\ell j},Z^\Qtau_{\ell i}). \]
As such, we have,
\begin{align*}
    &\sum_{j=1 }^nH_{ij}^\Qtau (\dist(i,j)+1)^\mu \\
    &\leq  \sum_{j=1 }^n \frac{1}{n} \sum_{\ell=1}^n 2 \min(Z^\Qtau_{\ell j},Z^\Qtau_{\ell i})  (\dist(i,\ell) + 1 + \dist(j,\ell)+1)^\mu\\
    &\leq \sum_{j=1 }^n \frac{1}{n} \sum_{\ell=1}^n 2 \min(Z^\Qtau_{\ell j},Z^\Qtau_{\ell i})  2^{\mu-1}\left[(\dist(i,\ell) + 1)^\mu +( \dist(j,\ell)+1)^\mu\right]\\
    &\leq \sum_{j=1 }^n \frac{1}{n} \sum_{\ell=1}^n 2  Z^\Qtau_{\ell i} 2^{\mu-1} (\dist(i,\ell) + 1)^\mu + \sum_{j=1 }^n \frac{1}{n} \sum_{\ell=1}^n 2 Z^\Qtau_{\ell j} 2^{\mu-1}( \dist(j,\ell)+1)^\mu \\
    &\leq 2^\mu \nu' + 2^\mu \nu' = 2^{\mu+1}\nu'.
\end{align*}
Similarly, we have,
\begin{align*}
    &\sum_{i=1 }^nH_{ij}^\Qtau (\dist(i,j)+1)^\mu \\
    &\leq  \sum_{i=1 }^n \frac{1}{n} \sum_{\ell=1}^n 2 \min(Z^\Qtau_{\ell j},Z^\Qtau_{\ell i})  (\dist(i,\ell) + 1 + \dist(j,\ell)+1)^\mu\\
    &\leq \sum_{i=1 }^n \frac{1}{n} \sum_{\ell=1}^n 2 \min(Z^\Qtau_{\ell j},Z^\Qtau_{\ell i})  2^{\mu-1}\left[(\dist(i,\ell) + 1)^\mu +( \dist(j,\ell)+1)^\mu\right]\\
    &\leq  \frac{2^\mu}{n} \sum_{\ell=1}^n  \sum_{i=1 }^n  Z^\Qtau_{\ell i}  (\dist(\ell,i) + 1)^\mu + \frac{2^\mu}{n} \sum_{i=1 }^n  \sum_{\ell=1}^n  Z^\Qtau_{\ell j}  ( \dist(\ell,j)+1)^\mu \\
    &\leq 2^\mu \nu' + 2^\mu \nu' = 2^{\mu+1}\nu',
\end{align*}
which completes the proof.

\subsection{Proof of \Cref{lem:log_tv}}\label{subsec:log_tv_proof}
We have, $\frac{d_i/d_j}{d'_i/d'_j} = \exp\big(\log \frac{d_i/d_j}{d'_i/d'_j} \big) \in [e^{-\epsilon},e^\epsilon]$. Then, $|d_i/d_j - d_i'/d_j'| \leq |d_i'/d_j'| |\frac{d_i/d_j}{d'_i/d'_j}-1| \leq e^c \max(e^\epsilon -1,1-e^{-\epsilon})=e^c(e^\epsilon -1)$, since $e^{\epsilon}-1\geq 1-e^{-\epsilon}$ for any $\epsilon$. Based on these observations, we have
\begin{align*}
   |d_i - d_i'|= \left|\frac{d_i}{\sum_j d_j} -\frac{d_i'}{\sum_j d_j'  }\right|=\left|\frac{1}{\sum_j d_j/d_i} -\frac{1}{\sum_j d_j'/d_i'  }\right|\leq \sum_j|d_j/d_i - d_j'/d_i'|\leq m e^c(e^\epsilon-1),
\end{align*}
where we used the fact that $\sum_j d_j/d_i\geq 1$. This concludes the proof. 

\subsection{Proof of \Cref{lem:mul_weights_convergence}}\label{subsec:proof_mul_weights}

In this subsection, we prove the convergence of the proposed algorithm in \eqref{eq:mul_weights}. This guarantees the existence of a limiting policy, which is the solution of the entropy regularized optimization problem in \eqref{eq:entropy_regu_Q}.

\textbf{Derivation of the multiplicative weights. }We derive the multiplicative weights approach using the mirror descent perspective. For this paragraph, we drop the dependence on $s$ as it is fixed. We denote $\pi = (\pi_1, \pi_2,\ldots,\pi_n)\in \Delta_{\mathcal{A}_1}\times \cdots\times \Delta_{\mathcal{A}_n}$. We consider the following $h(\pi) = \sum_{i=1}^n  h_i(\pi_i) := \sum_{i=1}^n \sum_{a_i\in\mathcal{A}_i} \pi_i(a_i)\log \pi_i(a_i)$. Its corresponding Bregman divergence is 
\begin{align*}
    D_h(\pi'||\pi) &= h(\pi') - h(\pi) - \langle \nabla h(\pi),\pi'-\pi)\\
    &= \sum_{i=1}^n\left( h_i(\pi_i') - h_i(\pi_i) - \langle \nabla h_i(\pi_i),\pi'_i-\pi_i)\right)\\
    &= \sum_{i=1}^n \sum_{a_i\in\mathcal{A}_i} \pi_i'(a_i)\log \frac{\pi_i'(a_i)}{\pi_i(a_i)}\\
    &=\sum_{i=1}^n D_{KL} (\pi_i'||\pi_i).
\end{align*}
Let the objective function in \eqref{eq:entropy_regu_Q} be $F(\pi)$. Then mirror descent update rule given the Bregman divergence is then,
\begin{align*}
    \pi^{p+1} = \argmin_{\pi\in \Delta_{\mathcal{A}_1}\times \cdots\times \Delta_{\mathcal{A}_n}} (\eta  \langle -\nabla F(\pi^p) ,\pi\rangle + D_h(\pi||\pi^p)).
\end{align*}
Given the additive structure of the Bregman divergence, we have the above rule can be written component wise as
\begin{align*}
    \pi_i^{p+1} = \argmin_{\pi_i\in \Delta_{\mathcal{A}_i}} (\eta  \langle -\nabla_{\pi_i} F(\pi^p) ,\pi_i\rangle + D_{KL}(\pi_i||\pi_i^p)).
\end{align*}
Therefore, $\pi_i^{p+1}$ will satisfy
\begin{align*}
    -\eta [ \mathbb{E}_{a_{-i}\sim \pi_{-i}^p} \Qtau(s,a_{-i},a_i) -\tau (\log \pi_i^p(a_i) + 1) ] + 1 + \log \pi_i^{p+1}(a_i)  - \log \pi_i^p(a_i) + \lambda=0,
\end{align*}
where $\lambda$ is a Lagrangian multiplier that accounts for the constraint that $\sum_{a_i}\pi_i^{p+1}(a_i)=1$. 

As a result,
\begin{align*}
    \pi_i^{p+1}(a_i)\propto \pi_i^p(a_i)^{1-\eta\tau}\exp(\eta\mathbb{E}_{a_{-i}\sim \pi_{-i}^p} \Qtau(s,a_{-i},a_i) ).
\end{align*}
which recovers the algorithm in \eqref{eq:mul_weights}. %
Using this interpretation, we prove the convergence of the algorithm~\eqref{eq:mul_weights}.
\begin{proof}[Proof of \Cref{lem:mul_weights_convergence}] 
Fixing $s$, we consider two runs of the multiplicative weight algorithm with different initializers $\pi^0$ and $\bar{\pi}^0$ respectively. Their respective trajectories are denoted as $\pi^p$ and $\bar{\pi}^p$. The update rule in \eqref{eq:mul_weights} suggests that for each $\ell$ and $a_\ell$, it holds that
\begin{align*}
    \log \pi_\ell^{p+1}(a_\ell|s) = (1-\eta\tau)\log \pi_\ell^p(a_\ell|s)+\eta\E_{a_{-\ell}\sim \pi_{-\ell}^p(s)} [\Qtau(s,a_{-\ell},a_\ell)] + c(s) ,
\end{align*}
where $c(s)$ is a normalization constant that only depends on $s$. Therefore, for any two action pairs $a_\ell, \tilde{a}_\ell$, we have
\begin{align*}
  &  \overbrace{\log \pi_\ell^{p+1}(a_\ell|s)-\log \pi_\ell^{p+1}(\tilde{a}_\ell|s)}^{\xi_\ell^{p+1}} \\
    &= (1-\eta\tau)\underbrace{(\log \pi_\ell^p(a_\ell|s) -\log \pi_\ell^p(\tilde{a}_\ell|s) )}_{\xi_\ell^p}+\eta\E_{a_{-\ell}\sim \pi_{-\ell}^p(s)} [\Qtau(s,a_{-\ell},a_\ell) - \Qtau(s,a_{-\ell},\tilde{a}_\ell)]  ,
\end{align*}
where for notational simplicity, we denote $\log \pi_\ell^{p+1}(a_\ell|s)-\log \pi_\ell^{p+1}(\tilde{a}_\ell|s) = \xi_\ell^{p+1}$ for now (where we have fixed an action pair $a_\ell,\tilde{a}_\ell$). Similar to \eqref{eq:sigma_regu_xi}, we have $|\xi_{\ell}^p|\leq \nu'/\tau$ for all $p,\ell$.

Do the same for the other trajectory starting with a different initialization policy and define $\bar{\xi}_\ell^{p+1}$ similarly. We have
\begin{align*}
  &  \xi_\ell^{p+1} - \bar{\xi}_\ell^{p+1}\\
  & = (1-\eta\tau)(\xi_\ell^{p} - \bar{\xi}_\ell^{p}) \\
    &\quad + \eta\E_{a_{-\ell}\sim \pi_{-\ell}^p(s)} [\Qtau(s,a_{-\ell},a_\ell) - \Qtau(s,a_{-\ell},\tilde{a}_\ell)] - \eta\E_{a_{-\ell}\sim \bar{\pi}_{-\ell}^p(s)} [\Qtau(s,a_{-\ell},a_\ell) - \Qtau(s,a_{-\ell},\tilde{a}_\ell)].
\end{align*}
Using \Cref{lem:diff_dist_f}, we have that,
 \begin{align*}
    | \xi_\ell^{p+1} - \bar{\xi}_\ell^{p+1}|
 & \leq (1-\eta\tau)|\xi_\ell^{p} - \bar{\xi}_\ell^{p}| + \eta \sum_{j\neq \ell} \TV(\pi_j^p(s),\bar{\pi}_j^p(s))H_{\ell j}^{\Qtau}.
\end{align*}
Denote $H^{\Qtau}_{\ell:}$ as the $\ell$-th row of $H^{\Qtau}$. Let vector $v^p\in\mathbb{R}^n$ be defined as $v^p_\ell = \TV(\pi_\ell^p(s),\bar{\pi}_\ell^p(s))$ ($\ell=1,2,\ldots,n$). Then the above update can be written as
\begin{align*}
| \xi_\ell^{p+1} - \bar{\xi}_\ell^{p+1}|
    & \leq (1-\eta\tau)|\xi_\ell^{p} - \bar{\xi}_\ell^{p}| + \eta H_{\ell:}^{\Qtau}v^p\\
    &\leq (1-\eta\tau)^{p+1}|\xi_\ell^{0} - \bar{\xi}_\ell^{0}| + \sum_{k=0}^{p}\eta(1-\eta\tau)^{p-k} H_{\ell:}^{\Qtau}v^k.
\end{align*}
Now we use $\eta=1/\tau$ and obtain
\begin{align*}
    | \xi_\ell^{p+1} - \bar{\xi}_\ell^{p+1}|
 & \leq \frac{1}{\tau} H_{\ell:}^{\Qtau}v^p.
\end{align*}
By the definition of $\xi_{\ell}^{p+1}$ and the regularity of $\pi_{\ell}^{p+1}$ it holds that $|\xi_{\ell}^{p+1}|=|\log\pi_{\ell}^{p+1}(a_{\ell}|s)-\log\pi_{\ell}^{p+1}(\tilde a_{\ell}|s)|\leq\nu'/\tau$. %
By \Cref{lem:log_tv} we have
\begin{align*}
    v_\ell^{p+1} &\leq 1/2 \actionmax^2 e^{\nu'/n\tau}\big(\exp\big(\frac{1}{\tau} H_{\ell:}^{\Qtau}v^p\big)-1\big)=c_{\TV} \big(  \exp(\eta H_{\ell:}^{\Qtau}v^p)-1\big),
\end{align*}
where $c_{\TV}=\actionmax^2 e^{\nu'/n\tau}/2$. Similar to the previous proof, we have $H_{\ell:}^{\Qtau}v^p\leq 2^{\mu+1}\nu'$ due to  \Cref{lem:q_second_interaction_matrix} and the fact that $v_\ell^p\leq1$ for all $\ell,p$. Therefore, when {$\tau\geq  \nu'2^{\mu+2}$}, we have
\begin{align}\label{eq:TV_convergence_result}
    v^{p+1}\leq 2c_{\TV} \big(  \eta H^{\Qtau}v^p\big)\leq (2\frac{1}{\tau} c_{\TV}H^{\Qtau})^{p+1}v^{0},
\end{align}
where in the first inequality we stack all $v_{\ell}^{p+1}$ together. When taking the infinity norm on both sides, we have 
\begin{align}
   \Vert v^{p+1} \Vert_\infty \leq (2\frac{1}{\tau} c_{\TV})^{p+1} \Vert H^{\Qtau}\Vert_\infty^{p+1} \Vert v^{0}\Vert_\infty\leq  (2\frac{1}{\tau} c_{\TV}  2^{\mu+1}\nu')^{p+1} = (\frac{1}{\tau}2^{\mu+1}\nu' \actionmax^2 e^{\nu'/n\tau} )^{p+1} \label{eq:multi_weight_convergence}
\end{align}
Therefore, when { $\tau > 2^{\mu+1}\nu' \actionmax^2 e^{\nu'/n\tau} $}, the right hand side of the above inequality converges to zero when $p\rightarrow\infty$.

The above argument shows the optimization problem \eqref{eq:policy_improvement_optimization} has at most one stationary point in the manifold $\Delta_{\mathrm{policy}}$ because otherwise, the algorithm \eqref{eq:mul_weights} starting from two different stationary points will not converge together. Also, the global maximizer of the optimization problem \eqref{eq:policy_improvement_optimization}, $(\zeta_1(\cdot|s),\ldots,\zeta_n(\cdot|s)) $, must be one stationary point, and this stationary point must be unique.    

As a result, for the two trajectories considered in \eqref{eq:multi_weight_convergence}, if we select one to be the trajectory that always sits at the unique stationary point, and the other is from an arbitrary starting point, we have convergence the algorithm \eqref{eq:mul_weights} will converge in the following sense:
\begin{align}
    \sup_{i\in\cN} \TV(\pi_i^p(\cdot|s),\zeta_i(\cdot|s)) \leq \left(\frac{1}{\tau}2^{\mu+1}\nu' \actionmax^2 e^{\nu'/n\tau} \right)^{p}.
\end{align}
This completes the proof.

\end{proof}

\section{Proof of \Cref{lem:policy_improvement_error}: Analysis of Policy Improvement Error} \label{sec:proof_policy_improvement_error}
Since \Cref{lem:policy_improvement_error} only concerns a single outer loop step $m$, for notational convenience we drop the dependence on $m$ and restate a slightly different version of the lemma as follows.

\begin{lemma}\label{lem:policy_improvement_error_restated}
Suppose $Q$-functions $\{ Q_i\}_{i=1}^n$ are $(\nu',\mu)$-decay. Further, suppose its truncated estimates $\{\hat{Q}_i\}_{i=1}^n$ satisfies
\begin{align*}
    \sup_{s\in\mathcal{S},a\in\mathcal{A}} |Q_i(s,a) - \hat{Q}_i(s_{\nib},a_{\nib}) | \leq \epsilon = { \frac{c_{pe}}{(\beta+1)^\mu} }. 
\end{align*}
Consider the following updates: 
\begin{align}\label{eq:one_step_analysis:policy_improvement_err}
    \hat{\pi}_i^{p+1}(a_i|s_{\nik})\propto\hat{\pi}_i^p(a_i|s_{\nik})^{1-\eta\tau}\exp\Big(\eta \mathbb{E}_{a_{j}\sim\hat{\pi}_{j}^{p}([\ext{s_{\nik}}]_{\njk}),j\in \nik/\{i\}}\big[\frac{1}{n}\sum_{j\in \nik} \hat{Q}_j([\ext{s_{\nik}}]_{\njb},[\ext{a_{\nik}}]_{\njb \backslash i},a_i)\big]\Big).
    \end{align}
Let $\zeta = (\zeta_1,\ldots,\zeta_n)$ be the result of the policy improvement w.r.t. $Q=1/n\sum_{i}Q_i$, i.e.
\begin{align}\label{eq:one_step_analysis:argmax}
 (\zeta_1(\cdot|s),\ldots,\zeta_n(\cdot|s))\leftarrow   \argmax_{\pi_1(\cdot|s),\ldots,\pi_n(\cdot|s)} \mathbb{E}_{a_i\sim \pi_i(\cdot|s)} [\Qtau(s,a_1,\ldots,a_n)]+ \tau 
 \sum_{i=1}^n H(\pi_i(\cdot|s)).
\end{align}
Let $\tilde\sigma'=\frac{1}{\tau}(\frac{2f(\kappa)c_{pe}}{n(\beta+1)^\mu}+\frac{\nu'}{n})$. Assume $\tau\geq \max( 6 \actionmax^2 e^{\tilde\sigma'} 2^{\mu+1}\nu'  ,2(2^{\mu+4}\nu' + 2\frac{c_{pe}}{(\beta+1)^\mu}), 4(2^{\mu+1}\nu'+f(\kappa) (\kappa+1)^\mu \frac{c_{pe}}{(\beta+1)^\mu}))$ and $\eta=\frac{1}{\tau}$. Then we have
\begin{enumerate}
    \item[(a)]$\forall p$, $\hat{\pi}_i^p$ is $\tilde\sigma'$-regular.
    \item[(b)] $\forall p$, 
$\hat{\pi}_i^p$ is $(\tilde\nu,\mu)$-decay where $\tilde{\nu}=\frac{\actionmax^2e^{\tilde\sigma'}\tilde\nu_H}{\tau-\actionmax^2e^{\tilde\sigma'}\tilde\nu_H}$ and $\tilde{\nu}_H=2^{\mu+1}\nu'+f(\kappa) (\kappa+1)^\mu \frac{c_{pe}}{(\beta+1)^\mu}$.    
    \item[(c)]When $p\geq -\log_2{\frac{4 + \frac{c_{pe}}{2^{\mu}\nu'}}{3(\kappa/2+1)^\mu}}$, we have
\begin{align*}
    \sup_{s\in \mathcal{S}}\TV(\hat{\pi}_i^p(\cdot|s_{\nik}),\zeta_i(\cdot|s )) \leq \frac{4 + \frac{c_{pe}}{2^{\mu}\nu'}}{(\kappa/2+1)^\mu},
\end{align*}
and
\begin{align}
    \mathcal{T} \Vtau  - \Vtau^{\hat{\pi}^{p}} \leq  \frac{1}{1-\gamma}\bigg(\frac{2f(\kappa)c_{pe}}{(\beta+1)^\mu}+2\nu'\bigg) \frac{4 + \frac{c_{pe}}{2^{\mu}\nu'}}{(\kappa/2+1)^\mu} \mathbf{1}.
\end{align}
\end{enumerate}
\end{lemma}

Based on \Cref{lem:policy_improvement_error_restated}, we now prove \Cref{lem:policy_improvement_error} as follows. Recall the parameter settings in \Cref{thm:convergence_svi}, which implies
\begin{align*}
    \tau\geq 40\times 2^\mu \bar{r}\frac{4-3\gamma}{4-5\gamma}\actionmax^2 e \geq 10 \actionmax^2  e^{\frac{2\nu'}{\tau n}} 2^{\mu+2}\nu' =10 \actionmax^2  e^{\tilde\sigma} 2^{\mu+2}\nu'.
\end{align*}
Further, the parameter $\beta = \frac{\kappa+1}{2}(\frac{2 f(\kappa) c_{pe}}{\nu'})^{\frac{1}{\mu}}$ ensures that $f(\kappa) (\kappa+1)^\mu \frac{c_{pe}}{(\beta+1)^\mu}\leq 2^\mu \nu' $ and $\tilde\sigma' \leq \frac{2\nu'}{n\tau} = \tilde{\sigma}$. 
As a result, one can easily verify that the lower bound assumption on $\tau$ in \Cref{lem:policy_improvement_error_restated} holds. 

Therefore, by setting the $Q$-functions in \Cref{lem:policy_improvement_error_restated} as $\{Q_i^{\hat{\zeta}^m}\}_{i\in\cN}$, we obtain that the returned policy $\hat{\zeta}^{m+1}$ is $\tilde\sigma$ regular and $(\tilde{\nu},\mu)$-decay by part (a) and part (b) of \Cref{lem:policy_improvement_error_restated}. 
Further, we can check that $\tilde{\nu}_H \leq 2^{\mu+2}\nu'$ and $\tilde{\nu}\leq \frac{1}{9}$. 
Therefore, $\hat{\zeta}^{m+1}\in\Delta_{\frac{1}{9},\mu,\tilde\sigma}$ and we can apply \Cref{lem:policy-q} (more specifically, \Cref{eq:proof_policy_q_nuprime}) and get that the local $Q$-functions $\{Q_i^{\hat{\zeta}^{m+1}}\}_{i\in\cN}$ of the output policy $\hat{\zeta}^{m+1}$ satisfies $(\nu'_{\mathrm{next}},\mu)$-decay, where 
   $\nu'_{\mathrm{next}} =  \bar{r} + \frac{\gamma(\bar{r}+n\tau\tilde{\sigma})}{8(1-\gamma)} $, and we can easily check $\nu'_{\mathrm{next}}\leq \nu'$. As a result, the local $Q$-functions $\{Q_i^{\hat{\zeta}^{m+1}}\}_{i\in\cN}$ satisfies $(\nu',\mu)$ decay property.  
   
Also, by part (c) of \Cref{lem:policy_improvement_error_restated}, we have
\[  \mathcal{T} \Vtau  - \Vtau^{\hat{\zeta}^{m+1}} \leq  \frac{3\nu' }{1-\gamma} \frac{4 + \frac{c_{pe}}{2^{\mu}\nu'}}{(\kappa/2+1)^\mu} \mathbf{1} = \frac{3(4-3\gamma)\bar{r} }{(1-\gamma)(4-5\gamma)} \frac{4 + \frac{c_{pe} (4-5\gamma)}{2^{\mu}(4-3\gamma)\bar{r}}}{(\kappa/2+1)^\mu} \mathbf{1} := \frac{c_{pe}}{(\kappa/2+1)^\mu} \mathbf{1}.\]
This concludes the proof of \Cref{lem:policy_improvement_error}. 
We will next prove the three parts of \Cref{lem:policy_improvement_error_restated} in the following three subsections.

\subsection{Proof of Part (a)}
Recall that in \eqref{eq:one_step_analysis:policy_improvement_err}, $\hat\pi_i^p$ is defined and updated only on the local state $s_{\nik}$. For each $i$, we define $\tilde{\pi}_i^p (\cdot|s) = \tilde{\pi}_i^p (\cdot|s_\nik,s_{\nminusik}) := \hat{\pi}_i^p(\cdot|s_{\nik})$ which is an extended version of $\hat{\pi}_i^p$ with nominal (but not actual) dependence on $s_{\nminusik}$. Then, $\tilde{\pi}_i^p (\cdot|s)$ obeys, 
\begin{align}\label{eq:one_step_analysis:policy_improvement_err_ext}
    \tilde{\pi}_i^{p+1}(a_i|s) \propto \tilde{\pi}_i^p(a_i|s)^{1-\eta\tau}\exp\bigg(\eta\mathbb{E}_{a_j\sim \tilde{\pi}_j^p(\cdot|\ext{s_{\nik}}),\forall j \neq  i}\bigg[\frac{1}{n}\sum_{j\in \nik}\hat{\Qtau}_j([\ext{s_{\nik}}]_{\njb},[\ext{a_{\nik}}]_{\njb\backslash i},a_i)\bigg]\bigg),
\end{align}
which is equivalent to \eqref{eq:one_step_analysis:policy_improvement_err}. 
We also use the following abbreviation notation $\tilde{\Qtau}^i(s,a_{-i},a_i)=\frac{1}{n}\sum_{j\in \nik}\hat{\Qtau}_j([\ext{s_{\nik}}]_{\njb},[\ext{a_{\nik}}]_{\njb\backslash i},a_i)$ which again has nominal (but not actual) dependence on $(s_{\nminusik},a_{\nminusik})$. Then the update rule \eqref{eq:one_step_analysis:policy_improvement_err} can be further written as for all $i$ and $s\in\mathcal{S}$
\begin{align}\label{eq:one_step_analysis:policy_improvement_abbr}
    \tilde\pi_i^{p+1}(a_i|s)\propto\tilde\pi_i^p(a_i|s)^{1-\eta\tau}\exp{(\eta\mathbb{E}_{a_{j}\sim \tilde\pi_{j}^p(\cdot|\ext{s_{\nik}}),j\neq i}[\tilde{\Qtau}^i( s,a_{-i},a_i)])}.
\end{align}

Now consider a given agent $i$ and fix two actions $a_i,a_i'$. We define $\tilde\xi_i^{p+1}(s)=\log\tilde\pi_i^{p+1}(a_i| s)-\log\tilde\pi_i^{p+1}(a'_i| s)$ which according to \eqref{eq:one_step_analysis:policy_improvement_abbr}, follows the following update,
\begin{align}\label{eq:xi_update}
    \tilde\xi_i^{p+1}(s)=(1-\eta\tau)\tilde\xi_i^{p}(s)+\eta\mathbb{E}_{a_{j}\sim\tilde\pi_{j}^p(\cdot|\ext{s_{\nik}}),j\neq i}[\tilde{\Qtau}^i(s,a_{-i},a_i)-\tilde{\Qtau}^i(s,a_{-i},a_i')].
\end{align}
We now give an upper bound for $|\tilde{Q}^i(s,a_{-i},a_i)-\tilde{Q}^i(s,a_{-i},a_i')|$ as follows:
\begin{align}\label{eq:tilde_q_difference}
|\tilde{Q}^i(s,a_{-i},a_i)-\tilde{Q}^i(s,a_{-i},a_i')|&\leq \frac{1}{n}\sum_{j\in \nik} |\hat{Q}_j([\overline{s_{\nik}}]_{\njb},[\overline{a_{\nik}}]_{\njb\backslash i},a_i)-\hat{Q}_j([\overline{s_{\nik}}]_{\njb},[\overline{a_{\nik}}]_{\njb\backslash i},a_i')|\notag\\
&= \frac{1}{n}\sum_{j\in \nik} \Big( \Big|\hat{Q}_j([\overline{s_{\nik}}]_{\njb},[\overline{a_{\nik}}]_{\njb\backslash i},a_i)-Q_j(\overline{s_{\nik}},[\overline{a_{\nik}}]_{-i},a_i)\Big|\notag\\
&\qquad +\Big| Q_j(\overline{s_{\nik}},[\overline{a_{\nik}}]_{-i},a_i')-\hat{Q}_j([\overline{s_{\nik}}]_{\njb},[\overline{a_{\nik}}]_{\njb\backslash i},a_i')\Big|\notag\\
&\qquad+\Big|Q_j(\overline{s_{\nik}},[\overline{a_{\nik}}]_{-i},a_i)-Q_j(\overline{s_{\nik}},[\overline{a_{\nik}}]_{-i},a_i')\Big|\Big)\notag\\
&\leq \frac{1}{n}\sum_{j\in \nik}(2\epsilon+Z^{Q}_{ji})\notag\\
&\leq 2\frac{|\nik|}{n}\epsilon+\frac{\nu'}{n},
\end{align}
where the second inequality is due to \Cref{def:Z_Q} and \eqref{lem:policy_evaluation_error}.  As a result, we have
\begin{align*}
|\tilde\xi_i^{p+1}(s)|&=|(1-\eta\tau)\tilde\xi_i^{p}(s)+\eta\mathbb{E}_{a_{j}\sim\tilde\pi_{j}^p(\cdot|\overline{\nik}),j\neq i}[\tilde{Q}^i(s,a_{-i},a_i)-\tilde{Q}^i(s,a_{-i},\tilde{a}_i)]|\\
&\leq (1-\eta\tau)|\tilde\xi_i^{p}(s)|+\eta\bigg(2\frac{|\nik|}{n}\epsilon+\frac{\nu'}{n}\bigg).
\end{align*}
Considering the fact that we start from a uniform policy, we have $\forall p$, $|\tilde\xi_i^p(s)|\leq \frac{1}{\tau}(2\frac{|\nik|}{n}\epsilon+\frac{\nu'}{n})\leq \tilde{\sigma}'=\frac{1}{\tau}(\frac{2f(\kappa)c_{pe}}{n(\beta+1)^\mu}+\frac{\nu'}{n})$, where we have used $\epsilon=\frac{c_{pe}}{(\beta+1)^\mu}$.

\subsection{Proof of Part (b)} 
The proof is a similar but slightly more complicated version of the proof of \Cref{lem:q-policy}. Similar to the steps in the proof of \Cref{lem:q-policy}, we fix $i,j\in\mathcal{N}$, and let $s=(s_j,s_{-j})$ and $s'=(s'_j,s_{-j})$. Then, we calculate, 
\begin{align*}
&\tilde\xi_i^{p+1}(s)-\tilde\xi_i^{p+1}(s')\\&=(1-\eta\tau)(\tilde\xi_i^{p}(s)-\tilde\xi_i^{p}(s'))+\eta\mathbb{E}_{a_{\ell}\sim\tilde\pi_{\ell}^p(\cdot|\overline{s_{\nik}}),\ell\neq i}[\tilde{Q}^i(s,a_{-i},a_i)-\tilde{Q}^i(s,a_{-i},\tilde{a}_i)]\\
&\quad -\eta\mathbb{E}_{a_{\ell}\sim\tilde\pi_{\ell}^p(\cdot|\overline{s'_{\nik}}),\ell\neq i}[\tilde{Q}^i(s',a_{-i},a_i)-\tilde{Q}^i(s',a_{-i},\tilde{a}_i)]\\
&=(1-\eta\tau)(\tilde\xi_i^{p}(s)-\tilde\xi_i^{p}(s'))\\
&+\eta\mathbb{E}_{a_{\ell}\sim\tilde\pi_{\ell}^p(\cdot|\overline{s_{\nik}}),\ell\neq i}[\tilde{Q}^i(s,a_{-i},a_i)-\tilde{Q}^i(s,a_{-i},\tilde{a}_i)]-\eta\mathbb{E}_{a_{\ell}\sim\tilde\pi_{\ell}^p(\cdot|\overline{s'_{\nik}}),\ell\neq i}[\tilde{Q}^i(s,a_{-i},a_i)-\tilde{Q}^i(s,a_{-i},\tilde{a}_i)]\\
&+\eta\mathbb{E}_{a_{\ell}\sim\tilde\pi_{\ell}^p(\cdot|\overline{s'_{\nik}}),\ell\neq i}[\tilde{Q}^i(s,a_{-i},a_i)-\tilde{Q}^i(s,a_{-i},\tilde{a}_i)]-\eta\mathbb{E}_{a_{\ell}\sim\tilde\pi_{\ell}^p(\cdot|\overline{s'_{\nik}}),\ell\neq i}[\tilde{Q}^i(s',a_{-i},a_i)-\tilde{Q}^i(s',a_{-i},\tilde{a}_i)]\\
&\leq (1-\eta\tau)(\tilde\xi_i^{p}(s)-\tilde\xi_i^{p}(s'))+ \eta\sum_{\ell\neq i}\TV(\pi_{\ell}^p(\cdot|\overline{s_{\nik}}),\pi_{\ell}^p(\cdot|\overline{s'_{\nik}}))H_{i\ell}^{\tilde{Q}^i} +\eta H_{ij}^{\tilde{Q}^i}.
\end{align*}
Recall the notation $Z_{\ell j}^{\tilde{\pi}^p}=\sup_{s_j,s_j',s_{-j}} \TV(\tilde\pi_{\ell}^p(\cdot|s_j,s_{-j}),\tilde\pi_{\ell}^p(\cdot|s_j',s_{-j}))$. We then have that,
\begin{align*}
\tilde\xi_i^{p+1}(s)-\tilde\xi_i^{p+1}(s')&\leq (1-\eta\tau)(\tilde\xi_i^{p}(s)-\tilde\xi_i^{p}(s'))+\eta\sum_{\ell\neq i}Z_{\ell j}^{\tilde{\pi}^p}H_{i\ell}^{\tilde{Q}^i}+\eta H_{ij}^{\tilde{Q}^i}\\
&\leq (1-\eta\tau)(\tilde\xi_i^{p}(s)-\tilde\xi_i^{p}(s'))+ \eta H_{i:}^{\tilde{Q}^i}Z_{: j}^{\tilde{\pi}^p}+\eta H_{ij}^{\tilde{Q}^i}\\
&\leq (1-\eta\tau)^{p+1}(\tilde\xi_i^{0}(s)-\tilde\xi_i^{0}(s'))+\sum_{k=0}^p \eta(1-\eta\tau)^{p-k}(H_{i:}^{\tilde{Q}^i}Z_{: j}^{\tilde{\pi}^k}+H_{ij}^{\tilde{Q}^i}).
\end{align*}
{Recalling $\eta=\frac{1}{\tau}$}, we have,
\begin{align}
    \tilde\xi_i^{p+1}(s)-\tilde\xi_i^{p+1}(s')\leq \frac{1}{\tau}(H_{i:}^{\tilde{Q}^i}Z_{: j}^{\tilde{\pi}^p}+H_{ij}^{\tilde{Q}^i}). 
\end{align}
Since $\tilde\pi_i^p$ is $\tilde\sigma'$-regular, we can use \Cref{lem:log_tv} to obtain
\begin{align}
    \TV(\tilde\pi_{i}^{p+1}(\cdot|s),\tilde\pi_{i}^{p+1}(\cdot|s'))\leq \frac{\actionmax^2e^{\tilde\sigma'}}{2}\left(\exp{(\frac{1}{\tau}(H_{i:}^{\tilde{Q}^i}Z_{: j}^{\tilde{\pi}^p}+H_{ij}^{\tilde{Q}^i}))}-1\right).  \label{eq:proof_pi_error:decay_tv_exp_bound}
\end{align}
We have the following Lemma regarding $H^{\tilde{Q}^i}$. 
\begin{lemma}\label{lem:H_tilde_q_i}
$H^{\tilde{Q}^i} $ satisfies, 
\[ \sum_{\ell\in\mathcal{N}} H_{i\ell}^{\tilde{Q}^i} (\dist(i,\ell)+1)^\mu \leq 2^{\mu+1}\nu'+f(\kappa) (\kappa+1)^\mu \frac{c_{pe}}{(\beta+1)^\mu}:=\tilde{\nu}_{H}. \]
\end{lemma}
{Therefore, $\frac{1}{\tau}(H_{i:}^{\tilde{Q}^i}Z_{: j}^{\tilde{\pi}^p}+H_{ij}^{\tilde{Q}^i})\leq \frac{1}{\tau}(2\tilde{\nu}_H)\leq \frac{1}{2}$, which can be satisfied when $\tau \geq 4\tilde{\nu}_H$}. As a result, we can simplify the above equation to
\begin{align*}
    \TV(\tilde\pi_{i}^{p+1}(\cdot|s),\tilde\pi_{i}^{p+1}(\cdot|s'))\leq \actionmax^2e^{\tilde\sigma'}\frac{1}{\tau}(H_{i:}^{\tilde{Q}^i}Z_{: j}^{\tilde{\pi}^p}+H_{ij}^{\tilde{Q}^i}). 
\end{align*}
In the above equation, taking the sup over $s_{-j},s_j,s_j'$, the left hand side would become $Z_{ij}^{\tilde{\pi}^{p+1}}$. As a result, if we repeat the above equation for all $i,j$ pairs, we get,
\begin{align}
    Z^{\tilde{\pi}^{p+1}}\leq |\actionmax|^2e^{\tilde\sigma'}\frac{1}{\tau} (\tilde{H}Z^{\tilde{\pi}^p}+\tilde{H}), \label{eq:proof_pi_error:Z_recursion}
\end{align}
where $\tilde{H}$ is a matrix whose $i$th row is $H^{\tilde{Q}^i}_{i,:}$.
Suppose $Z^{\tilde\pi^p}$ is $(\tilde\nu_p,\mu)$-decay. Clearly we have $\tilde\nu_p = 0$, and by \eqref{eq:proof_pi_error:Z_recursion}, we have that
\begin{align}
    \tilde\nu_{p+1}\leq \actionmax^2e^{\tilde\sigma'}\frac{1}{\tau} (\tilde\nu_H\tilde\nu_p+\tilde\nu_H)\leq \frac{\actionmax^2e^{\tilde\sigma'}\tilde\nu_H}{\tau-\actionmax^2e^{\tilde\sigma'}\tilde\nu_H}.
\end{align}

Now it remains to prove \Cref{lem:H_tilde_q_i} to finish the proof of Part (b) of \Cref{lem:policy_improvement_error_restated}.
\begin{proof}[Proof of \Cref{lem:H_tilde_q_i}]
we give an upper bound for the term $|[\tilde{Q}^i(z_j,z_k,z_{-(j,k)})-\tilde{Q}^i(z_j',z_k,z_{-(j,k)})]-[\tilde{Q}^i(z_j,z_k',z_{-(j,k)})-\tilde{Q}^i(z_j',z_k',z_{-(j,k)})] |$.
For $j,k\in\nik$, consider $\tilde{Q}^i(z_j,z_k,z_{-(j,k)})-\tilde{Q}^i(z_j',z_k,z_{-(j,k)})$, we have that
\begin{align*}
    &\tilde{Q}^i(z_j,z_k,z_{-(j,k)})-\tilde{Q}^i(z_j',z_k,z_{-(j,k)}) - [\tilde{Q}^i(z_j,z_k',z_{-(j,k)})-\tilde{Q}^i(z_j',z_k',z_{-(j,k)})]\\
    &=\frac{1}{n}\sum_{\ell\in\nik\cap \njb\cap \nkb} \Big[\hat{Q}_\ell(z_j,z_k,[\overline{z_{\nik}}]_{\nlb\backslash (j,k)})-\hat{Q}_\ell(z_j',z_k,[\overline{z_{\nik}}]_{\nlb\backslash (j,k)}) \\
    &\qquad-\hat{Q}_\ell(z_j,z_k',[\overline{z_{\nik}}]_{\nlb\backslash (j,k)}) +\hat{Q}_\ell(z_j',z_k',[\overline{z_{\nik}}]_{\nlb\backslash (j,k)})\Big].
\end{align*}
Notice that
\begin{align}
    |\hat{Q}_\ell(z_j,z_k,[\overline{z_{\nik}}]_{\nlb\backslash (j,k)})-Q_\ell(z_j,z_k,[\overline{z_{\nik}}]_{-(j,k)})|\leq \epsilon.
\end{align}
We can use triangle inequality to obtain
\begin{align*}
    &\big|\big[\tilde{Q}^i(z_j,z_k,z_{-(j,k)})-\tilde{Q}^i(z_j',z_k,z_{-(j,k)})\big]-\big[\tilde{Q}^i(z_j,z_k',z_{-(j,k)})-\tilde{Q}^i(z_j',z_k',z_{-(j,k)})\big] \big|\\
    &\leq  4\frac{|\nik|}{n}\epsilon +\frac{1}{n}\sum_{\ell\in\nik\cap\njb\cap\nkb}\Big|Q_\ell(z_j,z_k,[\ext{z_{\nik}}]_{-(j,k)})-Q_\ell(z_j',z_k,[\ext{z_{\nik}}]_{-(j,k)})\\
    &\qquad -Q_\ell(z_j,z_k',[\ext{z_{\nik}}]_{-(j,k)})+Q_\ell(z_j',z_k',[\ext{z_{\nik}}]_{-(j,k)}) \Big|
\end{align*}
Thus, we have
\begin{align}
    H^{\tilde{Q}^i}_{jk}\leq H^{Q}_{jk}+4\frac{|\nik|}{n}\epsilon.
\end{align}
For the decay property, note that $\forall j\notin \nik, H_{ij}^{\tilde{Q}^i}=0$, we have
\begin{align}
    \sum_{j\in\mathcal{N}} H^{\tilde{Q}^i}_{ij}(\dist(i,j)+1)^\mu
    \leq & \sum_{j\in \nik}(H^{Q}_{ij}+4\frac{|\nik|}{n}\epsilon)(\dist(i,j)+1)^\mu\notag\\
    \leq & 2^{\mu+1}\nu'+|\nik|\epsilon (\kappa+1)^\mu,
\end{align}
which completes the proof.
\end{proof}

\subsection{Proof of Part (c)} 

Recall that in \Cref{lem:mul_weights_convergence}, we showed that the following update will converge to $\zeta$, the policy resulting from the exact policy improvement w.r.t. $Q$ (with appropriate choices of $\eta$ and $\tau$). 
\begin{align} \label{eq:proof_alg:recallprototype}
    \pi_i^{p+1}(a_i|s)\propto \pi_i^p(a_i|s)^{1-\eta\tau}\exp{(\eta\mathbb{E}_{a_{j}\sim \pi_{j}^p(\cdot|s),j\neq i}[Q(s,a_{-i},a_i)])}, \forall i\in\mathcal{N}, s\in\mathcal{S}.
\end{align}
For two arbitrary actions $a_i,a'_i$ we define $\xi_i^{p+1}(s)=\log\pi_i^{p+1}(a_i|s)-\log\pi_i^{p+1}(a'_i|s)$.
The central step of our proof is to track the difference between our algorithm \eqref{eq:one_step_analysis:policy_improvement_abbr} and the algorithm \eqref{eq:proof_alg:recallprototype}. 
To this end, we will provide a recursive bound on
\begin{align}\label{eq:tv_dist_pi_tilde_pi_approx_version}
    v_i^p = \sup_{s\in\mathcal{S}} \TV (\pi_i^p(\cdot|s), \tilde{\pi}_i^p(\cdot|s) ).
\end{align}
Recall \eqref{eq:xi_update} as follows.
\begin{align*}
    \tilde\xi_i^{p+1}(s)=(1-\eta\tau)\tilde\xi_i^{p}(s)+\eta\mathbb{E}_{a_{j}\sim\tilde\pi_{j}^p(\cdot|\ext{s_{\nik}}),j\neq i}[\tilde{\Qtau}^i(s,a_{-i},a_i)-\tilde{\Qtau}^i(s,a_{-i},a_i')].
\end{align*}
Now let's track the difference of our algorithm and the prototype algorithm. For all $i$, and for all $s$, we have
\begin{align}\label{eq:xi_difference_decomposition}
    \xi_i^{p+1}(s)-\tilde\xi_i^{p+1}(s)
    &=(1-\eta\tau)(\xi_i^{p}(s)-\tilde\xi_i^{p}(s))+\eta{\mathbb{E}_{a_{-i}\sim \pi_{-i}^p(\cdot|s)}[\Qtau(s,a_{-i},a_i)-\Qtau(s,a_{-i},a_i')]}\notag\\
    &\qquad -\eta{\mathbb{E}_{a_{j}\sim\tilde\pi_{j}^p(\cdot|\ext{s_{\cN_i^{\kappa}}}),j\neq i}[\tilde{\Qtau}^i(s,a_{-i},a_i)-\tilde{\Qtau}^i(s,a_{-i},a_i')]}\notag\\
    &=(1-\eta\tau)(\xi_i^{p}(s)-\tilde\xi_i^{p}(s)) + \eta(E_0 + E_1 + E_2 + E_3),
\end{align}
where we decompose the last two terms into the following four quantities:
\begin{align}
E_0&={ \mathbb{E}_{a_{-i}\sim \pi_{-i}^p(\cdot|s)}[\Qtau(s,a_{-i},a_i)-\Qtau(s,a_{-i},a_i')]}-\mathbb{E}_{a_{j}\sim\tilde\pi_{j}^p(\cdot|s),j\neq i}[\Qtau(s,a_{-i},a_i)-\Qtau(s,a_{-i},a_i')]\notag\\
E_1&=\mathbb{E}_{a_{j}\sim \tilde{\pi}_{j}^p(\cdot|s),j\neq i}[\Qtau(s,a_{-i},a_i)-\Qtau(s,a_{-i},a_i')]-\mathbb{E}_{a_{j}\sim\tilde\pi_{j}^p(\cdot|\ext{s_{\nik}}),j\neq i}[\Qtau(s,a_{-i},a_i)-\Qtau(s,a_{-i},a_i')]\notag\\
E_2&=\mathbb{E}_{a_{j}\sim\tilde\pi_{j}^p(\cdot|\ext{s_{\nik}}),j\neq i}[\Qtau(s,a_{-i},a_i)-\Qtau(s,a_{-i},a_i')]\notag\\
&\qquad-\mathbb{E}_{a_{j}\sim\tilde\pi_{j}^p(\cdot|\ext{s_{\nik}}),j\neq i}[\Qtau(\ext{s_{\nik}},[\ext{a_{\nik}}]_{-i},a_i)-\Qtau(\ext{s_{\nik}},[\ext{a_{\nik}}]_{-i},a_i')]\notag\\
E_3&=\mathbb{E}_{a_{j}\sim\tilde\pi_{j}^p(\cdot|\ext{s_{\nik}}),j\neq i}[\Qtau(\ext{s_{\nik}},[\ext{a_{\nik}}]_{-i},a_i)-\Qtau(\ext{s_{\nik}},[\ext{a_{\nik}}]_{-i},a_i')]\notag\\
&\qquad-{\mathbb{E}_{a_{j}\sim\tilde\pi_{j}^p(\cdot|\ext{s_{\cN_i^\kappa}}),j\neq i}[\tilde{\Qtau}^i(s,a_{-i},a_i)-\tilde{\Qtau}^i(s,a_{-i},a_i')]}.\notag
\end{align}
First note that $E_0$ is the difference caused by sampling from $\pi_j^p(\cdot|s)$ and $\tilde\pi_j^p(\cdot|s)$ while $E_1$ is the difference induced by sampling from $\tilde{\pi}_{j}^p(\cdot|s)$ and $\tilde{\pi}_{j}^p(\cdot|\ext{s_{\nik}})$. Second, term $E_2$ is the error between querying $\Qtau$ functions at $s$ and $a_{-i}$ and at $\ext{s_{\nik}}$ and $[\ext{a_{\nik}}]_{-i}$. Finally, $E_3$ accounts for the difference in $\Qtau$ and $\tilde \Qtau$. In what follows we bound these quantities one by one. 

\textbf{Bound on $E_0$.}
\Cref{lem:diff_dist_f} can be directly applied to bound $E_0$, which leads to 
\begin{align}
    |E_0|\leq\sum_{\ell\neq i}\TV(\pi_\ell^p(\cdot|s),\tilde\pi_\ell^p(\cdot|s))H_{i\ell}^\Qtau\leq \sum_{\ell\neq i} H_{i\ell}^Q v_\ell^p, 
\end{align}
where in the last inequality we have used the definition of $v_\ell^p$ in \eqref{eq:tv_dist_pi_tilde_pi_approx_version}.

\textbf{Bound on $E_1$. } By \Cref{lem:diff_dist_f}, $E_1$ can be bounded by
\begin{align*}
    |E_1|&\leq \sum_{j\neq i} \TV(\tilde{\pi}_j^p(\cdot|s), \tilde{\pi}_j^p(\cdot|\ext{s_{\nik}} )) H_{ij}^Q\\
    &\leq    \sum_{j\neq i} \left(\TV(\tilde{\pi}_j^p(\cdot|s),\pi_j^p(\cdot|s))+\TV(\pi_j^p(\cdot|s),\pi_j^p(\cdot|\ext{s_{\nik}}))+\TV(\pi_j^p(\cdot|\ext{s_{\nik}}),\tilde{\pi}_j^p(\cdot|\ext{s_{\nik}}))\right)H_{ij}^Q \\
    &\leq 2\sum_{j\neq i} v_j^p H_{ij}^Q + \sum_{j\neq i}\TV(\pi_j^p(\cdot|s),\pi_j^p(\cdot|\ext{s_{\nik}})) H_{ij}^Q\\
    &\leq  2\sum_{j\neq i} v_j^p H_{ij}^Q + \sum_{j\not\in \mathcal{N}_i^{\kappa/2}}  H_{ij}^Q + \sum_{j\in \mathcal{N}_i^{\kappa/2},j\neq i} \sum_{\ell \not\in \nik} Z^{\pi^p}_{j\ell} H_{ij}^Q \\
    &\leq 2\sum_{j\neq i} v_j^p H_{ij}^Q + \sum_{j\not\in \mathcal{N}_i^{\kappa/2}}  H_{ij}^Q+ \sum_{j\in \mathcal{N}_i^{\kappa/2},j\neq i} \sum_{\ell: \dist(\ell,j)>\frac{\kappa}{2}} Z^{\pi^p}_{j\ell} H_{ij}^Q\\
    &\leq 2\sum_{j\neq i} v_j^p H_{ij}^Q + \frac{2^{\mu+1}\nu'+ 2^{\mu+1}\nu' \frac{2^{\mu+1}\nu'\actionmax^2e^{\frac{\nu'}{n\tau}}}{\tau-2^{\mu+1}\nu' \actionmax^2e^{\frac{\nu'}{n\tau}}}}{(\kappa/2+1)^\mu} \\
    &\leq 2\sum_{j\neq i} v_j^p H_{ij}^Q + \frac{2^{\mu+2}\nu'}{(\frac{\kappa}{2}+1)^\mu}, 
    \end{align*}
where in the fourth inequality we used similar augments like we did in \eqref{eq:tv_dist_triangle_decompose} and \eqref{eq:proof_thm1:tv_bound}, and in the last two inequalities we have used \Cref{lem:q_second_interaction_matrix} and \Cref{cor:q_policy_pi_p_decay} and the conditions on 
$\tau>2^{\mu+2}\nu' \actionmax^2 e^{\frac{\nu'}{n\tau}} $.

\textbf{Bound on $E_2$.} To bound $E_2$, we only have to bound
\begin{align*}
    &\left|\Qtau(s,a_{-i},a_i)-\Qtau(s,a_{-i},a_i')-\Qtau(\ext{s_{\nik}},[\ext{a_{\nik}}]_{-i},a_i)+\Qtau(\ext{s_{\nik}},[\ext{a_{\nik}}]_{-i},a_i')\right|\\
    &\leq\frac{1}{n}\sum_{j=1}^n \left| \Qtau_j(s,a_{-i},a_i)-\Qtau_j(s,a_{-i},a_i')-\Qtau_j(\ext{s_{\nik}},[\ext{a_{\nik}}]_{-i},a_i)+\Qtau_j(\ext{s_{\nik}},[\ext{a_{\nik}}]_{-i},a_i')\right|,
\end{align*}
which further implies that we only have to upper bound each term inside the above summation. To this end, we introduce the following lemma.
\begin{lemma}\label{lem:E_2_bound}
Under the scenario of \Cref{lem:policy_improvement_error}, we have the following bound:
\begin{align*}
    |\Qtau_j(s,a_{-i},a_i)-\Qtau_j(s,a_{-i},a_i')-\Qtau_j(\ext{s_{\nik}},[\ext{a_{\nik}}]_{-i},a_i)+\Qtau_j(\ext{s_{\nik}},[\ext{a_{\nik}}]_{-i},a_i')|\leq \frac{2\nu'}{(\kappa/2+1)^\mu}.
\end{align*}
\end{lemma}
\begin{proof}
Indeed, we have two ways to upper bound this term, the first way is
\begin{align*}
    &|\Qtau_j(s,a_{-i},a_i)-\Qtau_j(s,a_{-i},a_i')-\Qtau_j(\ext{s_{\nik}},[\ext{a_{\nik}}]_{-i},a_i)+\Qtau_j(\ext{s_{\nik}},[\ext{a_{\nik}}]_{-i},a_i')|\\
&\leq  |\Qtau_j(s,a_{-i},a_i)-\Qtau_j(s,a_{-i},a_i')|+|\Qtau_j(\ext{s_{\nik}},[\ext{a_{\nik}}]_{-i},a_i)-\Qtau_j(\ext{s_{\nik}},[\ext{a_{\nik}}]_{-i},a_i')|\\
&\leq  2\frac{\nu'}{(\dist(i,j)+1)^\mu},
\end{align*}
where in the last inequality, we have used the $(\nu',\mu)$-decay property of $\{Q_j\}$. Also we have another way for this bound, that is,
\begin{align*}
&|\Qtau_j(s,a_{-i},a_i)-\Qtau_j(s,a_{-i},a_i')-\Qtau_j(\ext{s_{\nik}},[\ext{a_{\nik}}]_{-i},a_i)+\Qtau_j(\ext{s_{\nik}},[\ext{a_{\nik}}]_{-i},a_i')|\\
&\leq |\Qtau_j(s,a_{-i},a_i)-\Qtau_j(\ext{s_{\nik}},[\ext{a_{\nik}}]_{-i},a_i)|+|\Qtau_j(s,a_{-i},a_i')-\Qtau_j(\ext{s_{\nik}},[\ext{a_{\nik}}]_{-i},a_i')|\\
&\leq 2\frac{\nu'}{\max(1,\kappa-\dist(i,j)+1)^\mu}.
\end{align*}
Thus we have
\begin{align*}
    &|\Qtau_j(s,a_{-i},a_i)-\Qtau_j(s,a_{-i},a_i')-\Qtau_j(\ext{s_{\nik}},[\ext{a_{\nik}}]_{-i},a_i)+\Qtau_j(\ext{s_{\nik}},[\ext{a_{\nik}}]_{-i},a_i')|\\ &\leq 2\min\bigg\{\frac{\nu'}{(\dist(i,j)+1)^\mu}, \frac{\nu'}{\max(1,\kappa-\dist(i,j)+1)^\mu} \bigg\}\\
    &\leq  \frac{2\nu'}{(\kappa/2+1)^\mu},
\end{align*}
where the last inequality holds since either $\dist(i,j)$ or $\kappa-\dist(i,j)$ will be larger than $\kappa/2$.
\end{proof}
By \Cref{lem:E_2_bound}, we can derive an upper bound for $E_2$ as follows.
\begin{align}
|E_2|\leq \frac{2}{n}\sum_{j=1}^n \frac{\nu'}{(\kappa/2+1)^\mu}
&=\frac{2\nu'}{(\kappa/2+1)^\mu}.
\end{align}

\textbf{Bound on $E_3$.} Similar to the proof we presented in bounding term $E_2$, to upper bound $E_3$, we only need to bound the following term.
\begin{align}\label{eq:one_step_analysis:E3}
   &\left| \Qtau(\ext{s_{\nik}},[\ext{a_{\nik}}]_{-i},a_i)-\Qtau(\ext{s_{\nik}},[\ext{a_{\nik}}]_{-i},a_i')-\tilde{\Qtau}(s,a_{-i},a_i)+\tilde{\Qtau}(s,a_{-i},a_i') \right| \notag\\
&= \frac{1}{n}\bigg| \sum_{j\in \nik}\left[\Qtau_j(\ext{s_{\nik}},[\ext{a_{\nik}}]_{-i},a_i)-\hat{\Qtau}_j([\ext{s_{\nik}}]_{\njb},[\ext{a_{\nik}}]_{\njb\backslash i},a_i)\right]+\sum_{j\notin \nik}\Qtau_j(\ext{s_{\nik}},[\ext{a_{\nik}}]_{-i},a_i) \notag\\
&\qquad - \sum_{j\in \nik}\left[\Qtau_j(\ext{s_{\nik}},[\ext{a_{\nik}}]_{-i},a'_i)-\hat{\Qtau}_j([\ext{s_{\nik}}]_{\njb},[\ext{a_{\nik}}]_{\njb\backslash i},a'_i)\right] - \sum_{j\notin \nik}\Qtau_j(\ext{s_{\nik}},[\ext{a_{\nik}}]_{-i},a'_i)] \bigg| \notag\\
&\leq \frac{1}{n}\sum_{j\in \nik}\left|\Qtau_j(\ext{s_{\nik}},[\ext{a_{\nik}}]_{-i},a_i)-\hat{\Qtau}_j([\ext{s_{\nik}}]_{\njb},[\ext{a_{\nik}}]_{\njb\backslash i},a_i)\right| \notag\\
&\qquad + \frac{1}{n} \sum_{j\in\nik}\left|\Qtau_j(\ext{s_{\nik}},[\ext{a_{\nik}}]_{-i},a'_i)-\hat{\Qtau}_j([\ext{s_{\nik}}]_{\njb},[\ext{a_{\nik}}]_{\njb\backslash i},a'_i)\right|\notag\\
&\qquad+\frac{1}{n}\sum_{j\notin \nik}\left| \Qtau_j(\ext{s_{\nik}},[\ext{a_{\nik}}]_{-i},a_i) -\Qtau_j(\ext{s_{\nik}},[\ext{a_{\nik}}]_{-i},a'_i)\right|\notag\\
&\leq 2\frac{|\nik|}{n}\epsilon + \frac{\nu'}{n(\kappa+1)^\mu},
\end{align}
where in the last step we used the error bound on $\hat{Q}_j$ in the condition of \Cref{lem:policy_improvement_error_restated}, and the property of $(\mu,\nu)$-decay matrices presented in \eqref{eq:mu_decay_imply_distance_decay}.

\textbf{Putting terms $E_0, E_1, E_2, E_3$ together. } Combining the bounds in the previous steps, we have,
\begin{align}\label{eq:xi_difference_E0123}
    |\xi_i^{p+1}(s)-\tilde\xi_i^{p+1}(s)|
&\leq (1-\eta\tau)|\xi_i^{p}(s)-\tilde\xi_i^{p}(s)|+\eta(|E_0|+|E_1|+|E_2|+|E_3|)\nonumber\\
&\leq (1-\eta\tau)|\xi_i^{p}(s)-\tilde\xi_i^{p}(s)|+\eta\bigg(3\sum_{\ell\neq i}v_\ell^p H_{i\ell}^\Qtau+\frac{2^{\mu+2}\nu'+2\nu'}{(\kappa/2+1)^{\mu}}+\frac{\nu'/n}{(\kappa+1)^{\mu}}+2|\nik|\epsilon/n\bigg)\notag\\
&\leq (1-\eta\tau)|\xi_i^{p}(s)-\tilde\xi_i^{p}(s)|+3\eta\sum_{\ell\neq i}v_\ell^p H_{i\ell}^\Qtau+\eta\bigg(\frac{3\nu' + 2^{\mu+2}\nu'}{(\kappa/2+1)^\mu}+2\epsilon\bigg).
\end{align}
{ Plugging $\eta = \frac{1}{\tau}$} into the above inequality, we have 
\begin{align*}
    |\xi_i^{p+1}(s)-\tilde\xi_i^{p+1}(s)|
&\leq  \frac{3}{\tau} \sum_{\ell\neq i}v_\ell^p H_{i\ell}^\Qtau+\frac{1}{\tau}\bigg(\frac{3\nu' + 2^{\mu+2}\nu'}{(\kappa/2+1)^\mu}+2\epsilon\bigg).
\end{align*}
By \Cref{lem:log_tv}, we have
\begin{align*}
    v_i^{p+1} &\leq \frac{1}{2} \actionmax^2 e^{\tilde\sigma'} \left(\exp\left( \frac{3}{\tau} \sum_{\ell\neq i}v_\ell^p H_{i\ell}^\Qtau+\frac{1}{\tau}\bigg(\frac{3\nu' + 2^{\mu+2}\nu'}{(\kappa/2+1)^\mu}+2\epsilon\bigg)\right) -1\right).
\end{align*}
Recall our definition of $\epsilon = \frac{c_{pe}}{(\beta+1)^\mu}$ in \eqref{lem:policy_evaluation_error} and the choice of $\tau >  2(2^{\mu+4}\nu' + 2\frac{c_{pe}}{(\beta+1)^\mu})$ in \Cref{lem:policy_improvement_error_restated}. With these results, we can easily show that the quantity inside the $\exp(\cdot)$ can be upper bounded by 
\begin{align}\label{eq:approx_version_policy_improvement_exponent_small}
    \frac{3}{\tau} \sum_{\ell\neq i}v_\ell^p H_{i\ell}^\Qtau+\frac{1}{\tau}\bigg(\frac{3\nu' + 2^{\mu+2}\nu'}{(\kappa/2+1)^\mu}+2\epsilon\bigg)
    &\leq \frac{1}{\tau} \left( 3 \cdot 2^{\mu+1}\nu' + 3\nu' + 2^{\mu+2}\nu' + 2c_{pe} \frac{1}{(\beta+1)^\mu}\right)\notag\\
    &\leq \frac{1}{2}.
\end{align}
By the simple fact that $e^x\leq 1+2x$ for $x\leq1/2$, we have
\begin{align*}
    v_i^{p+1} &\leq   \actionmax^2 e^{\tilde\sigma'}  \left(3 \frac{1}{\tau} \sum_{\ell\neq i}v_\ell^p H_{i\ell}^\Qtau+\frac{1}{\tau}\bigg(\frac{ 2^{\mu+3}\nu' + 2c_{pe}}{(\kappa/2+1)^\mu}\bigg)\right). 
\end{align*}
If we stack all $v_i^p$ into a vector $v^p = [v_i^p]_{i\in\mathcal{N}}$, the above inequality implies
\begin{align*}
    v^{p+1} &\leq   \actionmax^2 e^{\tilde\sigma'}  \left(3 \frac{1}{\tau}H^\Qtau v^p +\frac{1}{\tau}\bigg(\frac{ 2^{\mu+3}\nu'+ 2c_{pe}}{(\kappa/2+1)^\mu}\bigg) \mathbf{1} \right) ,
\end{align*}
and it immediately follows that 
\[  \Vert v^{p+1}\Vert_\infty \leq   \actionmax^2 e^{\tilde\sigma'}   \frac{3}{\tau} \Vert H^\Qtau\Vert_\infty \Vert v^p\Vert_\infty +\actionmax^2 e^{\tilde\sigma'}\frac{1}{\tau}\frac{ 2^{\mu+3}\nu'+2c_{pe}}{(\kappa/2+1)^\mu}.      \] 
Therefore, when  $\tau > 6\actionmax^2 e^{\tilde\sigma'}2^{\mu+1}\nu' $, we have $2\actionmax^2e^{\tilde\sigma'}\|H^\Qtau\|_{\infty}/\tau\leq 1/2$. And thus for all $p$, it holds that
\begin{align*}
    \Vert v^p\Vert_\infty \leq 2\actionmax^2 e^{\tilde\sigma'}\frac{1}{\tau}\frac{  2^{\mu+3}\nu'+2c_{pe}}{(\kappa/2+1)^\mu}\leq \frac{4 + \frac{c_{pe}}{2^{\mu}\nu'}}{3(\kappa/2+1)^\mu}.
\end{align*}
Further, using \Cref{lem:mul_weights_convergence}, when { $\tau> 2^{\mu+2} \nu' \actionmax^2 e^{\nu'/n\tau}$}, we have for each $i\in[n]$ that
\[ \sup_{s\in\cS}\TV(\pi_i^p(\cdot|s), \zeta_i(\cdot|s))\leq \frac{1}{2^p}. \]
By triangle inequality, we have
\[ \sup_{s\in\cS}\TV(\hat{\pi}_i^p(\cdot|s_{\nik}), \zeta_i(\cdot|s))\leq \frac{1}{2^p} + \frac{4 +\frac{c_{pe}}{2^{\mu}\nu'}}{3(\kappa/2+1)^\mu}. \]
Therefore, when { $p\geq-\log_2{\frac{4 + \frac{c_{pe}}{2^{\mu}\nu'}}{3(\kappa/2+1)^\mu}}$}, we have
\begin{align}\label{eq:l7_total_variation_bound}
\sup_{s\in\cS}\TV(\hat{\pi}_i^p(\cdot|s_{\nik}), \zeta_i(\cdot|s))\leq  \frac{4 + \frac{c_{pe}}{2^{\mu}\nu'}}{(\kappa/2+1)^\mu}.
\end{align}
This completes the proof of the first statement in {\bf Part (c)} of \Cref{lem:policy_improvement_error_restated}. Next we prove the second statement. Since we now have a total variation bound \eqref{eq:l7_total_variation_bound}, we can use \Cref{lem:performance_diff_tv} to bound the value function difference between the policy $\zeta_i$ obtained by directly applying the Bellman optimal operator and the policy $\hat{\pi}_i$ carried out by our policy iteration process. Specifically, we have that 

\begin{align*}
    \mathcal{T} \Vtau - \Vtau^{\hat{\pi}^p}    &\leq   \Vtau^\zeta - \Vtau^{\hat{\pi}^p} \\
      &\leq \| V^{\zeta} -V^{\hat\pi^p} \|_\infty \mathbf{1} \\
    &\leq \frac{1}{1-\gamma} (\tau\tilde\sigma' + \frac{\nu'}{n}) \sum_{i=1}^n  \sup_{s\in\cS}\TV(\zeta_i(\cdot|s),\hat{\pi}_i(\cdot|s_\nik)) \mathbf{1}\\
    &\leq \frac{1}{1-\gamma}(n\tau\tilde\sigma'+\nu')\frac{4 + \frac{c_{pe}}{2^{\mu}\nu'}}{(\kappa/2+1)^\mu}\mathbf{1}\\
    &=\frac{1}{1-\gamma}(\frac{2f(\kappa) c_{pe}}{(\beta+1)^\mu}+2\nu') \frac{4 + \frac{c_{pe}}{2^{\mu}\nu'}}{(\kappa/2+1)^\mu} \mathbf{1}.
\end{align*}
where the third inequality is due to {\bf Part (b)} of \Cref{lem:policy_improvement_error_restated} and  \Cref{lem:performance_diff_tv}.
Now that we obtain the upper bound for each step of our outer loop, which completes the proof of part (c).

\section{Results on Sufficient Exploration}\label{ap:exploration}

In this section, we provide a general result on the sufficient exploration for non-deterministic policies. Given $\delta>0$, we define
$\Pi_\delta=\{\zeta\in\Delta_{\text{policy}} \mid \min_{s,a}\zeta(a|s)\geq \delta\}$. For a $\sigma$-regular policy $\zeta$, since $\zeta_i(a_i|s)/\zeta_i(\tilde{a}_i|s)\leq e^\sigma$ for all $i=1,2\cdots,n$, $s\in\mathcal{S}$, and $a_i,\tilde{a}_i\in\mathcal{A}_i$, we have \begin{align*}
    \min_{a_i\in\mathcal{A}_i}\zeta_i(a_i|s)\geq \frac{\max_{\tilde{a}_i\in\mathcal{A}_i}\zeta_i(\tilde{a}_i|s)}{e^\sigma}\geq \frac{1}{|\mathcal{A}_i|e^\sigma}.
\end{align*}
It follows that
\begin{align*}
    \min_{a\in\mathcal{A}}\zeta(a|s)\geq \frac{1}{e^{n\sigma}\prod_{i=1}^n|\mathcal{A}_i|}.
\end{align*}
Therefore, all $\sigma$-regular policies are contained in $\Pi_\delta$ with $\delta=1/(e^{n\sigma}\prod_{i=1}^n|\mathcal{A}_i|)$.

\begin{proposition}\label{prop:exploration}
Suppose that \Cref{as:MC} is satisfied. Then we have the following results for any $\zeta\in \Pi_\delta$.
\begin{enumerate}
    \item The Markov chain $\{S_k\}$ induced by $\zeta$ is irreducible and aperiodic.
    \item Let $\mu^\zeta$ be the unique stationary distribution of the Markov chain $\{S_k\}$ induced by $\zeta$, then we have $\mu_{\min}:=\inf_{\zeta\in \Pi_\delta}\min_{s\in\mathcal{S}}\mu^\zeta(s)>0$.
    \item There exist constants $C>0$ and $\rho\in (0,1)$ such that
    \begin{align*}
        \sup_{\zeta\in \Pi_\delta}\text{TV}\left(P_\zeta^t(\cdot\mid s(0)=s),\mu_\zeta(\cdot)\right)\leq C\rho^t,\quad \forall\;t\geq 0.
    \end{align*}
\end{enumerate}
\end{proposition}
\Cref{prop:exploration} states that, as long as there is one policy $\zeta_b\in\Pi_\delta$ that is explorative, then all policies in $\Pi_\delta$ are uniformly explorative. Note that  \Cref{network-assp:geometric-mixing} is a direct consequence of \Cref{prop:exploration}. Therefore, we only need to prove \Cref{prop:exploration}.

\begin{proof}[Proof of  \Cref{prop:exploration}]
\noindent\textbf{Proof of Part (1).} Let $\zeta\in \Pi_\delta$ be arbitrary, and let $P_{\zeta_b}$ and $P_{\zeta}$ be the transition probability matrices under $\zeta_b$ and $\zeta$, respectively. For any $s,s'\in\mathcal{S}$ and $t\geq 1$, we have
	\begin{align*}
		P_{\zeta_b}^t(s,s')=\;&\sum_{s_0}P^{t-1}_{\zeta_b}(s,s_0)P_{\zeta_b}(s_0,s')\\
		=\;&\sum_{s_0}P^{t-1}_{\zeta_b}(s,s_0)\sum_{a\in\mathcal{A}}\zeta_b(a|s_0)P_a(s_0,s')\\
		= \;&\sum_{s_0}P^{t-1}_{\zeta_b}(s,s_0)\sum_{a\in\mathcal{A}}\frac{\zeta_b(a|s_0)}{\zeta(a|s_0)}\zeta(a|s_0)P_a(s_0,s')\\
		\leq  \;&\frac{1}{\delta}\sum_{s_0}P^{t-1}_{\zeta_b}(s,s_0)\sum_{a\in\mathcal{A}}\zeta(a|s_0)P_a(s_0,s')\\
		\leq  \;&\frac{1}{\delta}\sum_{s_0}P^{t-1}_{\zeta_b}(s,s_0)P_{\zeta}(s_0,s')\\
		=  \;&\frac{1}{\delta}[P^{t-1}_{\zeta_b} P_{\zeta}](s,s').
	\end{align*}
	Since the previous inequality holds for all $s$ and $s'$, we in fact have $\delta P_{\zeta_b}^t\leq P^{t-1}_{\zeta_b} P_{\zeta}$. Repeatedly using the previous inequality and we obtain 
	\begin{align*}
	    (\delta)^tP_{\zeta_b}^t\leq P_{\zeta}^t,
	\end{align*}
	for all $t\geq 1$.
	When $P_{\zeta_b}$ is irreducible and aperiodic, the previous inequality implies that $P_{\zeta}$ is irreducible and aperiodic.

\noindent\textbf{Proof of Part (2).} To establish the result, we need the following sequence of lemmas.
	\begin{lemma}\label{le:1}
	The mapping from a policy $\zeta\in \Pi_\delta$ to its corresponding transition probability matrix  $P_\zeta$ is linear and is $1$-Lipschitz continuous with respect to the $\ell_\infty$-norm.
    \end{lemma}
\begin{proof}[Proof of \Cref{le:1}]
	The linearity is straightforward, we here only compute the Lipschitz constant. Let $\zeta_1,\zeta_2\in \Pi_\delta$ be two policies, and let $P_{\zeta_1}$ and $P_{\zeta_2}$ be the transition probability matrices induced by $\zeta_1$ and $\zeta_2$, respectively. We will view $\zeta_1$ and $\zeta_2$ as $|\mathcal{S}|$ by $|\mathcal{A}|$ matrices. Using the definition of induced matrix norm and we have
	\begin{align*}
		\|P_{\zeta_1}-P_{\zeta_2}\|_\infty=\;&
	\max_{s\in\mathcal{S}}\sum_{s'\in\mathcal{S}}|P_{\zeta_1}(s,s')-P_{\zeta_2}(s,s')|\\
	=\;&\max_{s\in\mathcal{S}}\sum_{s'\in\mathcal{S}}\left|\sum_{a\in\mathcal{A}}(\zeta_1(a|s)-\zeta_2(a|s))P_a(s,s')\right|\\
	\leq \;&\max_{s\in\mathcal{S}}\sum_{s'\in\mathcal{S}}\sum_{a\in\mathcal{A}}|\zeta_1(a|s)-\zeta_2(a|s)|P_a(s,s')\\
	= \;&\max_{s\in\mathcal{S}}\sum_{a\in\mathcal{A}}\sum_{s'\in\mathcal{S}}|\zeta_1(a|s)-\zeta_2(a|s)|P_a(s,s')\\
	= \;&\max_{s\in\mathcal{S}}\sum_{a\in\mathcal{A}}|\zeta_1(a|s)-\zeta_2(a|s)|\\
	=\;&\|\zeta_1-\zeta_2\|_\infty.
	\end{align*}
	This completes the proof.
\end{proof}

Since $\Pi_\delta$ is a convex compact set and the mapping from $\zeta\in \Pi_\delta$ to the transition probability $P_\zeta$ is linear (hence continuous), the image space $P_{\Pi_\delta}=\{P_\zeta\mid \zeta\in \Pi_\delta\}$ is also convex and compact.

\begin{lemma}\label{le:2}
	Let $P_1,P_2\in P_{\Pi_\delta}$, and let $\mu_1$ and $\mu_2$ be their corresponding unique stationary distributions, respectively. Then there exists $L$ (which depends on $P_1$) such that 
	\begin{align*}
		\|\mu_1-\mu_2\|_2\leq L\|P_1-P_2\|_2.
	\end{align*}
	As a result, the mapping from an irreducible and aperiodic stochastic matrix $P\in P_{\Pi_\delta}$ to its unique stationary distribution is continuous.
\end{lemma}

\begin{proof}[Proof of \Cref{le:2}]
Since $P_1$ is an irreducible and aperiodic stochastic matrix, there exists $C_1>0$ and $\rho_1 \in (0,1)$ such that 
\begin{align*}
	\|P_1^t-M_1\|_2\leq C_1\rho_1^t,
\end{align*}
for all $t\geq 0$, where $M_1$ is a matrix with $|\mathcal{S}|$ rows, each of which is the vector $\mu_1$ \cite[Theorem 4.9]{levin2017markov}. By definition we have $\mu_1^\top P_1^t=\mu_1^\top$ and $\mu_2^\top P_2^t=\mu_2^\top$ for all $t\geq 0$. It follows that
\begin{align}
	\mu_1-\mu_2=\;&(P_1^t)^\top (\mu_1-\mu_2)+(P_1^t-P_2^t)^\top \mu_2\nonumber\\
	=\;&(P_1^t-M_1)^\top (\mu_1-\mu_2)+M_1^\top (\mu_1-\mu_2)+(P_1^t-P_2^t)^\top \mu_2.\label{eq:1}
\end{align}
Applying $\|\cdot\|_2$ on both sides of Eq. (\ref{eq:1}) and then using triangle inequality, and we have
\begin{align*}
	\|\mu_1-\mu_2\|_2
	\leq\;& \|(P_1^t-M_1)^\top(\mu_1-\mu_2)\|_2+\|M_1^\top(\mu_1-\mu_2)\|_2+\|(P_1^t-P_2^t)^\top \mu_2\|_2\\
	\leq\;& \|(P_1^t-M_1)^\top\|_2\|\mu_1-\mu_2\|_2+\|M_1^\top(\mu_1-\mu_2)\|_2+\|(P_1^t-P_2^t)^\top\|_2 \|\mu_2\|_2\\
	\leq\;& C_1\rho_1^t\|\mu_1-\mu_2\|_2+\|M_1^\top(\mu_1-\mu_2)\|_2+\|P_1^t-P_2^t\|_2.
\end{align*}
For the term $\|M_1^\top(\mu_1-\mu_2)\|_2$, we have by definition of $M_1$ that
\begin{align*}
	\|M_1^\top(\mu_1-\mu_2)\|_2
	=\|\mu_1-\mu_1\|_2
	=0.
\end{align*}
It follows that 
\begin{align*}
	\|\mu_1-\mu_2\|_2\leq C_1\rho_1^t\|\mu_1-\mu_2\|_2+\|P_1^t-P_2^t\|_2.
\end{align*}
When $t\geq \frac{\log(2C_1)}{\log(1/\rho_1)}$ (which implies $C_1\rho_1^t\leq \frac{1}{2}$), we have from the previous inequality that
\begin{align*}
	\|\mu_1-\mu_2\|_2\leq 2\|P_1^t-P_2^t\|_2.
\end{align*}
We next bound $\|P_1^t-P_2^t\|_2$ in terms of $\|P_1-P_2\|_2$. Observe that
\begin{align*}
	P_1^t-P_2^t=\;&P_1^t+\sum_{i=1}^{t-1}(-P_1^{t-i}P_2^i+P_1^{t-i}P_2^i)-P_2^t\\
	=\;&\sum_{i=1}^{t}(P_1^{k-i+1}P_2^{i-1}-P_1^{t-i}P_2^i)\\
	=\;&\sum_{i=1}^{t}P_1^{t-i}(P_1-P_2)P_2^{i-1}.
\end{align*}
Therefore, we have
\begin{align*}
	\|P_1^t-P_2^t\|_2\leq \;&\sum_{i=1}^{t}\|P_1^{t-i}\|_2\|P_1-P_2\|_2\|P_2^{i-1}\|_2
	\leq \;tN_{\max}^2\|P_1-P_2\|_2,
\end{align*}
Here $N_{\max}:=\max\{\|P\|_2\;:\;P \in P_{\Pi_\delta}\}$, which is well-defined and finite since $P_{\Pi_\delta}$ is a compact set and the $\ell_2$-norm is a continuous function.
Finally, we obtain
\begin{align*}
	\|\mu_1-\mu_2\|_2\leq 2tN_{\max}^2\|P_1-P_2\|_2,
\end{align*}
when $t\geq \frac{\log(2C_1)}{\log(1/\rho_1)}$. Note that while $N_{\max}$ is independent of $P_1$ and $P_2$, the factor $t$ depends on $C_1$ and $\rho_1$, both of which are functions of the stochastic matrix $P_1$.
\end{proof}
\begin{lemma}\label{le:3}
    For any $\mu\in\mathbb{R}^{|\mathcal{S}|}$, the mapping from $\mu$ to its minimum component is Lipschitz continuous.
\end{lemma}
\begin{proof}[Proof of \Cref{le:3}]
For any $\mu_1,\mu_2\in\mathbb{R}^{|\mathcal{S}|}$, we have
\begin{align*}
    \left|\min_{s\in\mathcal{S}}\mu_1(s)-\min_{s\in\mathcal{S}}\mu_2(s)\right|\leq \;&\max_{s\in\mathcal{S}}|\mu_1(s)-\mu_2(s)|
    =\;\|\mu_1-\mu_2\|_\infty.
\end{align*}
\end{proof}

Combining \Cref{le:1}, \Cref{le:2}, and \Cref{le:3} and we conclude that the mapping from $\zeta\in \Pi_\delta$ to the minimum component of the stationary distribution $\mu_\zeta$ of the induced Markov chain $\{S_k\}$ is continuous. In addition, since $\Pi_\delta$ is compact, we have by Weierstrass extreme value theorem that
\begin{align*}
	\mu_{\min}=\inf_{\zeta\in \Pi_\delta}\min_{s\in\mathcal{S}}\mu_\zeta(s)=\min_{\zeta\in \Pi_\delta}\min_{s\in\mathcal{S}}\mu_\zeta(s)>0.
\end{align*}

\noindent\textbf{Proof of Part (3).}  Given the lemmas presented in the proof of Proposition \Cref{prop:exploration} (2), this part directly follows from existing results in the literature, see for example \cite[Lemma 1]{zhang2021global}. The idea is to mimick the proof of the ergodic theorem \citep[Theorem 4.9]{levin2017markov} for irreducible and aperiodic Markov chains and perform a refined analysis.
\end{proof}

\bibliographystyle{ims}
\bibliography{references}

\end{document}